\PassOptionsToPackage{svgnames}{xcolor}
\documentclass{article}
\usepackage[utf8]{inputenc}
\usepackage[a4paper]{geometry}

\usepackage{microtype}
\usepackage{subfigure}
\usepackage{booktabs} 

\usepackage{dsfont}
\usepackage{amsmath, amssymb, amsfonts, amsthm}
\usepackage{algorithm, algpseudocode, algorithmicx}
\usepackage{mathtools,commath}
\usepackage{graphicx, multirow}
\usepackage[round]{natbib}
\usepackage{xcolor}
\usepackage{hyperref}
\usepackage{wrapfig}

\usepackage{tabularx}

\usepackage{pifont}

\newtheorem{theorem}{Theorem}

\newtheorem{lemma}[theorem]{Lemma}
\newtheorem{proposition}[theorem]{Proposition}
\newtheorem{definition}{Definition}
\newtheorem{example}{Example}

\allowdisplaybreaks

\title{WeightedSHAP: analyzing and improving Shapley based feature attributions} 

\author{
  Yongchan Kwon\\
  Columbia University \\
  New York, NY, 10027\\
  \texttt{yk3012@columbia.edu}
  \and
  James Zou\\
  Stanford University and Amazon AWS\\
  Stanford, CA, 94305\\
  \texttt{jamesz@stanford.edu}
}
\date{\today}

\begin{document}

\maketitle

\begin{abstract}
Shapley value is a popular approach for measuring the influence of individual features. While Shapley feature attribution is built upon desiderata from game theory, some of its constraints may be less natural in certain machine learning settings, leading to unintuitive model interpretation. In particular, the Shapley value uses the same weight for all marginal contributions---i.e. it gives the same importance when a large number of other features are given versus when a small number of other features are given. This property can be problematic if larger feature sets are more or less informative than smaller feature sets. Our work performs a rigorous analysis of the potential limitations of Shapley feature attribution. We identify simple settings where the Shapley value is mathematically suboptimal by assigning larger attributions for less influential features. Motivated by this observation, we propose WeightedSHAP, which generalizes the Shapley value and learns which marginal contributions to focus directly from data. On several real-world datasets, we demonstrate that the influential features identified by WeightedSHAP are better able to recapitulate the model's predictions compared to the features identified by the Shapley value.
\end{abstract}

\section{Introduction}
\label{sec:intro}
Explaining how a feature impacts a model prediction is a crucial question in machine learning (ML) as it provides a deeper understanding of how the model behaves and what insights have been extracted from data. In many real-world applications, it has been increasingly common to deploy complicated models such as a deep neural network model or a random forest to achieve high predictability. However, it often comes with a cost of unintuitive interpretations, and it naturally calls for a principled and practical attribution method. The goal of this work is to quantify the contribution of individual features to a particular prediction, also known as the attribution problem. 

\citet{lundberg2017unified} proposed a model-agnostic attribution method, SHapley Additive exPlanations (SHAP), based on the Shapley value from economics. Supported by theoretical properties that the Shapley value satisfies, SHAP has been a popular method in the attribution literature \citep{janzing2020feature, sundararajan2020many}. For instance, \citet{frye2020shapley} and \citet{aas2021explaining} developed the SHAP algorithms for dependent features, and \citet{heskes2020causal} and \citet{wang2021shapley} proposed a rigorous framework in causal inference settings. One of the practical problems to use SHAP is the heavy computational costs, and there have been works on improving computational efficiency \citep{covert2021improving, lundberg2020local, jethani2021fastshap}. In addition, SHAP has been deployed to various applied scientific research \citep{lundberg2018explainable, janizek2021uncovering, qiu2022interpretable}.

The Shapley value, a mathematical basis of SHAP, is a simple average of the marginal contributions that quantify the average change in a coalition function when a feature of interest is added from a subset of features with a given coalition size. There are different version of the marginal contributions by the coalition size, and the Shapley value takes a uniform weight to summarize the influence of a feature. This uniform weight arises due to the efficiency axiom of the Shapley value, which requires the sum of attributions to equal the original model prediction, but it is often problematic because some may be more informative than others. As we will show in Section~\ref{sec:analysis_of_marginal_contrib}, the Shapley value is not optimal to sort features in order of influence on a model prediction.

\paragraph{Our contributions.} While SHAP is widely used for feature attribution, its limitations are still not rigorously understood. We first show the suboptimality of the Shapley value through an analysis of the marginal contributions. We identify a key limitation of the Shapley value in that it assigns uniform weights to marginal contributions in computing the attribution score. We show that this can lead to attribution mistakes when different marginal contributions have different signal and noise. Motivated by this analysis, we propose WeightedSHAP, a generalization of the Shapley value which is more flexible. WeightedSHAP uses a weighted average of marginal contributions where the weights can be learned from the data. On several real-world datasets, our experiments demonstrate that WeightedSHAP is better able to identify influential features that recapitulate a model's prediction compared to a standard SHAP. WeightedSHAP is a simple modification of SHAP and is easy to adapt from existing SHAP implementations.

\subsection{Related works}
\paragraph{Model interpretation}
There are mainly two types of model interpretation depending on the quantity to be accounted for; global and local interpretations. The global interpretation is to explain the impact of a feature on a prediction model across the entire dataset \citep{lipovetsky2001analysis, breiman2001random, owen2014sobol, broto2020variance, zhao2021causal, benard2022shaff}. For instance, for a decision tree model, \citet{breiman2017classification} measures the total decrease of node impurity at node split by a feature of interest as an impact. In contrast, the local interpretation is to explain the impact of a feature on a particular prediction value \citep{lundberg2017unified, chen2018learning, janzing2020feature, lundberg2020local, heskes2020causal, jethani2021fastshap}. For a deep neural network model, a gradient-based method uses the gradient evaluated at a particular input sample as an impact \citep{simonyan2013deep, sundararajan2017axiomatic, ancona2017towards, selvaraju2017grad, adebayo2018sanity}.
Our work studies the local interpretation problem with a focus on a marginal contribution-based method, which we review in Section~\ref{sec:preliminaries}. The marginal contribution-based method is potentially advantageous over a gradient-based method as it does not require the differentiability of a prediction model.

\paragraph{Shapley value and its extension}
The Shapley value, introduced as a fair division method from economics \citep{shapley1953}, has been deployed in various ML problems. One leading application is data valuation, where the main goal is to quantify the impact of individual data points in model training. 
\citet{ghorbani2019data} and \citet{jia2019} propose to use the Shapley value for measuring the individual data value, and this concept has been extended to handle the randomness of data \citep{ghorbani2020distributional, kwon2021efficient}. As for the other applications of the Shapley value, model explainability \citep{ghorbani2020neuron}, model valuation \citep{rozemberczki2021shapley}, federated learning \citep{liu2021gtg}, and multi-agent reinforcement learning \citep{li2021shapley} have been studied. We refer to \citet{rozemberczki2022shapley} for a complementary literature review of ML applications of the Shapley value. 

The relaxation of the Shapley axioms has been one of the central topics in cooperative game theory \citep{shapley1953additive, banzhaf1964weighted, kalai1987weighted, weber1988probabilistic}. Recently, \citet{kwon2021beta} propose to relax the efficiency axiom in the data valuation problem, showing promising results in the low-quality data detection task. Given that Shapley axioms are often not readily applicable to ML problems, relaxing them has the potential to capture a better notion of significance. In this work, we explore the benefits of relaxation of the efficiency axiom on the attribution problem.

\section{Preliminaries}
\label{sec:preliminaries}
We review the marginal contribution and the Shapley value in the context of an attribution problem. We first introduce some notations. For $d \in \mathbb{N}$, let $\mathcal{X} \subseteq \mathbb{R}^d$ and $\mathcal{Y} \subseteq \mathbb{R}$ be an input space and an output space, respectively. We use a capital letter $X=(X_1, \dots, X_d)$ for an input random variable defined on $\mathcal{X}$, and a lower case letter $x=(x_1, \dots, x_d)$ for its realized value. We denote a prediction model by $\hat{f}:\mathcal{X} \to \mathcal{Y}$. For $j \in \mathbb{N}$, we set $[j]:=\{1, \dots, j\}$ and denote a power set of $[j]$ by $2^{[j]}$. For a vector $u \in \mathbb{R}^d$ and a subset $S=(j_1, \dots, j_{|S|}) \subseteq [d]$, we denote a subvector by $u_S := (u_{j_1}, \dots, u_{j_{|S|}})$. We assume that $X$ has a joint distribution $p(X)$ such that a conditional distribution $p(X_{[d]\backslash S} \mid X_S)$ is well-defined for any subset $S\subsetneq[d]$. With the notations, a conditional coalition function $v_{x, \hat{f}} ^{\mathrm{(cond)}}: 2 ^{[d]} \to \mathbb{R}$ is defined as follows \citep{lundberg2017unified}.
\begin{align}
    v_{x, \hat{f}} ^{\mathrm{(cond)}} (S) := \mathbb{E} [ \hat{f}(x_S, X_{[d] \backslash S} ) \mid X_S =x_S ] - \mathbb{E} [\hat{f}(X)], 
    \label{eqn:conditional_coalition}
\end{align}
where the first expectation is taken with a conditional distribution $p(X_{[d] \backslash S} \mid X_S=x_S)$ and the second expectation is taken with a joint distribution $p(X)$. Here, we use a slight abuse of notation for $\hat{f}(x_S, X_{[d] \backslash S} )$ to describe $f(u)$ where $u_i = x_i$ if $i \in S$, and $u_i = X_i$ otherwise. By convention, we set $v_{x, \hat{f}} ^{\mathrm{(cond)}} ([d]) := \hat{f}(x) -\mathbb{E} [\hat{f}(X)]$ and $v_{x, \hat{f}} ^{\mathrm{(cond)}} (\emptyset) := 0$.
A conditional coalition function defined in Equation~\eqref{eqn:conditional_coalition} is a prediction recovered after observing partial information $x_S$ compared to the null information. For instance, if $S=\emptyset$, the first term becomes the marginal expectation $\mathbb{E} [\hat{f}(X)]$ and nothing is recovered by $S=\emptyset$. For ease of notation, we write $v_{x, \hat{f}} ^{\mathrm{(cond)}} (S) = v ^{\mathrm{(cond)}} (S)$ for remaining part of the paper.

\paragraph{Marginal contribution-based attribution methods}
Given that the goal of the attribution problem is to assign the significance of an individual feature $x_i$ on the prediction $\hat{f}(x)$, its primary challenge is how to measure the influence of the feature $x_i$. A leading approach is to quantify the difference in the conditional coalition function values $v^{\mathrm{(cond)}}$ after adding one feature of interest. We formalize this concept below.
\begin{definition}[Marginal contribution]
For $i, j \in [d]$, we define the marginal contribution of the $i$-th feature $x_i$ with respect to $j-1$ features as follows.
\begin{align}
    \Delta_{j}( x_i ) &:= \frac{1}{\binom{d-1}{j-1}}\sum_{S \subseteq [d] \backslash\{i\}, |S|=j-1 } v ^{\mathrm{(cond)}}(S\cup \{i\}) - v ^{\mathrm{(cond)}}(S).
    \label{eqn:marginal_contribution}
\end{align}
\label{def:marginal_contribution}
\end{definition}
The marginal contribution $\Delta_{j}(x_i)$ considers every possible subset $S \subseteq [d]\backslash \{i\}$ with the coalition size $|S|=j-1$ and takes a simple average of the difference $v ^{\mathrm{(cond)}}(S\cup \{i\}) - v ^{\mathrm{(cond)}}(S)$. That is, it measures the average contribution of the $i$-th feature $x_i$ when it is added to a subset $S$.

Different marginal contributions $\Delta_{j}( x_i )$ have been studied depending on the coalition size $j$ in the literature. \citet{zintgraf2017visualizing} considered $j=d$ and measured the leave-one-out marginal contribution $\Delta_{d}(x_i) = \hat{f}(x) - \mathbb{E} [ \hat{f}(x_{[d] \backslash \{i\}}, X_{i} ) \mid X_{[d] \backslash \{i\}} =x_{[d] \backslash \{i\}} ]$ as an influence of a feature. \citet{guyon2003introduction} considered $j=1$ and measured the coefficient of determination as an influence of a feature. Although they did not use an individual prediction, their idea is essentially similar to using $\Delta_{1}(x_i) = \mathbb{E} [ \hat{f}(x_i, X_{[d] \backslash \{i\}} ) \mid X_i = x_i ] - \mathbb{E} [\hat{f}(X)]$.

Another widely used marginal contribution-based method is the Shapley value \citep{lundberg2017unified, covert2021improving}. It summarizes the impact of one feature by taking a simple average across all marginal contributions. To be more specific, the Shapley value is defined as follows.
\begin{align}
    \phi_{\mathrm{shap}}(x_i) := \frac{1}{d} \sum_{j=1} ^{d} \Delta_{j}(x_i).
    \label{eqn:shapley_value}
\end{align}
The Shapley value in \eqref{eqn:shapley_value} is known as the unique function that satisfies the four axioms of a fair division in cooperative game theory \citep{shapley1953}. The four axioms and the uniqueness of the Shapley value are discussed in more detail in Appendix.

Although the Shapley value provides a principled framework in game theory, one critical issue is that the economic notion of the Shapley axioms is not intuitively applicable to the attribution problem \citep{kumar2020problems, rozemberczki2022shapley}. In particular, the efficiency axiom, which requires the sum of the attributions to be equal to $v^{\mathrm{cond}}([d])$, is not necessarily essential because an order of attributions is invariant to the constant multiplication. For instance, for any positive constant $C>0$, an attribution  $\phi_C(x_i) := C \times \phi_{\mathrm{shap}} (x_i)$ will have the same order as the Shapley value $\phi_{\mathrm{shap}}$, but the efficiency axiom is not required for $\phi_C$. In Section~\ref{sec:proposed}, we will revisit this point and introduce a new attribution method that relaxes the efficiency axiom. 

\paragraph{Evaluation metrics for the attribution problem.}
In the literature, different notions of goodness have been proposed, for instance, the complete axiom \citep{sundararajan2017axiomatic, shrikumar2017learning}, the local Lipschitzness \citep{alvarez2018robustness}, and the explanation infidelity \citep{yeh2019fidelity} with a focus on the total sum of attributions or the sensitivity of attributions.
Recently, \citet{jethani2021fastshap} suggested using the \textit{Inclusion AUC} to assess the goodness of an order of attributions. Specifically, the Inclusion AUC is measured as follows: Given an attribution method, features are first ranked based on their attribution values. Then the area under the receiver operating characteristic curve (AUC) is iteratively evaluated by adding features one by one from the most influential to the least influential. This procedure generates a AUC curve as a function of the number features added, and the Inclusion AUC is defined as the area under this curve. Similar evaluation metrics have been used in \citet{petsiuk2018rise} and \citet{lundberg2020local}. Following the literature, we consider the area under the prediction recovery error curve (AUP) defined as follows.

\begin{definition}[Area under the prediction recovery error curve]
For a given attribution method $\phi$, an input $x \in \mathcal{X}$, and $k\in[d]$, let $\mathcal{I}(k; \phi, x) \subseteq [d]$ be a set of $k$ integers that indicates $k$ most influential features based on their absolute value $|\phi(x_j)|$. For a prediction model $\hat{f}$, we define the area under the prediction recovery error curve at $x$ as follows.
\begin{align}
    \mathrm{AUP}(\phi; x, \hat{f}) := \sum_{k=1} ^d \left| \hat{f}(x)-\mathbb{E}[\hat{f}(X) \mid X_{\mathcal{I}(k; \phi, x)}=x_{\mathcal{I}(k; \phi, x)} ] \right|.
    \label{eqn:prediction_recovery}
\end{align}
\end{definition}
AUP is defined as the sum of the absolute differences between the original prediction $\hat{f}(x)$ and its conditional expectation $\mathbb{E}[\hat{f}(X) \mid X_{\mathcal{I}(k; \phi, x)}=x_{\mathcal{I}(k; \phi, x)} ]$ when the $k$ most influential features are given. Each term in Equation~\eqref{eqn:prediction_recovery} measures the amount of a prediction that is not recovered by the $k$ most influential features, and thus this prediction recovery error is expected to decrease as $k$ increases. The prediction recovery error can be described as a function of $k$, and the AUP measures the area under this function as in the Inclusion AUC. 

\section{The Shapley value is suboptimal}
\label{sec:analysis_of_marginal_contrib}
In this section, we show that the suboptimality of the Shapley value through a rigorous analysis of the marginal contribution. We first derive a useful closed-form expression of the marginal contribution when $\hat{f}$ is linear and $p(X)$ is Gaussian (Section~\ref{sec:closed_form}). 
With this theoretical result, we present two simulation experiments where the Shapley value incorrectly reflects the influence of features, resulting in a suboptimal order of attributions (Section~\ref{sec:motivational_examples}).

\subsection{A closed-form expression of the marginal contribution}
\label{sec:closed_form}
To this end, we assume that a prediction model $\hat{f}$ is linear and an input distribution $p(X)$ follows a Gaussian distribution with zero mean and a block diagonal covariance matrix. To be more specific, we define some notations. For $B \in \mathbb{N}$, we set a vector $\mathbf{d}=(d_1, \dots, d_B) \in \mathbb{N}^B$ such that $\sum_{b=1} ^B d_b = d$ and a vector $\pmb{\rho} := (\rho_1, \dots, \rho_B) \in [0,1)^{B}$. We denote a $d \times d$ block diagonal covariance matrix by $\Sigma_{\pmb{\rho}, \mathbf{d}} ^{\mathrm{(block)}} = \mathrm{diag}\left( \Sigma_{(\rho_1, d_1)}, \dots, \Sigma_{(\rho_B, d_B)} \right)$ where  $\Sigma_{(\rho_b,d_b)}=(1-\rho_b) I_{d_b} + \rho_b \mathds{1}_{d_b} \mathds{1}_{d_b} ^T$. Here, for $j \in \mathbb{N}$, we denote the $j \times j$ identity matrix by $I_j$, the $j$-dimensional vector of ones by $\mathds{1}_j := (1, \dots, 1)^T \in \mathbb{R}^{j}$ and $\mathbf{0}_j := 0 \times \mathds{1}_j$. Lastly, we denote a Gaussian distribution with a mean vector $\mu$ and a covariance matrix $\Sigma$ by $\mathcal{N}(\mu, \Sigma)$. With the notations, we assume $X \sim \mathcal{N}(\mathbf{0}_d, \Sigma_{\pmb{\rho}, \mathbf{d}} ^{\mathrm{(block)}})$. That is, every feature is normalized to have a unit variance and is included in one of $B$ independent clusters. For $j\in[B]$, the size of the $j$-th cluster is $d_j$, and features are equally correlated to each other within a cluster. The correlation levels can vary from cluster to cluster.

In general, the marginal contribution in Equation~\eqref{eqn:marginal_contribution} does not have a closed-form expression, and it makes a rigorous analysis of the Shapley value difficult. In the following theorem, we derive a closed-form expression of the marginal contribution when $\hat{f}$ is linear and $X \sim \mathcal{N}(\mathbf{0}_d, \Sigma_{\pmb{\rho}, \mathbf{d}} ^{\mathrm{(block)}})$.
\begin{theorem}[A closed-form expression for the marginal contribution]
Suppose $\hat{f}(x) = \hat{\beta}_0 + x^T \hat{\beta}$ for some $(\hat{\beta}_0, \hat{\beta}) \in \mathbb{R} \times \mathbb{R}^d$ and $X \sim \mathcal{N}(\mathbf{0}_d, \Sigma_{\pmb{\rho}, \mathbf{d}} ^{\mathrm{(block)}})$. Then, for $i, j \in [d]$, the marginal contribution of the $i$-th feature $x_i$ with respect to $j-1$ samples is expressed as
\begin{align*}
    \Delta_{j}(x_i) = x^T H(i, j) \hat{\beta},
\end{align*}
for some explicit matrix $H(i,j) \in \mathbb{R}^{d \times d}$.
\label{thm:marginal_contribution}
\end{theorem}
A proof and the explicit term for $H(i,j)$ are provided in Appendix. Theorem~\ref{thm:marginal_contribution} shows that the marginal contribution is a bilinear function of an input $x$ and the estimated regression coefficient $\hat{\beta}$. One direct consequence is that the Shapley value also has a bilinear form $\phi_{\mathrm{shap}}(x_i)= x^T H(i) \hat{\beta}$ for $H(i) := \sum_{j=1} ^d H(i,j)/d$. We emphasize that this bilinear form greatly improves computational efficiency. Specifically, for all $i,j\in[d]$, since the term $H(i,j)\hat{\beta}$ is not a function of an input $x$, we only need to compute the one-time in multiple attribution computations. Moreover, it also leads to a memory efficient algorithm as there is no need to store the $d\times d$ matrix $H(i,j)$.

\begin{figure}[t]
    \centering
    \subfigure[Illustrations of the suboptimality of Shapley-based feature attributions on the four different situations when $d=2$.]{
      \centering
      \label{fig:motivation_gaussian_two_features}
      \includegraphics[width=\textwidth]{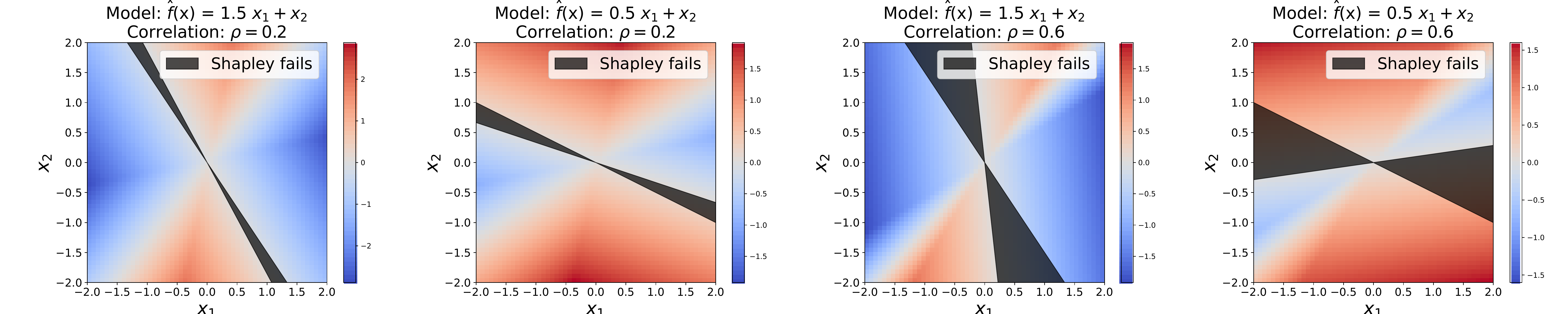}
    }
    \subfigure[Illustrations of a prediction recovery error curve and AUP comparison when $d=100$.]{
      \centering
      \label{fig:motivation_gaussian_more_than_two_features}
      \includegraphics[width=0.49\textwidth]{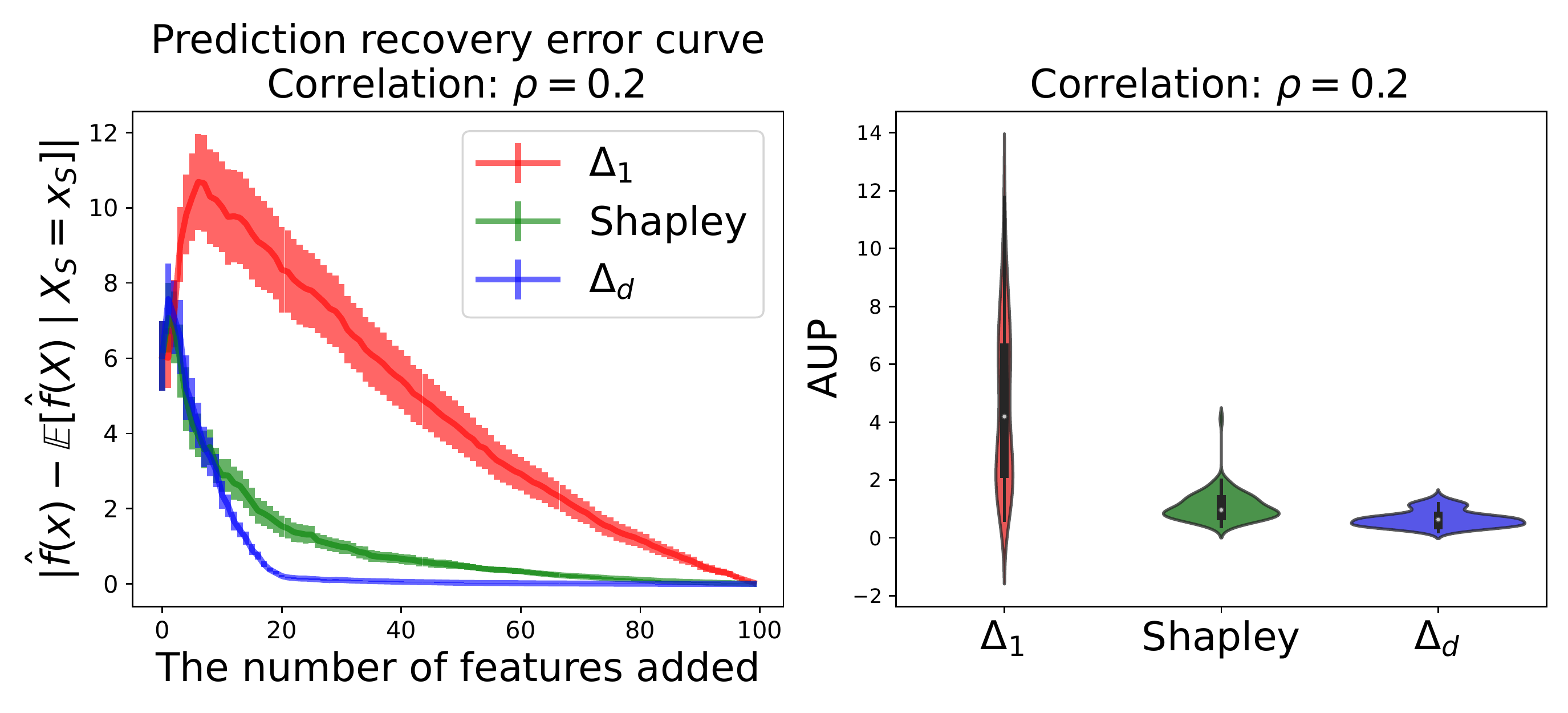}
      \includegraphics[width=0.49\textwidth]{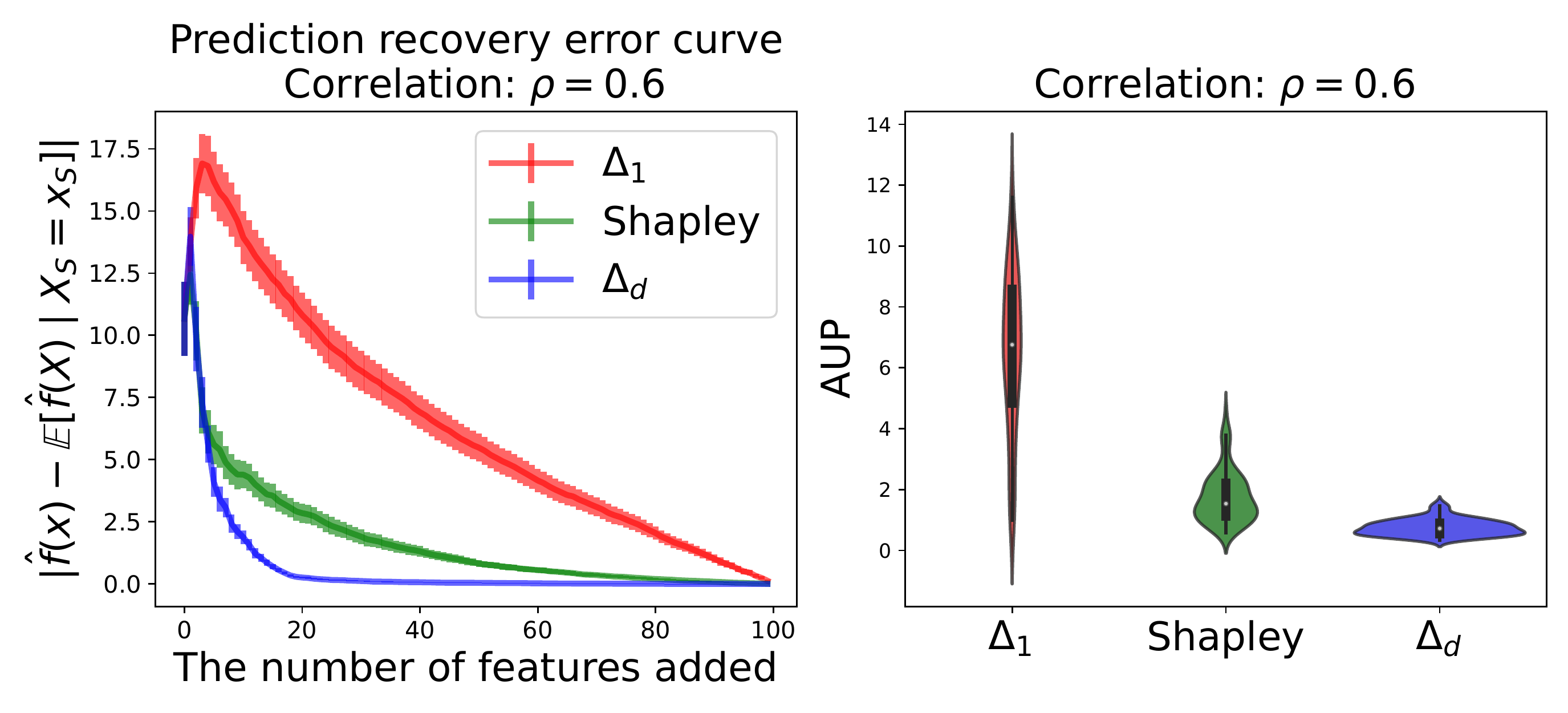}
    }
    \caption{The Shapley value is suboptimal. (Top) a region described in black denotes the area the Shapley value fails to select the more influential feature. We encode $\mathcal{E}(1; x, \hat{f})-\mathcal{E}(2; x, \hat{f})$ as the background color to visualize which feature is more influential. Blue color describes a region where the first feature $x_1$ is more influential, \textit{i.e.}, $\mathcal{E}(1; x, \hat{f}) < \mathcal{E}(2; x, \hat{f})$, and red describes a region where the second feature $x_2$ is more influential, \textit{i.e.}, $\mathcal{E}(1; x, \hat{f}) > \mathcal{E}(2; x, \hat{f})$. The intensity for $\mathcal{E}(1; x, \hat{f})-\mathcal{E}(2; x, \hat{f})$ is described in a color bar. (Bottom) we compare the three attribution methods $\Delta_1, \phi_{\mathrm{shap}},$ and $\Delta_d$ on the two different situations by varying correlation $\rho \in \{0.2, 0.6\}$. As for the prediction recovery error curve, we denote a 95\% confidence band based on 100 samples. The lower AUP is, the better. In both settings, the Shapley value is suboptimal according to AUP.}
\end{figure}

\subsection{Motivational examples}
\label{sec:motivational_examples}
With the theoretical result introduced in the previous subsection, we show that the Shapley-based feature attribution is suboptimal and fails to assign larger values to more influential features.

\paragraph{When there are two features.}
When $d=2$, there are only two possible values for AUP.
For any attribution method $\phi$, 
\begin{align*}
    \mathrm{AUP}(\phi; x, \hat{f})=\begin{cases}
    \mathcal{E}(1; x, \hat{f}) & \text{if  } \mathcal{I}(1; \phi, x)=\{1\}\\
    \mathcal{E}(2; x, \hat{f})  & \text{otherwise}
\end{cases},
\end{align*}
where $\mathcal{E}(k; x, \hat{f}) := \left| \hat{f}(x_1, x_2) - \mathbb{E}[\hat{f}(X_1, X_2) \mid X_k =x_k] \right|$ for $k \in \{1,2\}$. Therefore, the optimal order based on AUP is fully determined by $\mathcal{E}(1; x, \hat{f})$ and $\mathcal{E}(2; x, \hat{f})$, for instance, the first feature $x_1$ is more influential than the second one $x_2$ if $\mathcal{E}(1; x, \hat{f}) < \mathcal{E}(2; x, \hat{f})$. It is intuitively sensible because $\mathcal{E}(1; x, \hat{f}) < \mathcal{E}(2; x, \hat{f})$ means that the original prediction $\hat{f}(x)$ is more accurately recovered by the first feature $x_1$ than the second one $x_2$.

Using the optimal order, we demonstrate that the Shapley value does not necessarily assign a large attribution to a more influential feature. We consider the four different scenarios with two different prediction models $\hat{f}(x) \in \{1.5 x_1+x_2, 0.5 x_1+x_2\}$ and two different Gaussian distributions, $X \sim \mathcal{N} \left( \mathbf{0}_2, \Sigma_{(\rho, 2)} \right)$ for $\rho \in \{0.2, 0.6\}$. In these four scenarios, the terms $\mathcal{E}(1; x, \hat{f})$ and $\mathcal{E}(2; x, \hat{f})$ have a closed-form expression, and thus the optimal order is explicitly obtained. Moreover, due to Theorem~\ref{thm:marginal_contribution}, a more influential feature according to the Shapley value is explicitly obtained. 

Figure~\ref{fig:motivation_gaussian_two_features} illustrates the suboptimality of the Shapley value on the four different situations. In any situation, there is a non-negligible region (described in black) where the Shapley value fails to select a more influential feature. In addition, this suboptimal area increases as the correlation gets larger, showing that the Shapley value-based explainability becomes poor when features are highly correlated.

\paragraph{When there are more than two features}
When $d>2$, it is infeasible to find the exact optimal order because there are $2^{d-1}$ possible AUPs. For this reason, we compare the Shapley value with the two commonly used marginal contribution-based methods $\Delta_1$ and $\Delta_d$, showing the Shapley value is not optimal in terms of AUP. We assume the following setting: a trained model is linear $\hat{f}(x) = \hat{\beta}_0 + x^T \hat{\beta}$ for some $(\hat{\beta}_0, \hat{\beta}) \in \mathbb{R} \times \mathbb{R}^d$ and an input vector $X=(X_1, \dots, X_d)$ follows a Gaussian distribution $\mathcal{N} \left( \mathbf{0}_d, \Sigma_{(\rho, d)} \right)$. That is, there are $d$ features and they are equally correlated to each other with the correlation $\rho$. We set $d=100$ and consider two different situations by varying $\rho \in \{0.2, 0.6\}$. Similar to the previous analysis, due to Theorem~\ref{thm:marginal_contribution}, the three attribution methods are explicitly obtained. We evaluate the prediction recovery error and the AUP on the 100 held-out test samples randomly drawn from the distribution $\mathcal{N} \left( \mathbf{0}_d, \Sigma_{(\rho, d)} \right)$. Detailed information is provided in Appendix.

Figure~\ref{fig:motivation_gaussian_more_than_two_features} illustrates the suboptimality of the Shapley value when $d=100$ and $\rho\in\{0.2, 0.6\}$. In any situation, the prediction recovery curves for the Shapley value (described in green) have a steeper slope than $\Delta_1$ (described in red), but is not optimal as the $\Delta_d$ (described in blue) approaches to zero faster. When $\rho=0.6$, the suboptimality becomes more severe in that the gap between $\Delta_d$ and the Shapley value gets larger.

\section{Proposed method: WeightedSHAP}
\label{sec:proposed}

\begin{wrapfigure}[21]{R}{0.41\textwidth} 
    \vspace{-11pt}
    \centering
    \includegraphics[width=0.4\textwidth]{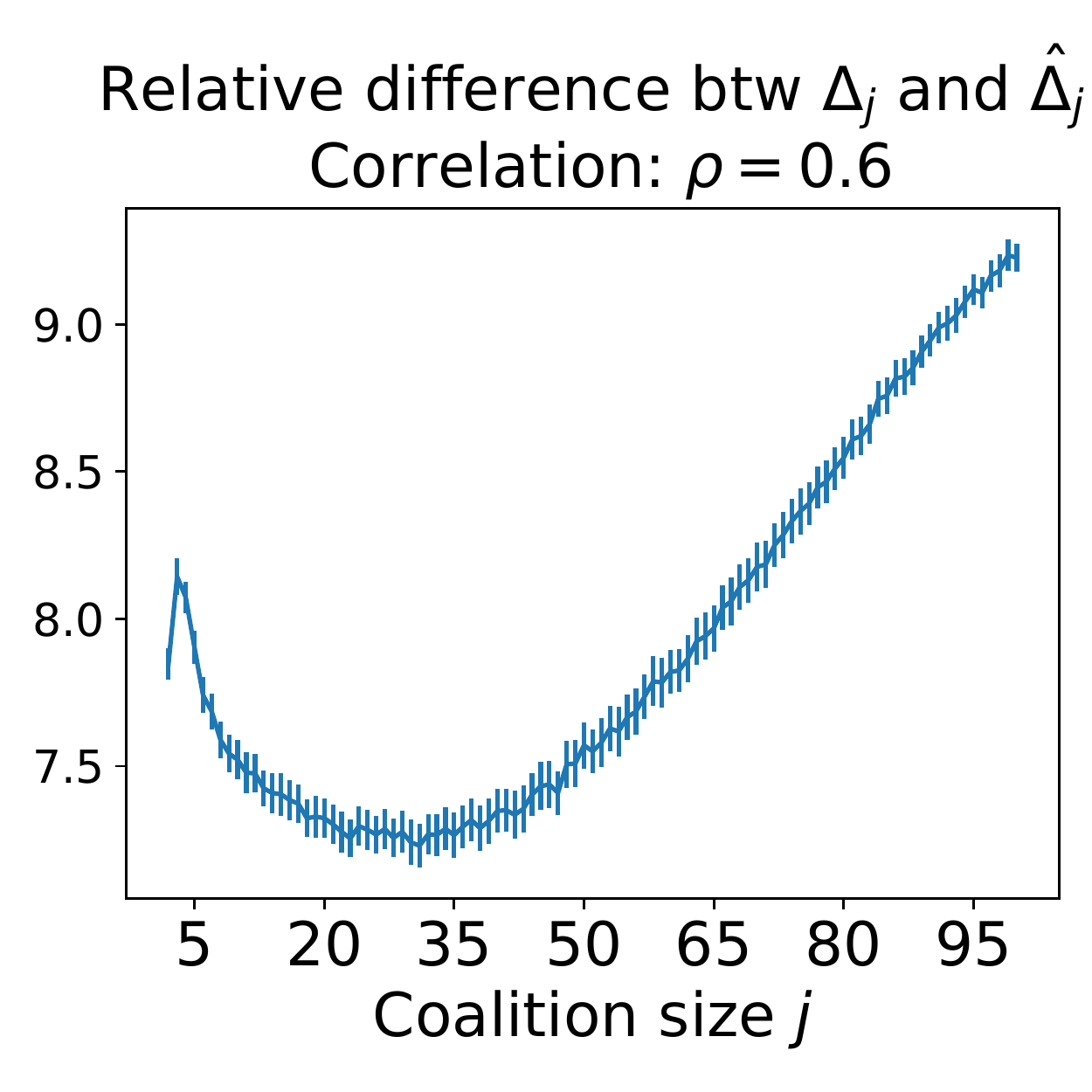}
    \caption{Illustrations of the relative difference between the true marginal contribution $\Delta_j$ and its estimate $\hat{\Delta}_j$ as a function of the coalition size $j \in [d]$. We consider the same setting used in Figure~\ref{fig:motivation_gaussian_more_than_two_features}. The $\Delta_d$ is shown to have the largest relative difference.}
    \label{fig:estimation_error_of_the_marginal_contributions}
\end{wrapfigure}
Our motivational examples in the previous section suggest that the Shapley value does not necessarily assign larger attributions for more influential features, leading to a suboptimal order of features. In fact, the last marginal contribution $\Delta_d$ outperforms other attribution methods in Figure~\ref{fig:motivation_gaussian_more_than_two_features}. 
Although the use of $\Delta_d$ is promising, we show that focusing on one marginal contribution might lead to an unstable attribution method, suggesting a weighted mean of the marginal contributions. To be more concrete, we first examine the estimation error of the marginal contribution in the following.

\paragraph{Analysis of estimation error}
In practice, the Shapley value needs to be estimated, resulting in an estimation error. Given the mathematical form of the Shapley value in \eqref{eqn:shapley_value}, this estimation error arises from the estimation error of the marginal contribution. In this reason, we investigate the estimation error of the marginal contribution. We consider the same setting used in Figure~\ref{fig:motivation_gaussian_more_than_two_features} with $\rho=0.6$. As for the estimation of the marginal contribution, we follow a standard algorithm to estimate a conditional coalition function $v^{\mathrm{(cond)}}$ used in \citet{jethani2021fastshap} and a sampling-based algorithm to approximate $\Delta_j$. A detailed explanation for the estimation procedure and the additional result for $\rho=0.2$ are provided in Appendix.

Figure~\ref{fig:estimation_error_of_the_marginal_contributions} shows the relative difference between the true marginal contribution $\Delta_j$ obtained by Theorem~\ref{thm:marginal_contribution} and its estimate $\hat{\Delta}_j$ as a function of the coalition size $j \in [d]$. Here, we use the relative difference between $A$ and $B$ defined as $|A-B|/\max(|A|,|B|)$ to avoid numerical instability that can be occurred by too small marginal contribution values. It shows the $\Delta_d$, the most informative marginal contribution in Figure~\ref{fig:motivation_gaussian_more_than_two_features}, has the largest relative difference from the true value. 
In other words, $\Delta_d$ has the largest signal to explain a model prediction, but at the same time, it is the most unstable in terms of the estimation error. This finding motivates us to consider a weighted mean of the marginal contributions that can reduce the estimation error while maintaining signals.

\paragraph{Proposed method}
For a weighted vector $\mathbf{w}=(w_1, \dots, w_d)$ such that $\sum_{j=1} ^d w_j =1$ and $w_j \geq 0$ for all $j\in[d]$, we consider a weighted mean of the marginal contributions
\begin{align}
    \phi_\mathbf{w} (x_i) := \sum_{j=1} ^{d} w_j \Delta_{j}(x_i).
    \label{eqn:semivalue}
\end{align}
A weighted mean $\phi_\mathbf{w} (x_i)$ is expected to capture the influence of features better than the Shapely value \eqref{eqn:shapley_value} by assigning a large weight to important marginal contributions. As for the game theoretic interpretation, a mathematical form of Equation~\eqref{eqn:semivalue} is known as a semivalue in cooperative game theory. It satisfies all the Shapley axioms but the efficiency axiom, which is not crucial in the attribution as we discussed in Section~\ref{sec:preliminaries} \citep{dubey1977probabilistic, ridaoui2018axiomatisation}. Due to the relaxation of the efficiency axiom, a semivalue is not uniquely determined, but it is known that the semivalue is \textit{almost} unique up to a weighted mean operation. A detailed explanation for the semivalue is provided in Appendix. 

One natural and practical question that arises when using a weighted mean $\phi_\mathbf{w} (x_i)$ is how to select the weight vector $\mathbf{w}$. Since a weight vector to be selected is desired to have a certain good property, this question can be rephrased as to which concept of goodness should be optimized. However, it is difficult to have one universal desideratum by the intricate nature of model interpretations. There are different types of goodness and they often represent independent characteristics, as we discussed in Section~\ref{sec:preliminaries}. In other words, a good attribution essentially depends on a practitioner's downstream task. To reflect this, we propose to learn a weight vector that optimizes a user-defined utility. 
To be more specific, we let $\mathcal{W} \subseteq \{w \in \mathbb{R}^d : \sum_{j=1} ^d w_j =1, w_j \geq 0 \}$ be a parametrized family of weights and $\mathcal{T}$ be a user-defined utility function that takes as input an attribution method and outputs its utility. Without loss of generality, we assume that the larger $\mathcal{T}$ is, the better it is (\textit{e.g.}, the negative value of AUP). Given $\mathcal{T}$ and $\mathcal{W}$, we propose WeightedSHAP as follows.
\begin{align}
    \phi_{\mathrm{WeightedSHAP}}(\mathcal{T}, \mathcal{W}) := \phi_{\mathbf{w}^* (\mathcal{T}, \mathcal{W})},
    \label{eqn:weightedSHAP}
\end{align} 
where $\mathbf{w}^* (\mathcal{T}, \mathcal{W}) := \mathrm{argmax}_{\mathbf{w} \in \mathcal{W}} \mathcal{T} (\phi_{\mathbf{w}})$. That is, we learn the optimal weight by optimizing a user-defined utility. 

When $\mathcal{W}$ includes the uniform weight $(1/d, \dots, 1/d) \in \mathbb{R}^d$, then by its construction, we can guarantee that WeightedSHAP is always better than or equal to the Shapley value according to the utility $\mathcal{T}$. For instance, when the negative value of AUP is used for the utility $\mathcal{T}$, AUP of WeightedSHAP is less than that of the Shapley value, \textit{i.e.}, $\mathrm{AUP}(\phi_{\mathrm{WeightedSHAP}}) \leq \mathrm{AUP}(\phi_{\mathrm{shap}})$. 
Moreover, the more weight vectors are in $\mathcal{W}$, the better the quality of $\phi_{\mathrm{WeightedSHAP}}$ is guaranteed. WeightedSHAP $\phi_{\mathrm{WeightedSHAP}}$ depends on a set $\mathcal{W}$. In our experiments, we parameterize an element $\mathbf{w} \in \mathcal{W}$ by the Beta distribution inspired by mathematical properties of the semivalue in \citet[Theorem 11]{monderer2002variations}. Detailed information is provided in Appendix. 

\begin{example}[WeightedSHAP and the Shapley value $\phi_{\mathrm{shap}}$ on AUP]
We revisit the motivational example introduced in Section~\ref{sec:motivational_examples}. With the negative AUP for $\mathcal{T}$ and some $\mathcal{W} \supseteq \{\Delta_d, \phi_{\mathrm{shap}}\}$, WeightedSHAP achieves significantly lower AUP than both $\Delta_d$ and $\phi_{\mathrm{shap}}$. Specifically, when $(d,\rho)=(100,0.6)$, the AUPs of ($\Delta_d$, $\phi_{\mathrm{shap}}$, $\phi_{\mathrm{WeightedSHAP}}$) are ($1.49\pm0.06, 1.65\pm0.08, 0.77\pm0.03$), respectively, where the numbers denote ``mean $\pm$ standard error'' based on the 100 held-out test samples. Meanwhile, the estimation errors of ($\Delta_d$, $\phi_{\mathrm{shap}}$, $\phi_{\mathrm{WeightedSHAP}}$) are ($9.23\pm0.02, 6.16\pm0.04, 7.90\pm0.11$), respectively. In short, WeightedSHAP achieves a significantly lower estimation error than $\Delta_d$ while achieving the lowest AUP. Although $\phi_{\mathrm{shap}}$ achieves the lowest estimation error, its AUP is significantly greater than both $\Delta_d$ and $\phi_{\mathrm{WeightedSHAP}}$. The uniform weight in $\phi_{\mathrm{shap}}$ helps reduce the estimation error, but it loses signals too much. In contrast, WeightedSHAP well balances the signal and the estimation error, \textit{i.e.}, reducing the estimation error while taking more signals.
\end{example}

\textbf{Implementation algorithm for WeightedSHAP} Given a finite set $\mathcal{W}$ and an easy-to-compute utility function $\mathcal{T}$, the optimal weight $\mathbf{w}^*$ can be achieved by iteratively evaluating the utility $\mathcal{T}$ for each attribution method $\phi_{\mathbf{w}}$ with $\mathbf{w} \in \mathcal{W}$. In addition, $\phi_{\mathbf{w}}$ is readily obtained as long as there are the marginal contribution estimates. Therefore, the key part of the implementation algorithm is to estimate a set of marginal contributions. The estimation of the marginal contributions consists of two parts, estimation of a conditional coalition function $v^{\mathrm{(cond)}}$ and approximation of the marginal contribution $\Delta_j$. As for the first part, we train a surrogate model that takes as input a subset of input features and outputs a conditional expectation of a prediction value given the same subset \citep{frye2020shapley, jethani2021fastshap, jethani2021have}. It is known that this surrogate model unbiasedly estimates a conditional expectation of a prediction value given a subset of features under mild conditions \citep{frye2020shapley, covert2020understanding}. Regarding the second part, a weighted mean is approximated by a sampling-based algorithm \citep{ghorbani2019data, kwon2021beta}. We provide a pseudo algorithm in Appendix. 
In terms of the computational cost, our algorithm is comparable to a standard the Shapley value estimation algorithm because both algorithms need to estimate the marginal contributions as a primary part \citep{lundberg2017unified, frye2020shapley}. For instance, with the classification dataset \texttt{fraud}, the marginal contribution estimation part takes $20.7$ seconds per sample on average but the weight optimization part only takes $0.18$ seconds, \textit{i.e.}, the weight optimization part is only $0.86 \%$ of the total compute.

\begin{figure}[t]
    \centering
    \subfigure[Illustrations of the prediction recovery error curve on the four regression datasets. The lower, the better.]{
    \includegraphics[width=0.22\textwidth]{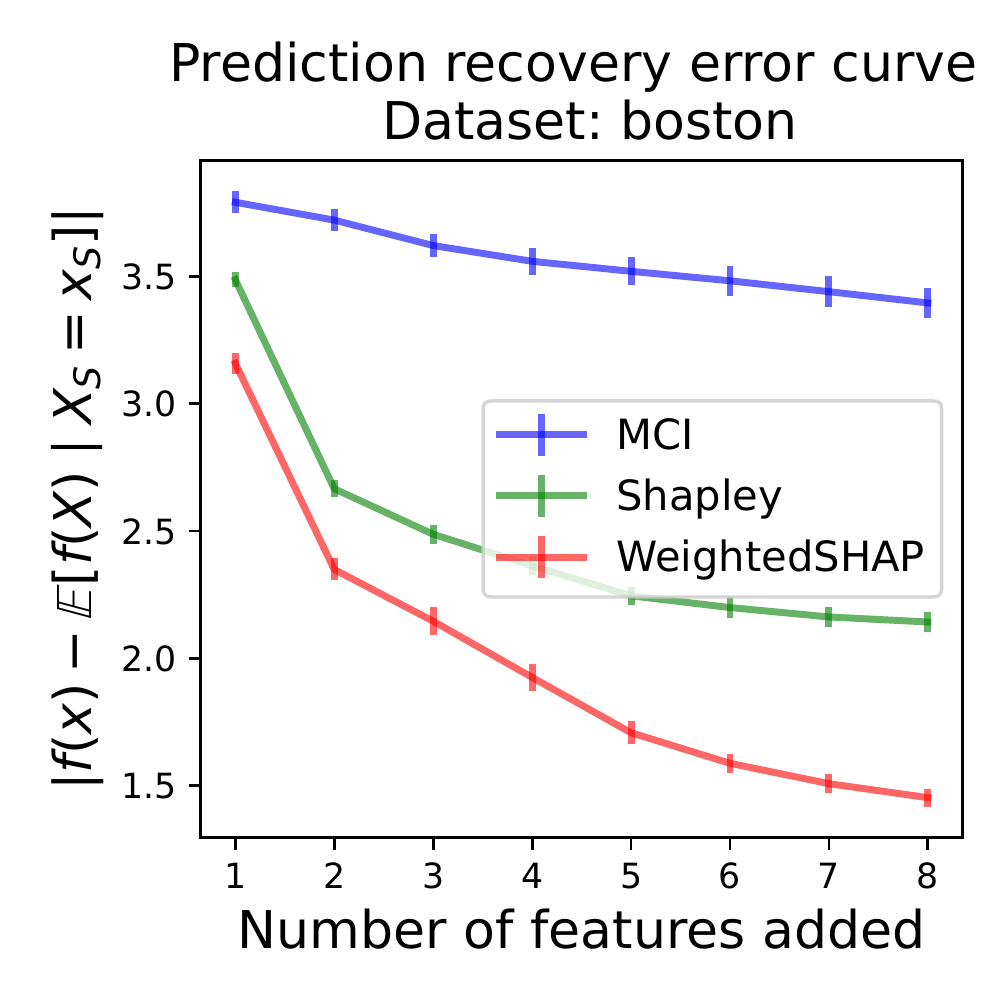}
    \includegraphics[width=0.22\textwidth]{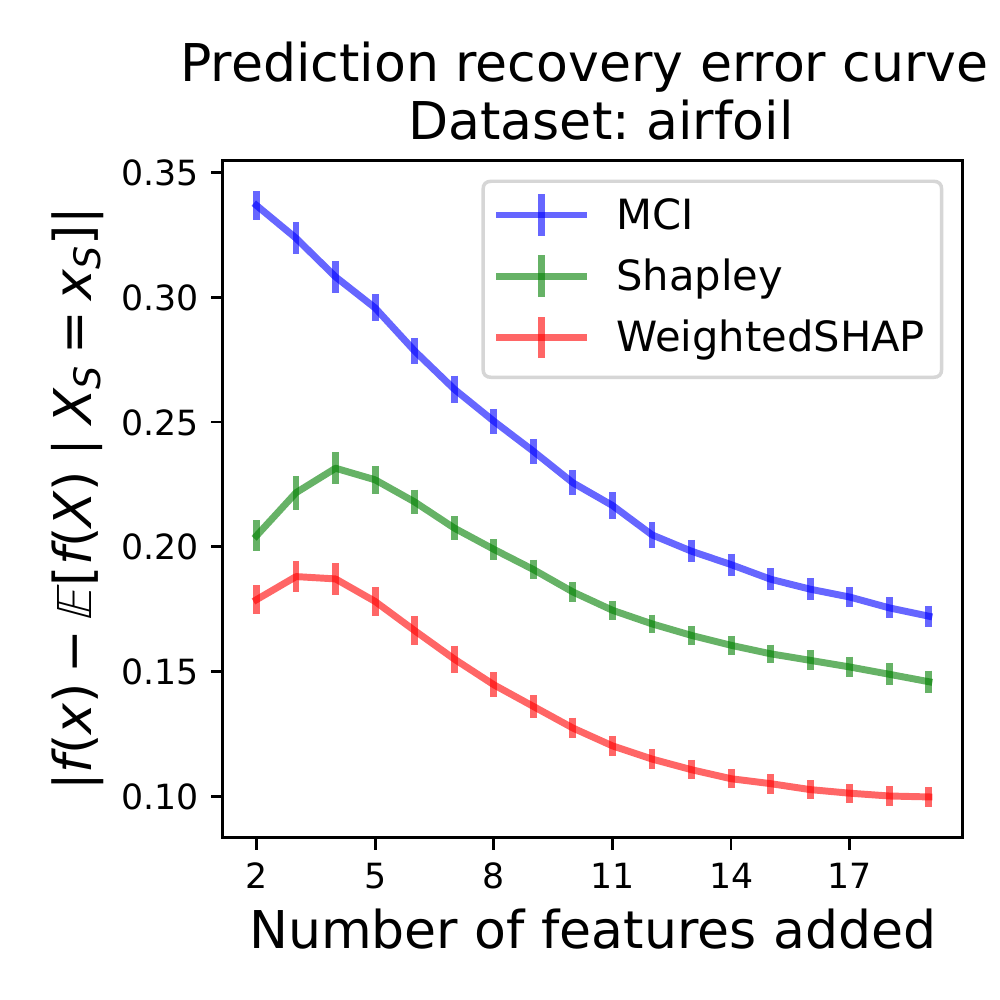}
    \includegraphics[width=0.22\textwidth]{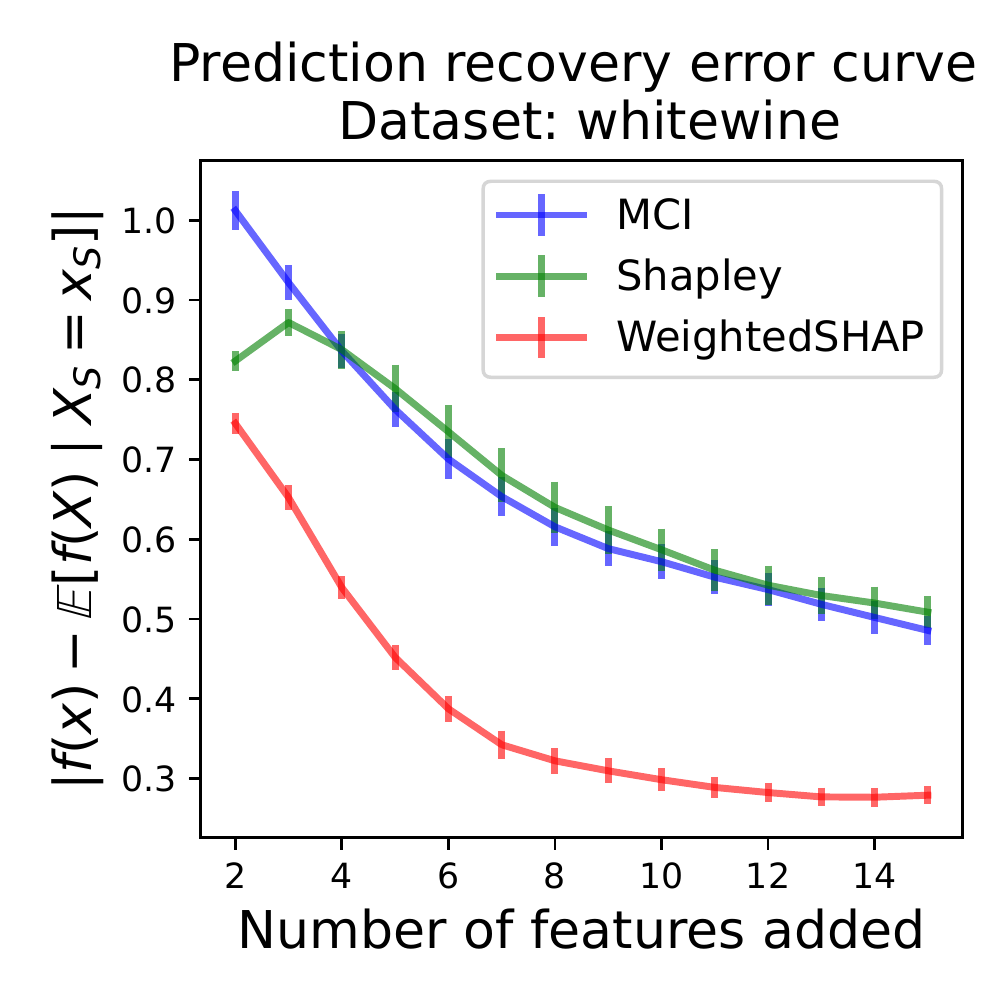}
    \includegraphics[width=0.22\textwidth]{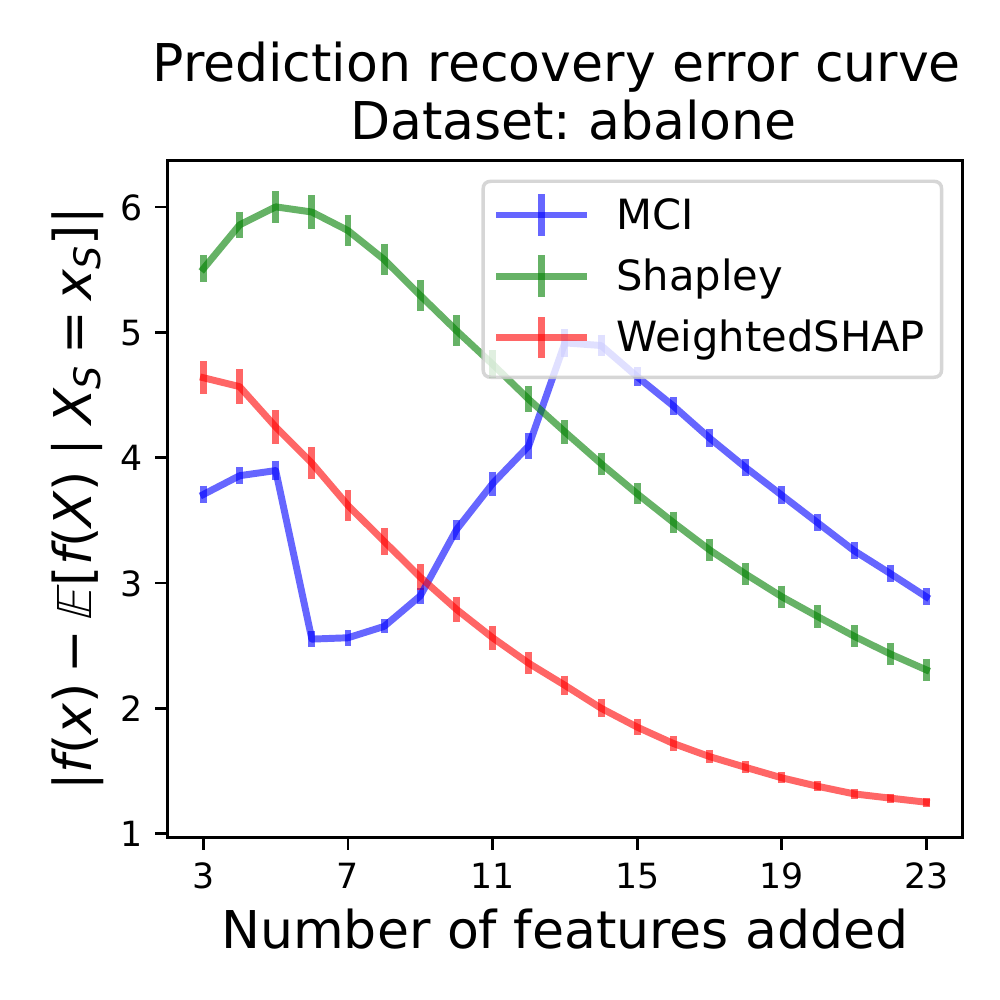}
    \label{fig:prediction_error_regression_boosting}
    }
    \subfigure[Illustrations of the Inclusion MSE curve on the four regression datasets. The lower, the better.]{
    \includegraphics[width=0.22\textwidth]{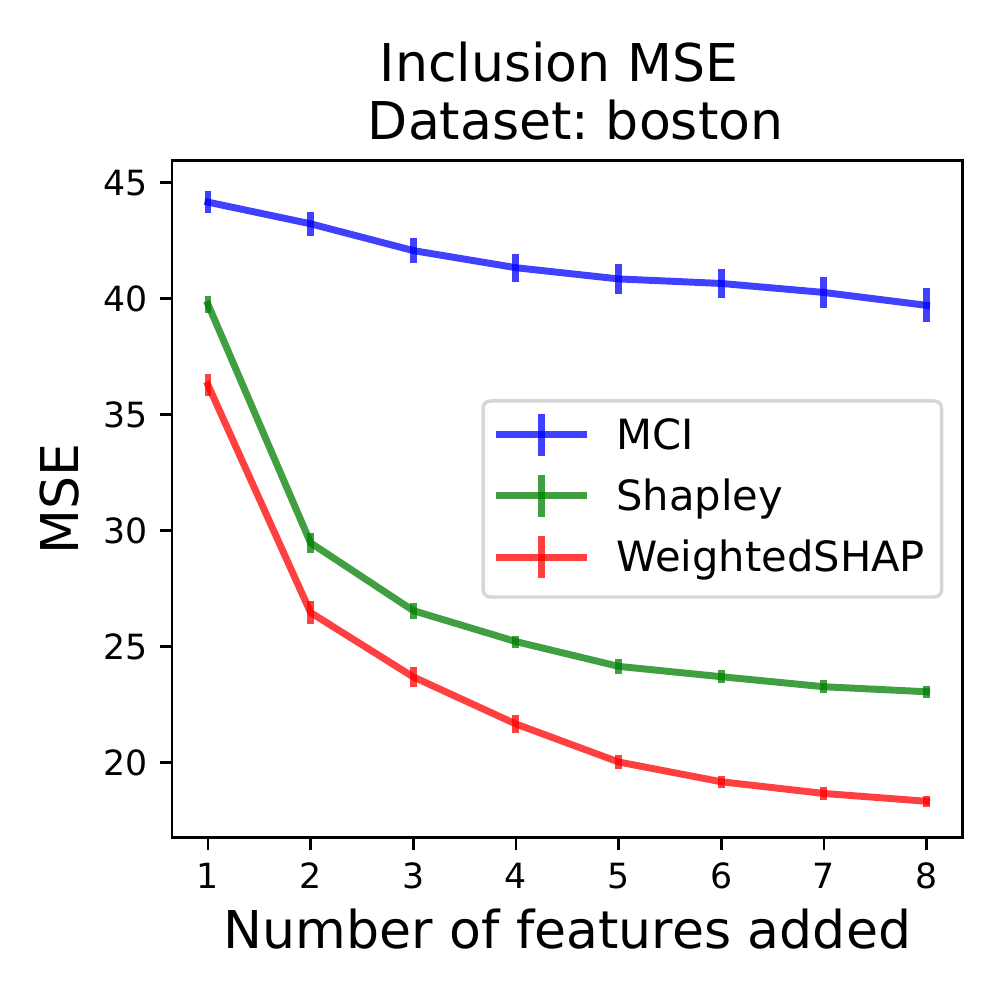}
    \includegraphics[width=0.22\textwidth]{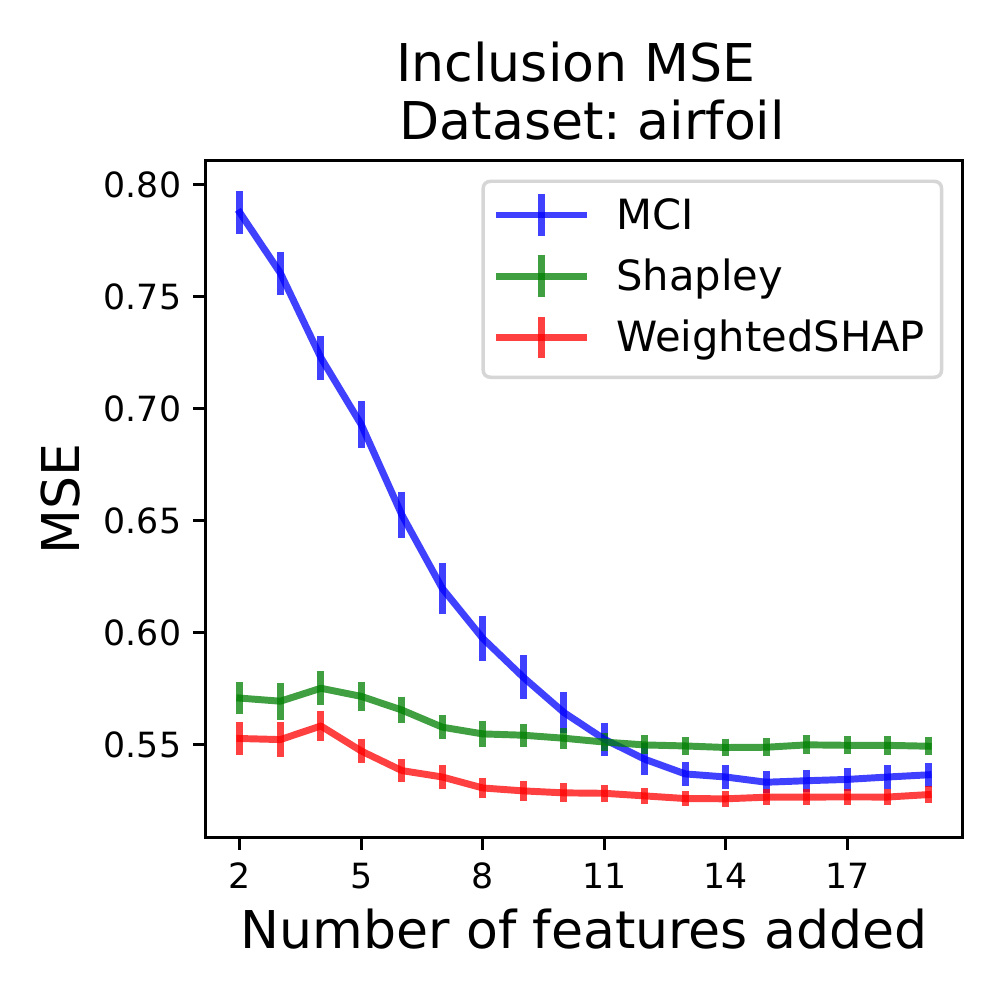}
    \includegraphics[width=0.22\textwidth]{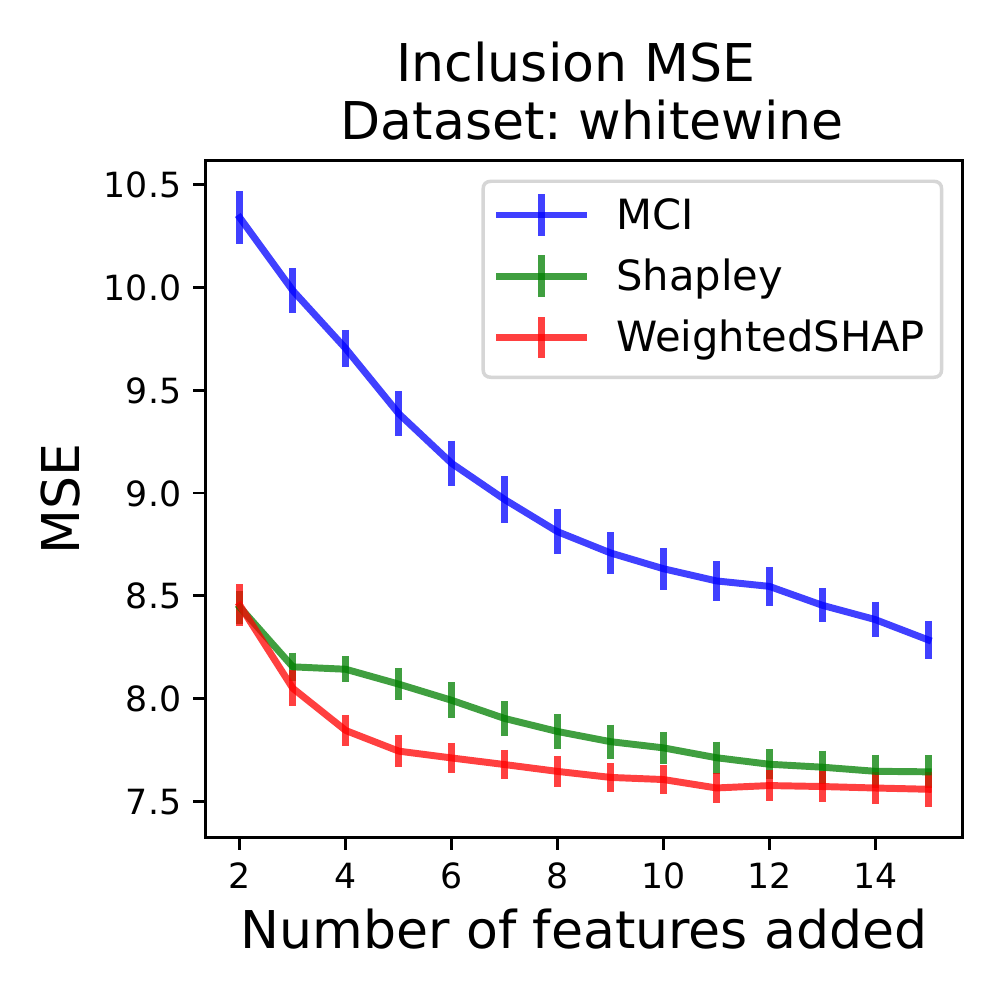}
    \includegraphics[width=0.22\textwidth]{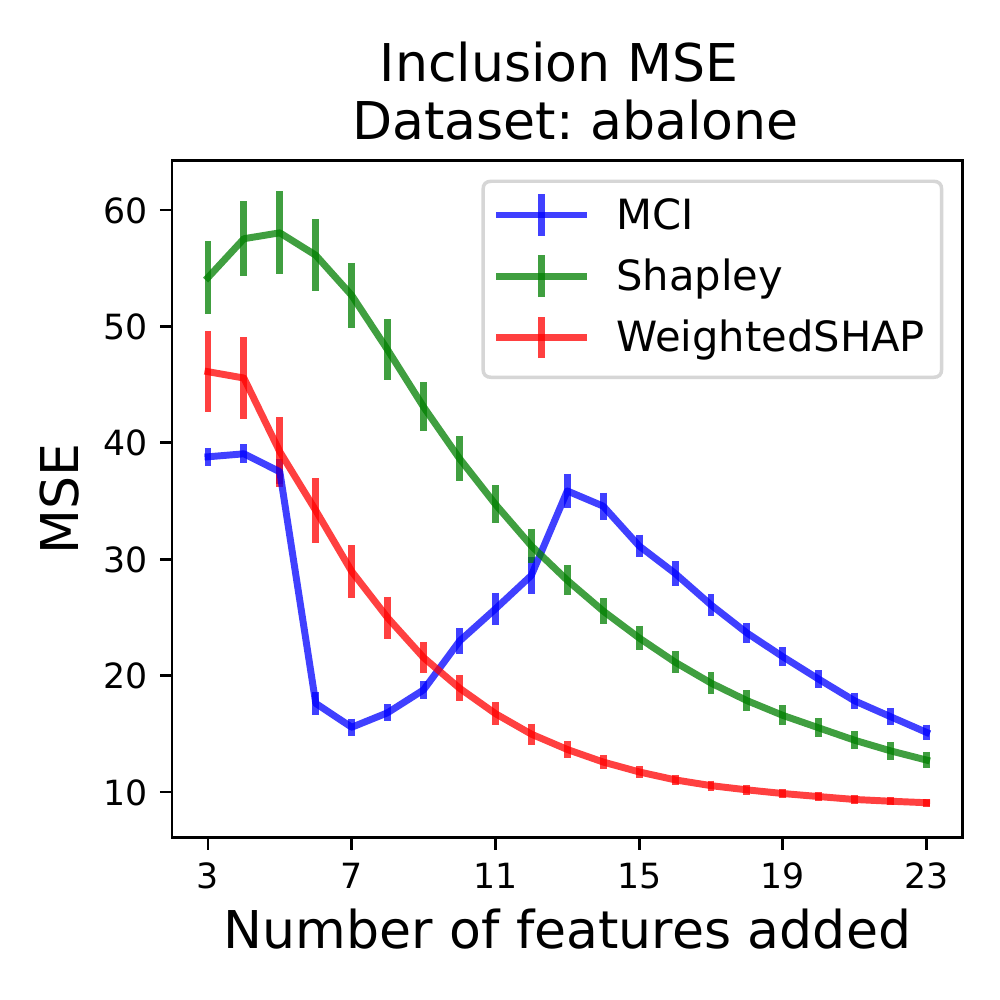}
    \label{fig:inclusion_auc_regression_boosting}
    }
    \caption{\textbf{Regression tasks.} Illustrations of the prediction recovery error curve and the Inclusion MSE curve as a function of the number of features added. We add features from most influential to the least influential. We denote a 95\% confidence interval based on 50 independent runs. WeightedSHAP achieves a significantly smaller MSE with fewer features than the MCI and the Shapley value.}
    \label{fig:regression_boosting_all}
\end{figure}

\begin{figure}[t]
    \centering
    \subfigure[Illustrations of the prediction recovery error curve on the four binary classification datasets. The lower, the better.]{
    \includegraphics[width=0.22\textwidth]{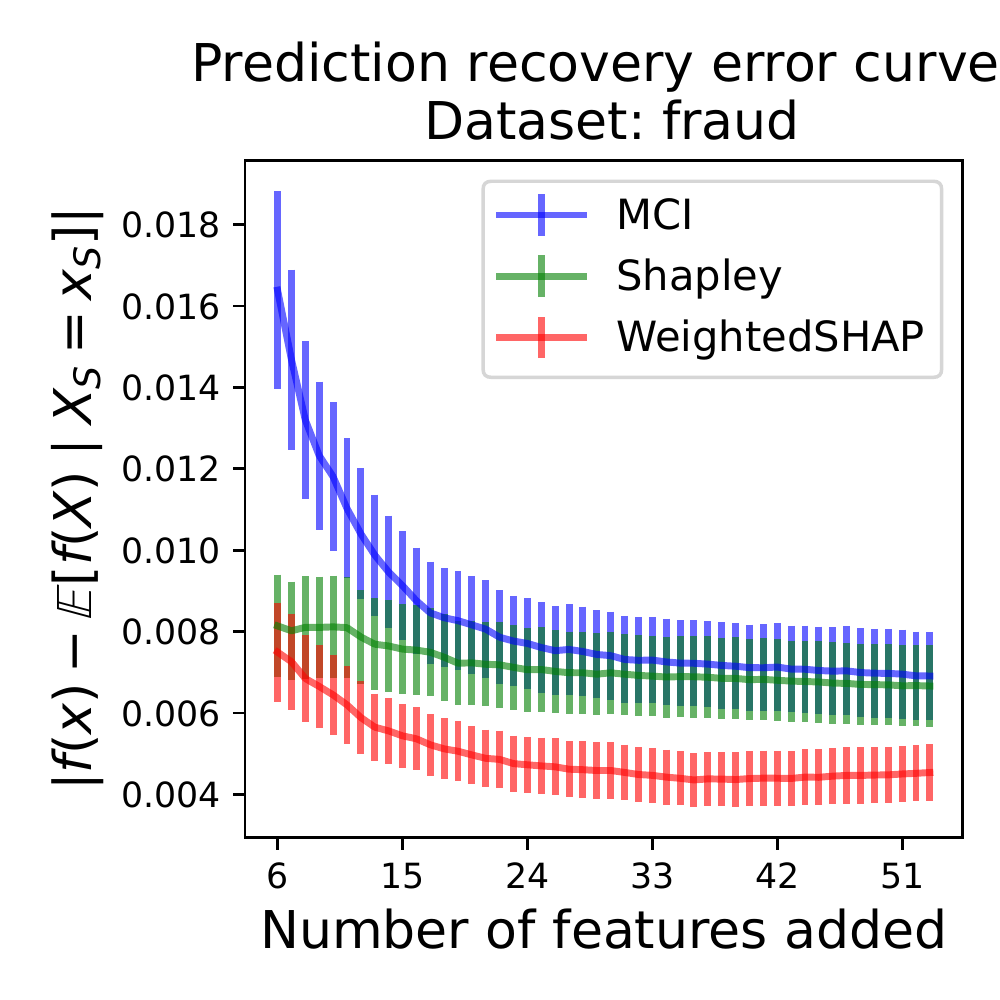}
    \includegraphics[width=0.22\textwidth]{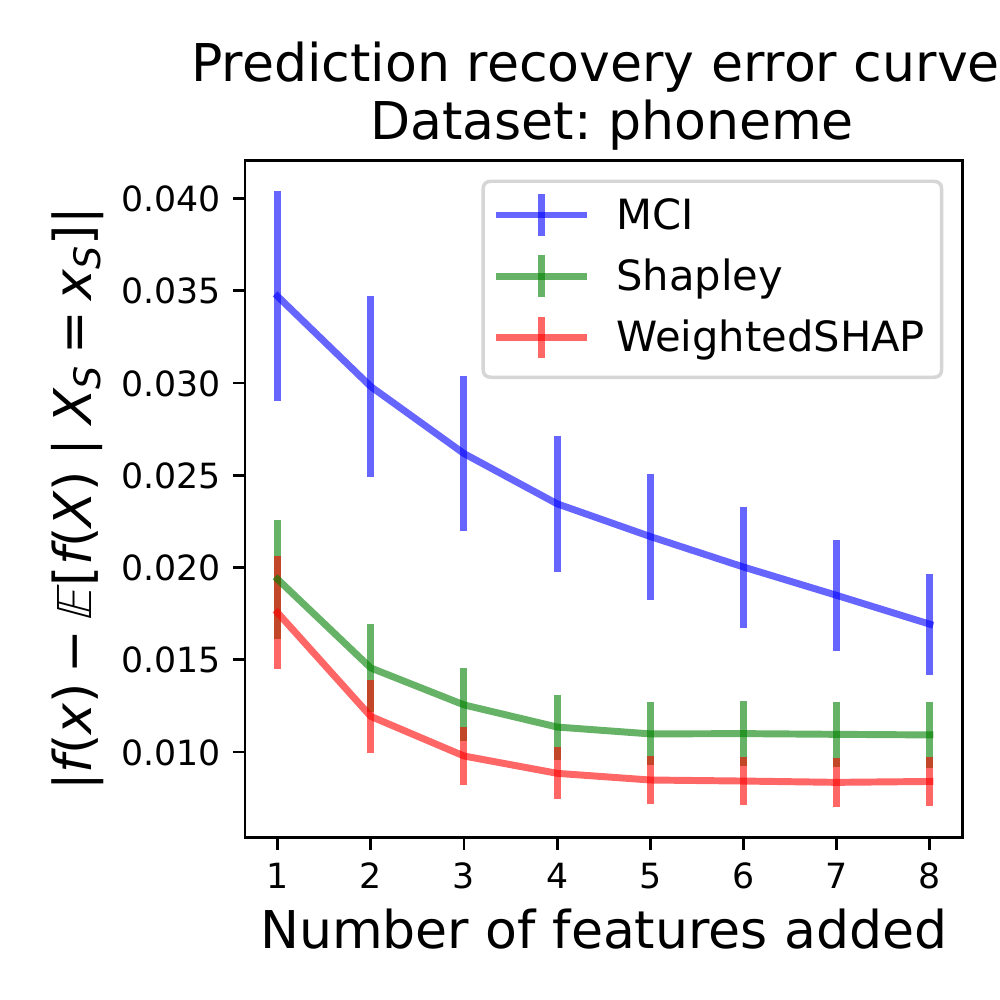}
    \includegraphics[width=0.22\textwidth]{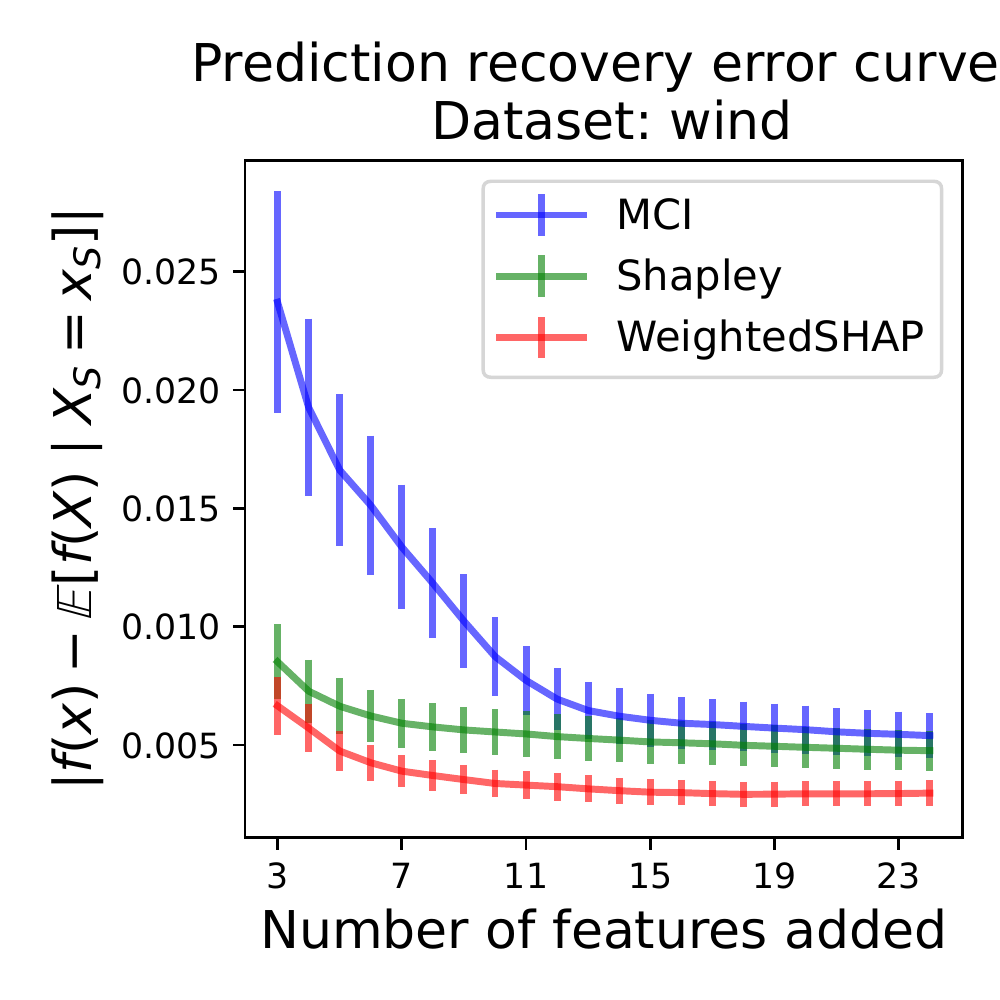}
    \includegraphics[width=0.22\textwidth]{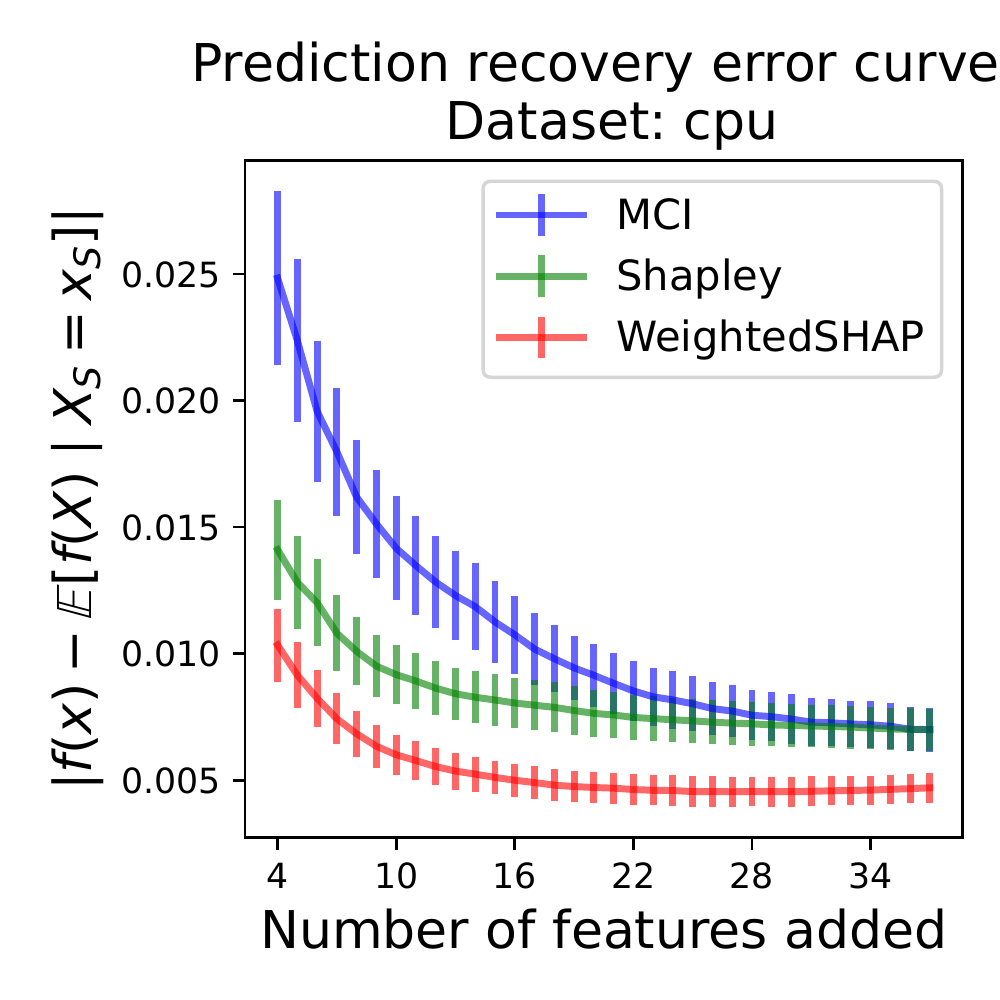}
    \label{fig:prediction_error_classification_boosting}
    }
    \subfigure[Illustrations of the Inclusion AUC curve on the four binary classification datasets. The higher, the better.]{
    \includegraphics[width=0.22\textwidth]{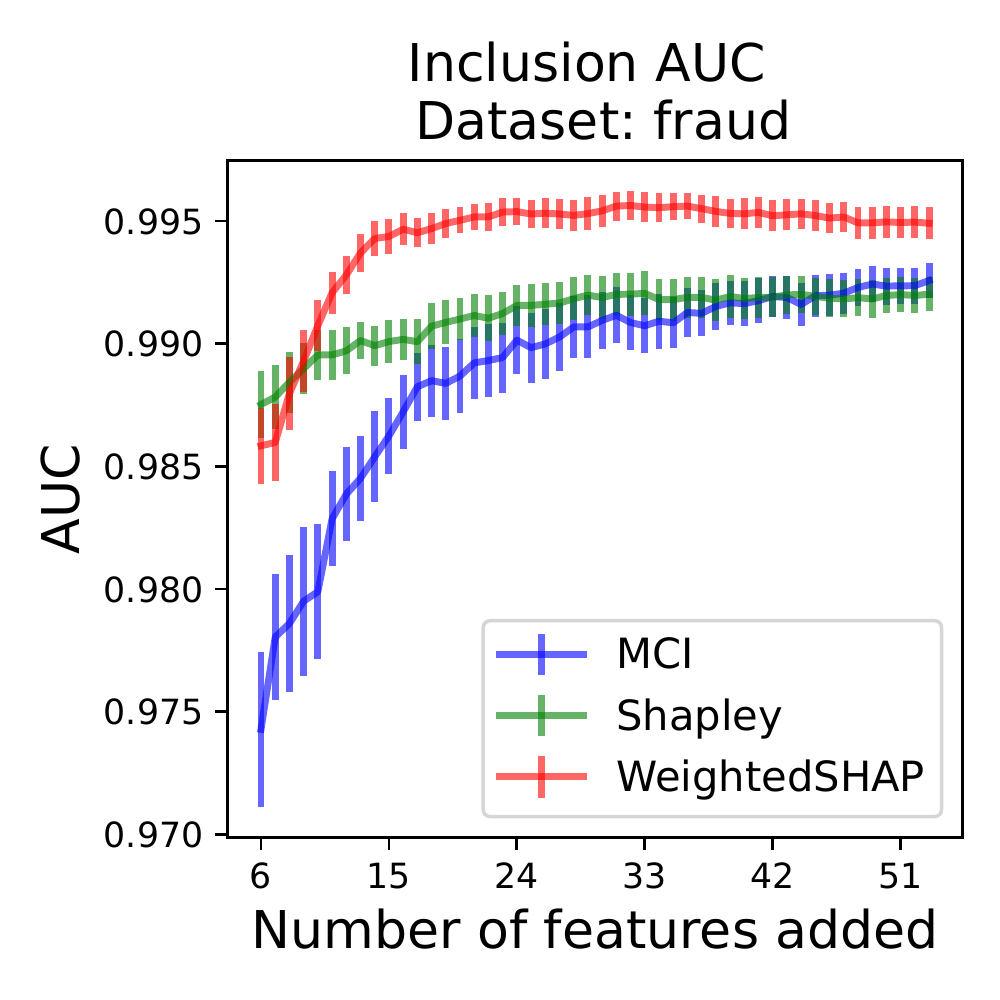}
    \includegraphics[width=0.22\textwidth]{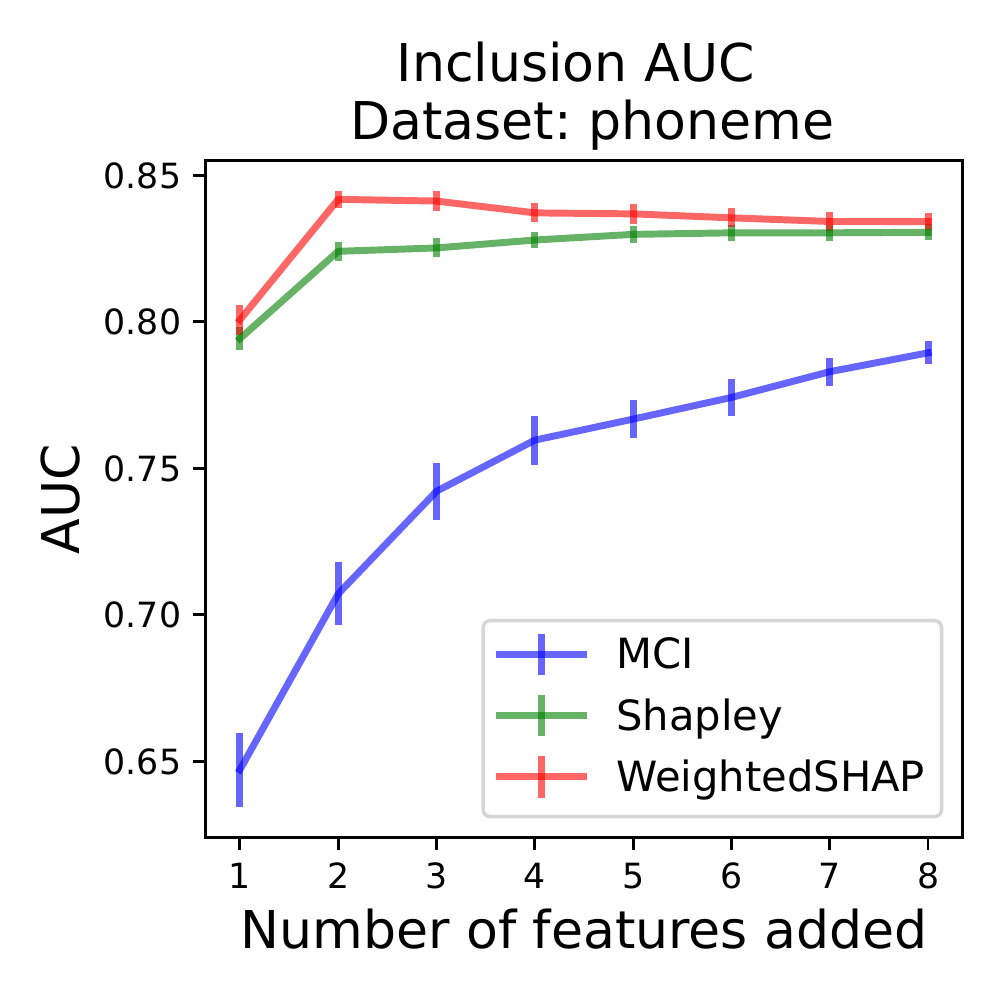}
    \includegraphics[width=0.22\textwidth]{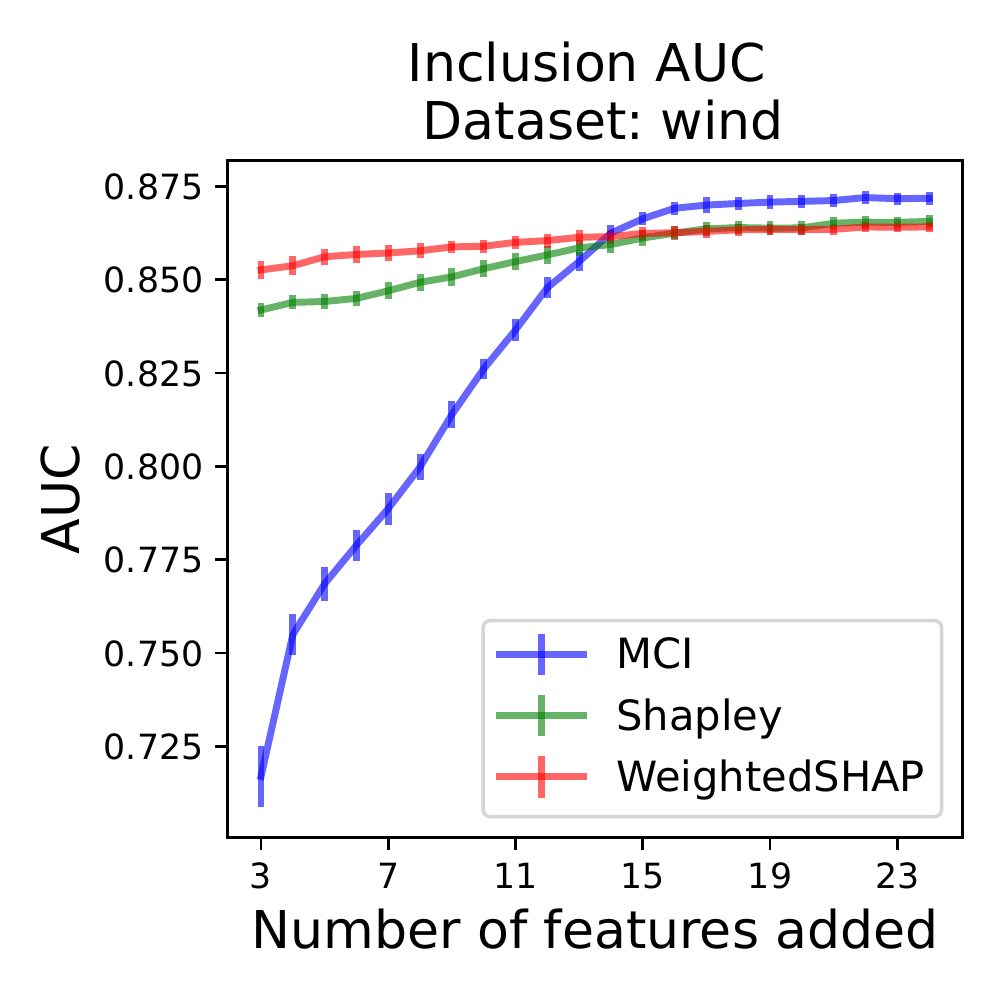}
    \includegraphics[width=0.22\textwidth]{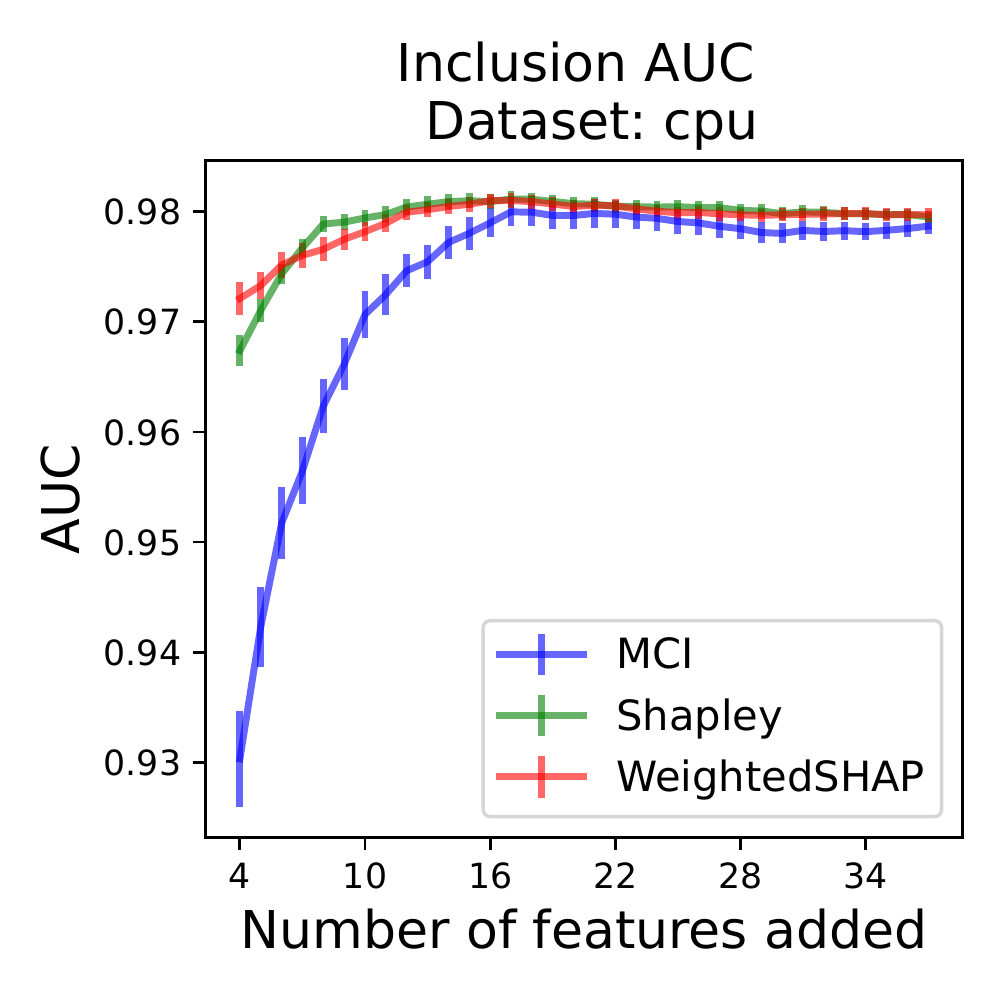}
    \label{fig:inclusion_auc_classification_boosting}
    }
    \caption{\textbf{Classification tasks.} Illustrations of the prediction recovery error curve and the Inclusion AUC curve as a function of the number of features added. Details are provided in Figure~\ref{fig:regression_boosting_all}. WeightedSHAP achieves a significantly higher AUC with fewer features than the MCI and the Shapley value.}
\end{figure}

\section{Experimental results}
\label{sec:experiment}

We demonstrate the practical efficacy of WeightedSHAP on various regression and classification datasets. We compare WeightedSHAP with the marginal contribution feature importance (MCI) proposed by \citet{catav21marginal} and the Shapley value on the prediction recovery error task and the inclusion performance task. Each task assesses the goodness of an attribution order by iteratively measuring how much the original model prediction or its performance is recovered with a given number of features. As for the model performance, we evaluate mean squared error (MSE) and AUC for regression and classification problems, respectively.
We consider a gradient boosting model for a prediction model $\hat{f}$. As for the surrogate model in coalition function estimation $v^{\mathrm{(cond)}}$, we use a multilayer perceptron model with the two hidden layers, and each layer has 128 neurons and the ELU activation function \citep{clevert2015fast}. As for the WeightedSHAP in \eqref{eqn:weightedSHAP}, we use the negative value of the AUP for $\mathcal{T}$ and a set $\mathcal{W}$ that has 13 different weights including $\Delta_1$, $\phi_{\mathrm{shap}}$, and $\Delta_d$. All the missing details about numerical experiments are provided in Appendix, and our Python-based implementations are available at \url{https://github.com/ykwon0407/WeightedSHAP}.

Figure~\ref{fig:prediction_error_regression_boosting} compares the prediction recovery error curves for the WeightedSHAP (described in red) with the MCI (described in blue) and the Shapley value (described in green). WeightedSHAP shows always lower prediction recovery errors than the MCI and the Shapley value. Given that WeightedSHAP minimizes the AUP, which is the sum of prediction recovery error $|\hat{f}(x) - \mathbb{E}[\hat{f}(X) \mid X_S=x_S]|$, WeightedSHAP does not necessarily have a smaller prediction recovery error for every number of features added. As for the MSE, Figure~\ref{fig:inclusion_auc_regression_boosting} shows that WeightedSHAP has a significantly smaller MSE than baseline methods with fewer features. In particular, on the \texttt{airfoil} dataset, WeightedSHAP achieves $0.53$ MSE with 10 features, but the Shapley value never achieves this value because of the suboptimality of the attribution order.

We also evaluate the prediction recovery error and AUC for the four classification datasets. Similar to the regression cases, Figures~\ref{fig:prediction_error_classification_boosting} and~\ref{fig:inclusion_auc_classification_boosting} show that WeightedSHAP effectively assigns larger values for more influential features and recovers the original prediction $\hat{f}(x)$ significantly faster than the MCI and the Shapley value. Specifically, on the \texttt{fraud} dataset, WeightedSHAP achieves $0.995$ AUC with 14 features, but the Shapley value always has the lower AUC value. Our findings are consistently observed with a different model for $\hat{f}$ or other datasets. Additional experimental results with different evaluation metrics and a qualitative assessment of WeightedSHAP are provided in Appendix.

\subsection{Illustrative examples from MNIST}
\label{app:illustrative_examples_mnist}

We present a qualitative assessment of WeightedSHAP and examine how its top influential features differ from those from SHAP using the MNIST dataset. We train a convolutional neural network model using the same setting suggested in \citet{jethani2021fastshap}. It achieves 98.6 \% accuracy on the test dataset. We select illustrative images with a significant difference in AUPs between WeightedSHAP and SHAP.

Figure~\ref{fig:MNIST_samples} compares the top 10\% influential attributions for WeightedSHAP and the Shapley value. On the top images, while SHAP fails to capture the last stroke of digit nine, which is a crucially important stroke to differentiate from the digit zero, WeightedSHAP clearly captures the strokes. On the bottom images, SHAP produces unintuitive negative attributions, providing noisy explanations. In contrast, WeightedSHAP provides noiseless and intuitive explanations.

\begin{figure}[t]
    \centering
    \includegraphics[width=0.425\textwidth]{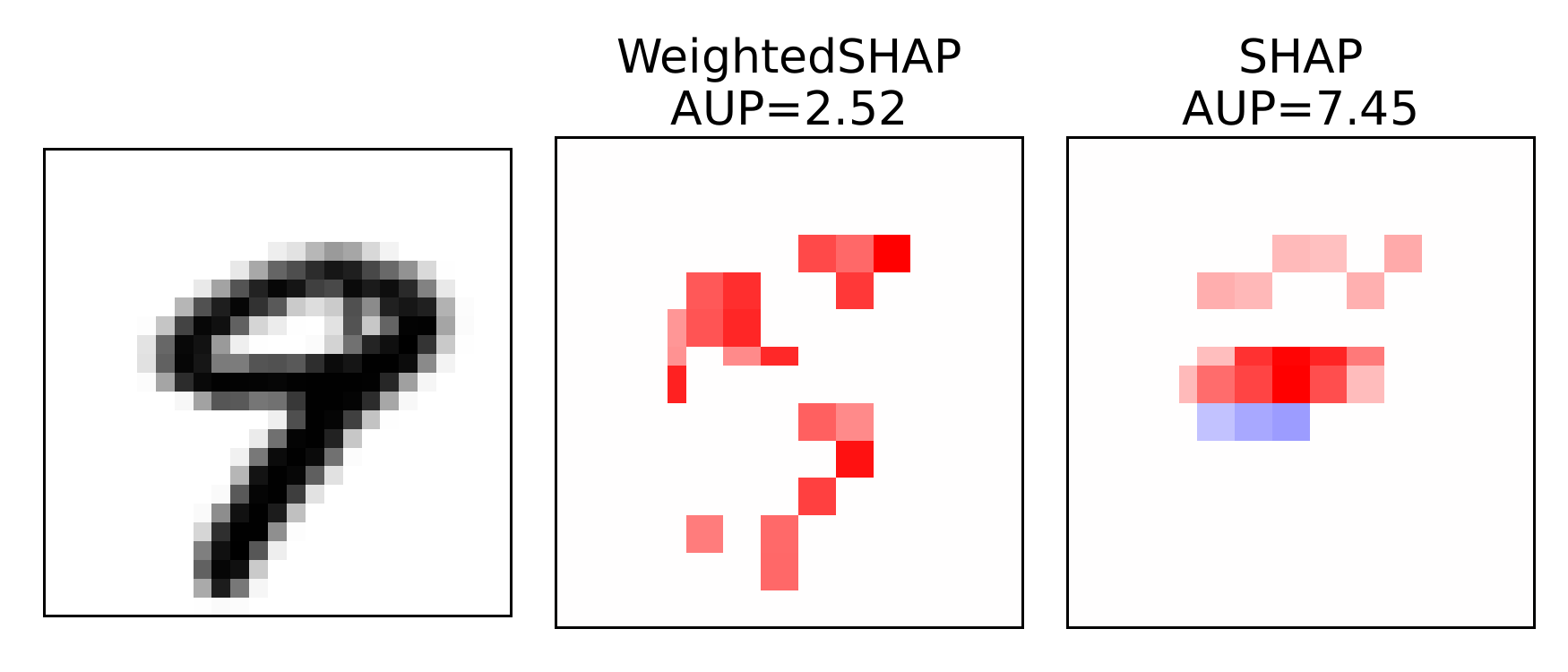}
    \includegraphics[width=0.45\textwidth]{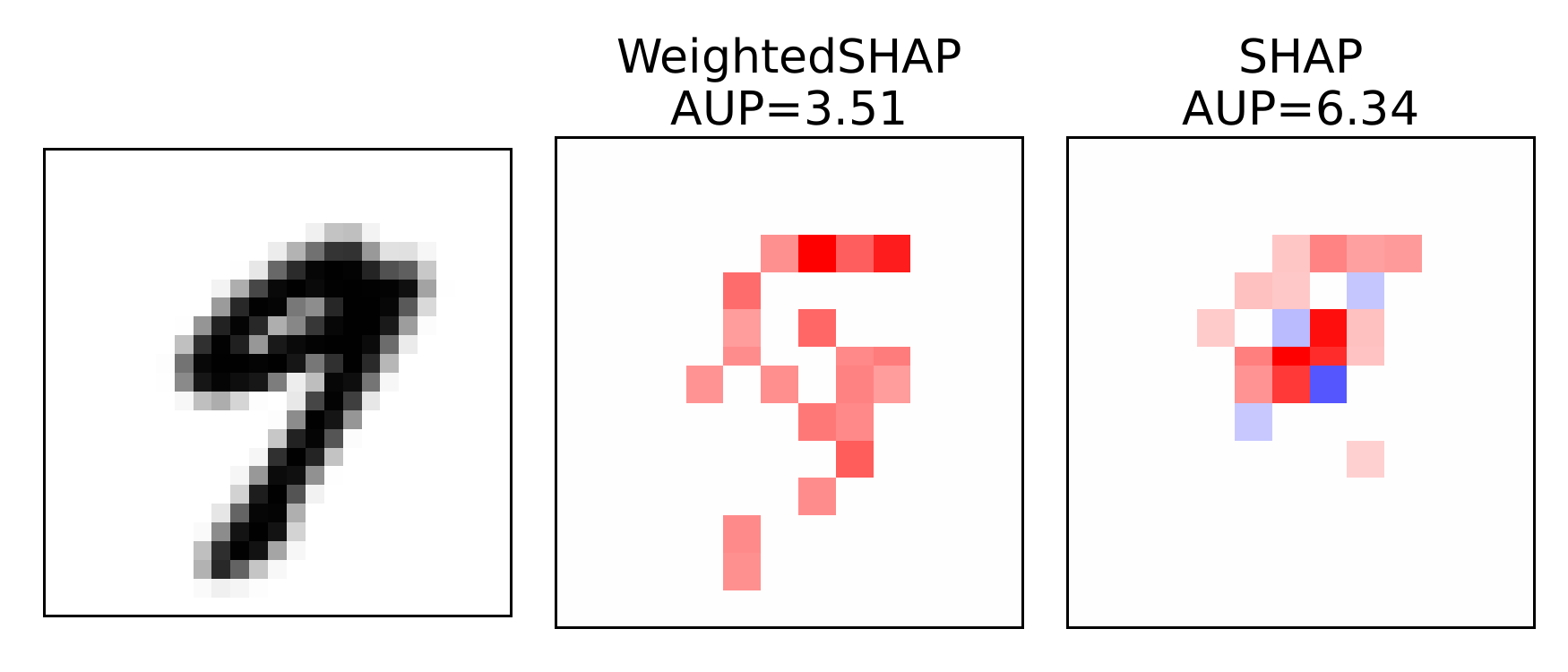}\\
    \includegraphics[width=0.425\textwidth]{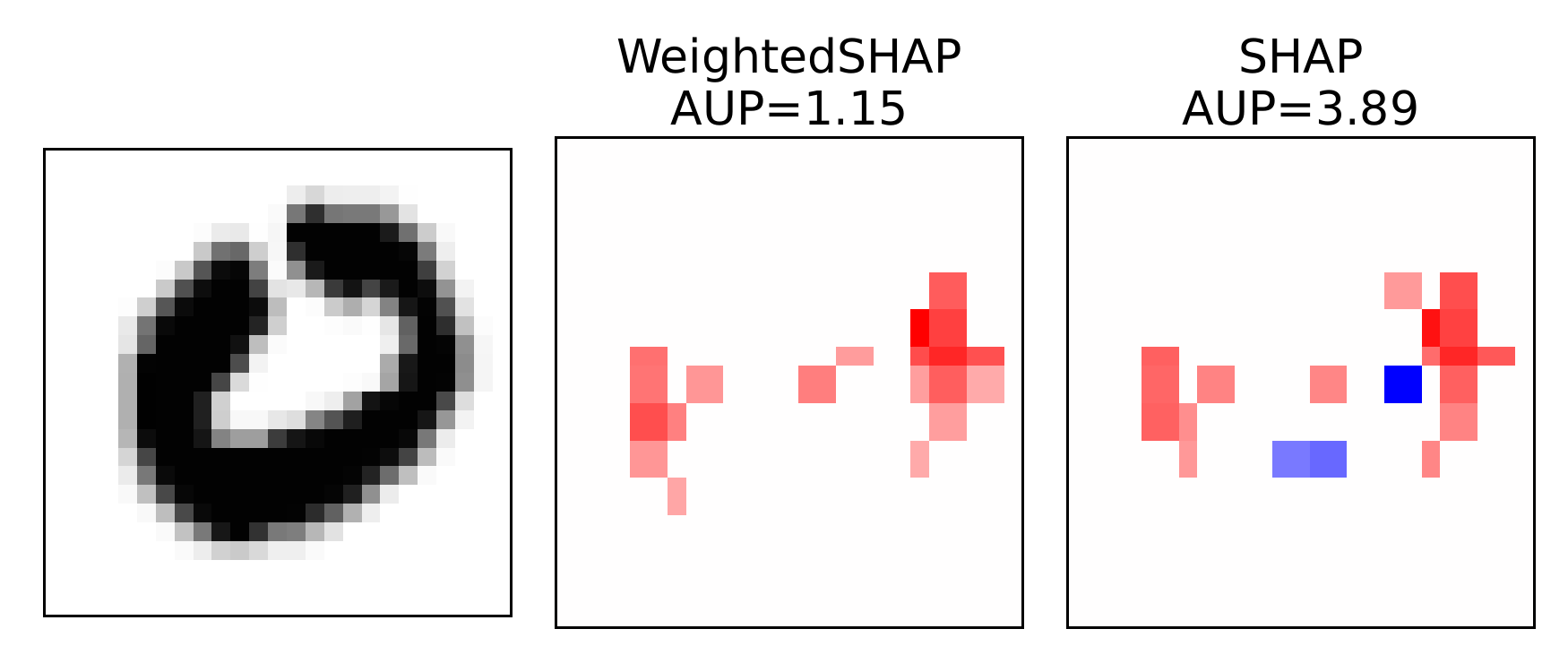}
    \includegraphics[width=0.45\textwidth]{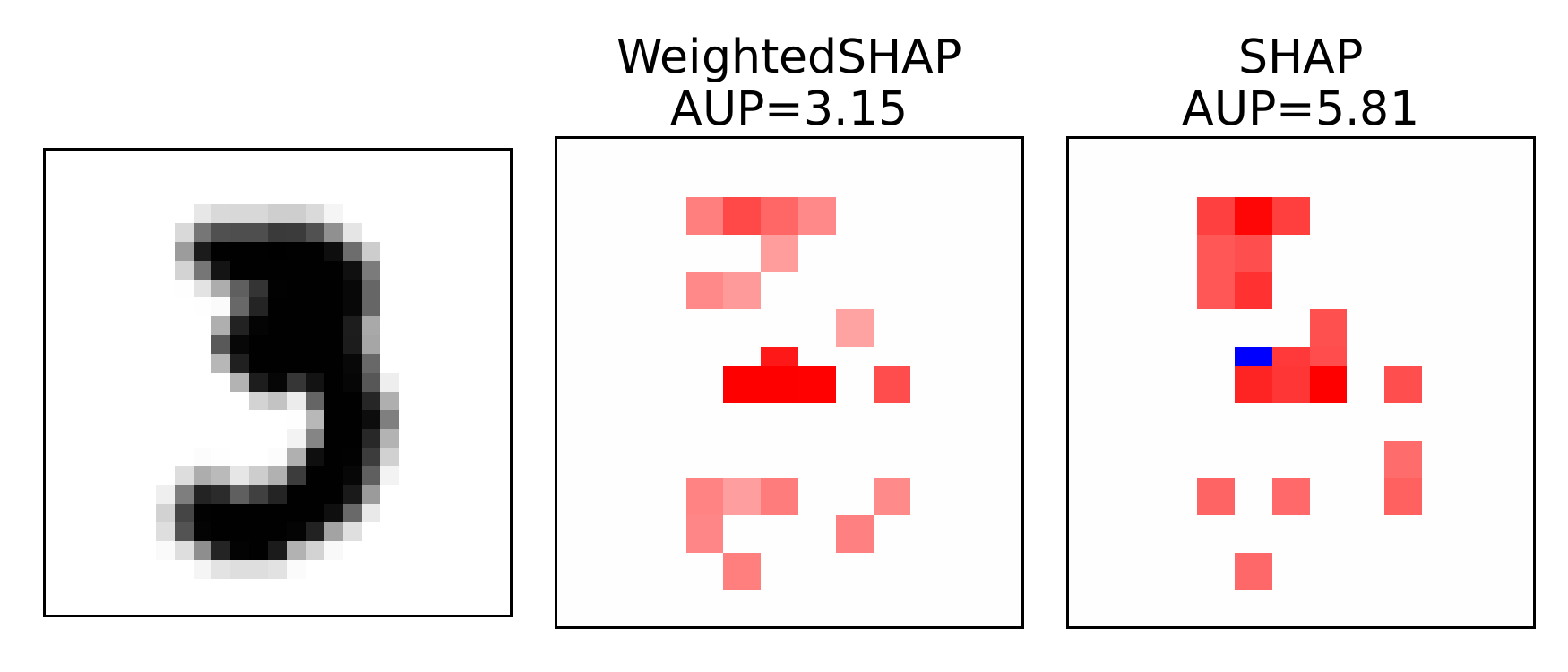}
    \caption{Illustrative examples of WeightedSHAP and SHAP attributions on MNIST images. We present the top 10\% of influential features. Red ({\it resp.} blue) color indicates the corresponding feature positively ({\it resp.} negatively) affects the model prediction. (Top) WeightedSHAP clearly captures the last stroke of nine while SHAP fails to capture it. (Bottom) While SHAP has noisy negative feature attributions described by blue pixels, WeightedSHAP provides noiseless and intuitive explanations.}
    \label{fig:MNIST_samples}
\end{figure}

\section{Conclusion}
\label{sec:conclusion}
In this paper, we provide an analysis of the widely used SHAP attribution method. We discover that even in simple natural settings, SHAP can incorrectly identify important features. Mathematically, a key limitation of SHAP is in that it assigns uniform weights to all marginal contributions. We propose WeightedSHAP which generalizes the Shapley value by relaxing the efficiency axiom. WeightedSHAP learns to pay more attention to the marginal contributions that have more signal on a prediction, assigning larger attributions for more influential features. There are several limitations of WeightedSHAP that motivate interesting future works. Here we use the AUP metric to optimize the weights because AUP is commonly used in practice. However, there is no agreed-upon metric for evaluating feature attribution methods. Different users may care about different notions of attribution. Developing variants of marginal contribution weighting optimized for different applications is an important direction of future research. We believe that the core contribution of this paper---that the uniform weighting used by SHAP can be suboptimal---still provides useful insights for these investigations. Here we focus our experiments on directly comparing WeightedSHAP with SHAP because our goal is to characterize the limitations of SHAP. There is a large body of works comparing SHAP with other attribution methods that are complementary to our work \citep{yeh2019fidelity, jethani2021fastshap}.    

\section*{Acknowledgment}
The authors would like to thank all anonymous reviewers for their constructive comments. We also would like to thank Ruishan Liu for the helpful discussion.

\appendix
\section*{Appendix}

In appendix, we provide proofs in Section~\ref{app:Proof}, a detailed discussion on the Shapley value and the semivalue in Section~\ref{app:shap_and_semi}, implementation details in Section~\ref{app:imple_details}, and additional experimental results in Section~\ref{app:add_results}. Our main findings presented in the manuscript are consistently observed with different datasets and prediction models: the influential features identified by WeightedSHAP better recapitulate the model's predictions compared to the features identified by the Shapley value. A Python-based implementation code is attached as a part of the supplementary material.

\section{Proofs}
\label{app:Proof}
In this section, we provide a proof of Theorem~\ref{thm:marginal_contribution}. We first prove a lemma that shows a closed-form expression of the marginal contribution when a covariance matrix has a exchangeable correlation structure, \textit{i.e.}, $\Sigma_{(\rho,d)} =  (1-\rho) I_{d} + \rho \mathds{1}_{d} \mathds{1}_{d} ^T$. To begin with, for $d \in \mathbb{N}$ and $S \in [d]$, we set a $d$-dimensional vector $e_S \in\{0,1\}^{d}$ whose $i$-th element is one if $i \in S$, otherwise zero. We set $e_i := e_{\{i\}}$ for $i \in [d]$.

\begin{lemma}[A closed-form expression of the Marginal contribution when a covariance matrix has an exchangeable correlation structure.]
For $d \in \mathbb{N}$, suppose $\hat{f}(x) = \hat{\beta}_0 + x^T \hat{\beta}$ for some $(\hat{\beta}_0, \hat{\beta}) \in \mathbb{R} \times \mathbb{R}^{d}$ and $X \sim \mathcal{N} (\mathbf{0}_d, \Sigma_{(\rho,d)})$ for $\rho \in [0,1)$. Then for $i, j \in [d]$, the marginal contribution of the $i$-th feature $x_i$ with respect to $j-1$ samples is given as
\begin{align*}
    \Delta_{j}(x_i) = x^T h_d(i, j) \hat{\beta}, 
\end{align*}
where
\begin{align*}
    &h_d(i,j) \\
    &:= e_i e_i^T + \frac{\rho}{1 + \rho (j-1)}  \frac{d-j}{d-1} e_i (\mathds{1}_d - e_i)^T - \frac{\rho}{1 -\rho + \rho (j-1)} \frac{j-1}{d-1} (\mathds{1}_d - e_i) e_i^T \\
    &+ \left( \frac{\rho}{1 + \rho (j-1)} - \frac{\rho}{1 - \rho + \rho (j-1)} \right) \frac{(j-1)(d-j)}{(d-1)(d-2)}  \left( (\mathds{1}_d-e_i) (\mathds{1}_d-e_i)^T - (I_d - e_i e_i ^T) \right).
\end{align*}
\label{lemma:marginal_contrib}
\end{lemma}

\begin{proof}[Proof of Lemma~\ref{lemma:marginal_contrib}]
Without loss of generality, we set $\hat{\beta}_0 = 0$ as a constant intercept does not affect the marginal contribution. We denote $\Sigma_{S, S} := (1-\rho) I_{|S|} + \rho \mathds{1}_{|S|} \mathds{1}_{|S|} ^T$. Then, for any $S$, we have 
\begin{align*}
    v^{\mathrm{(cond)}} (S) &= x_S ^T \hat{\beta}_S + (\Sigma_{D \backslash S,S} \Sigma_{S,S} ^{-1} x_S) ^T \hat{\beta}_{D \backslash S} \\ 
    &=x_S ^T \left( \hat{\beta}_S + \left(\frac{\rho}{1-\rho} - \frac{|S| \rho^2}{(1-\rho)(1+\rho(|S|-1)) } \right)  \mathds{1}_{|S|} \mathds{1}_{d-|S|}  ^T \hat{\beta}_{D \backslash S} \right) \\ 
    &=x_S ^T  \hat{\beta}_S + \frac{\rho}{1-\rho + \rho |S|} x_S ^T \mathds{1}_{|S|} \mathds{1}_{d-|S|}  ^T \hat{\beta}_{D \backslash S}\\
    &=x_S ^T  \hat{\beta}_S + \frac{\rho}{1-\rho + \rho |S|} x ^T e_{S} e_{S^c} ^T \hat{\beta}.
\end{align*}
Here, the second equality is due to $\Sigma_{S, S} ^{-1} = (1-\rho)^{-1}  I_{|S|} - \frac{\rho}{(1-\rho)(1+\rho(|S|-1))} \mathds{1}_{|S|} \mathds{1}_{|S|} ^T$. Hence, we have
\begin{align*}
    v^{\mathrm{(cond)}}(S\cup \{i\}) - v^{\mathrm{(cond)}}(S) = x_i \hat{\beta}_i  - \frac{\rho}{1 -\rho + \rho |S|} x ^T e_{S} e_{S^c} ^T  \hat{\beta} + \frac{\rho}{1 + \rho |S|} x ^T e_{S\cup \{i\} } e_{(S\cup \{i\})^c} ^T\hat{\beta}.
\end{align*}

[Step 1] Since $e_{S} e_{S^c} ^T =  e_{S} (\mathds{1}_d - e_{S} )^T$.
\begin{align*}
    \frac{1}{\binom{d-1}{j-1}}\sum_{S \subseteq [d] \backslash\{i\}, |S|=j-1 } e_{S} e_{S^c} ^T &= \frac{j-1}{d-1} (\mathds{1}_d-e_i) \mathds{1}_d^T - \frac{1}{\binom{d-1}{j-1}}\sum_{S \subseteq [d] \backslash\{i\}, |S|=j-1 } e_{S} e_{S} ^T \\
    &= \frac{j-1}{d-1} (\mathds{1}_d-e_i) \mathds{1}_d^T - P(i,j),
\end{align*}
where, for $k,l \in [d]$, $P(i,j)_{kl}$ is given by
\begin{align*}
    P(i,j)_{kl} = \begin{cases}
    \frac{\binom{d-2}{j-2}}{\binom{d-1}{j-1}} =\frac{j-1}{d-1} & \quad \mathrm{if} \quad k = l, i \notin \{k,l\}\\
    \frac{\binom{d-3}{j-3}}{\binom{d-1}{j-1}} = \frac{(j-1)(j-2)}{(d-1)(d-2)} & \quad \mathrm{if} \quad k \neq l, i \notin \{k,l\}\\
    0 & \quad \mathrm{if} \quad i \in \{k,l\}\\
    \end{cases}.
\end{align*}
That is, $P(i,j) = \frac{(j-1)(j-2)}{(d-1)(d-2)} (\mathds{1}_d-e_i) (\mathds{1}_d-e_i)^T + \frac{(j-1)(d-j)}{(d-1)(d-2)} (I_d - e_i e_i ^T)$ and
\begin{align*}
    &\frac{1}{\binom{d-1}{j-1}}\sum_{S \subseteq [d] \backslash\{i\}, |S|=j-1 } e_{S} e_{S^c} ^T \\
    &= \frac{j-1}{d-1} (\mathds{1}_d-e_i) \mathds{1}_d^T - \frac{(j-1)(j-2)}{(d-1)(d-2)} (\mathds{1}_d-e_i) (\mathds{1}_d-e_i)^T - \frac{(j-1)(d-j)}{(d-1)(d-2)} (I_d - e_i e_i ^T).
\end{align*}

[Step 2]
Since $e_{S \cup \{ i\} } e_{(S\cup \{ i\})^c} ^T =  (e_{S} +e_i) (\mathds{1}_d - e_{S}-e_i )^T = e_{S}  (\mathds{1}_d - e_{S} )^T - e_S e_i ^T +e_i (\mathds{1}_d - e_{S} )^T - e_i e_i ^T$,
\begin{align*}
    &\frac{1}{\binom{d-1}{j-1}}\sum_{S \subseteq [d] \backslash\{i\}, |S|=j-1 } e_{S \cup \{ i\} } e_{(S\cup \{ i\})^c} ^T \\
    &= \frac{j-1}{d-1} (\mathds{1}_d-e_i) \mathds{1}_d^T - P(i,j) - \frac{j-1}{d-1} (\mathds{1}_d - e_i) e_i^T + \frac{d-j}{d-1} e_i (\mathds{1}_d - e_i)^T\\
    &=\frac{1}{\binom{d-1}{j-1}}\sum_{S \subseteq [d] \backslash\{i\}, |S|=j-1 } e_{S} e_{S^c} ^T - \frac{j-1}{d-1} (\mathds{1}_d - e_i) e_i^T + \frac{d-j}{d-1} e_i (\mathds{1}_d - e_i)^T.
\end{align*}

[Step 3] Therefore, $\Delta_{j}(x_i) = x^T h_d(i, j) \hat{\beta}$ where
\begin{align*}
    &h_d(i,j) \\
    &= \frac{\rho}{1 + \rho (j-1)} \left( \frac{1}{\binom{d-1}{j-1}}\sum_{S \subseteq [d] \backslash\{i\}, |S|=j-1 } e_{S} e_{S^c} ^T - \frac{j-1}{d-1} (\mathds{1}_d - e_i) e_i^T + \frac{d-j}{d-1} e_i (\mathds{1}_d - e_i)^T \right) \\ 
    & +e_i e_i^T - \frac{\rho}{1 - \rho + \rho (j-1)} \left( \frac{1}{\binom{d-1}{j-1}}\sum_{S \subseteq [d] \backslash\{i\}, |S|=j-1 } e_{S} e_{S^c} ^T \right) \\ 
    &= e_i e_i^T + \left( \frac{\rho}{1 + \rho (j-1)} - \frac{\rho}{1 - \rho + \rho (j-1)} \right) \left( \frac{1}{\binom{d-1}{j-1}}\sum_{S \subseteq [d] \backslash\{i\}, |S|=j-1 } e_{S} e_{S^c} ^T \right) \\
    &+ \frac{\rho}{1 + \rho (j-1)} \left( - \frac{j-1}{d-1} (\mathds{1}_d - e_i) e_i^T + \frac{d-j}{d-1} e_i (\mathds{1}_d - e_i)^T \right)  \\
    &= e_i e_i^T + \frac{\rho}{1 + \rho (j-1)}  \frac{d-j}{d-1} e_i (\mathds{1}_d - e_i)^T - \frac{\rho}{1 -\rho + \rho (j-1)} \frac{j-1}{d-1} (\mathds{1}_d - e_i) e_i^T \\
    &+ \left( \frac{\rho}{1 + \rho (j-1)} - \frac{\rho}{1 - \rho + \rho (j-1)} \right) \frac{(j-1)(d-j)}{(d-1)(d-2)}  \left( (\mathds{1}_d-e_i) (\mathds{1}_d-e_i)^T - (I_d - e_i e_i ^T) \right).
\end{align*}
\end{proof}

Now we prove Theorem~\ref{thm:marginal_contribution}.
\begin{proof}[Proof of Theorem~\ref{thm:marginal_contribution}]
Without loss of generality, $i \in [d_1]$. That is, a feature of interest is in the first cluster. The key idea of the proof is to only consider the effect of the same cluster. That is, for all possible subsets $S \subseteq [d] \backslash\{i\}$ such that $|S|=j-1$, we can decompose $S = S^{(1)} \cup  S^{(2)}$ such that $S^{(1)} \cap  S^{(2)} =\emptyset$, $S^{(1)} \cap [d_1] = S^{(1)}$, and $S^{(2)} \cap [d_1] =\emptyset$. In addition, from the Vandermonde's identity, we have
\begin{align*}
    \sum_{k = \min\{d_1-j-1-d, 0\} } ^{\max\{d_1-1, j-1\}} \binom{d_1-1}{k} \binom{d-d_1}{j-1-k} = \binom{d-1}{j-1}.
\end{align*}
Hence, the marginal contribution is expressed as
\begin{align*}
    \Delta_j(x_i) &= \frac{1}{\binom{d-1}{j-1}}\sum_{S \subseteq [d] \backslash\{i\}, |S|=j-1 } v ^{\mathrm{(cond)}}(S\cup \{i\}) - v ^{\mathrm{(cond)}}(S) \\
    &= \frac{1}{\binom{d-1}{j-1}} \sum_{k = \min\{d_1-j-1-d, 0\} } ^{\max\{d_1-1, j-1\}} \sum_{S_1 \subseteq [d_1] \backslash\{i\}, |S| = k } v ^{\mathrm{(cond)}}(S_1 \cup \{i\}) - v ^{\mathrm{(cond)}}(S_1) \\
    &= \frac{1}{\binom{d-1}{j-1}} \sum_{k = \min\{d_1-j-1-d, 0\} } ^{\max\{d_1-1, j-1\}} \binom{d-d_1}{j-1-k} x_{[d_1]}^T h_{d_1}(i,k+1) \hat{\beta}_{[d_1]}\\
    &= x_{[d_1]}^T \left( \frac{1}{\binom{d-1}{j-1}} \sum_{k = \min\{d_1-j-1-d, 0\} } ^{\max\{d_1-1, j-1\}} \binom{d-d_1}{j-1-k}  h_{d_1}(i,k+1) \right) \hat{\beta}_{[d_1]}.
\end{align*}
Here, the second equality is because features in $S^{(2)}$ is independent of $S^{(1)}$ and the last equality is due to Lemma~\ref{lemma:marginal_contrib}. The last term can be expressed as $x^T H(i,j) \hat{\beta}$ by setting a block diagonal matrix
\begin{align*}
    H(i,j) = \mathrm{diag} \left( \left( \frac{1}{\binom{d-1}{j-1}} \sum_{k = \min\{d_1-j-1-d, 0\} } ^{\max\{d_1-1, j-1\}} \binom{d-d_1}{j-1-k}  h_{d_1}(i,k+1) \right), 0_{(d-d_1) \times (d-d_1) } \right).
\end{align*}
When $i$-th feature is in $m$-th cluster, then a similar method gives
\begin{align*}
    H(i,j) = \mathrm{diag} \Bigg( & 0_{(d_1+\cdots+d_{m-1}) \times (d_1+\cdots+d_{m-1}) },  \\
    &\left( \frac{1}{\binom{d-1}{j-1}} \sum_{k = \min\{d_m-j-1-d, 0\} } ^{\max\{d_m-1, j-1\}} \binom{d-d_m}{j-1-k}  h_{d_m}(i,k+1) \right), \\
    &0_{(d-d_1-\cdots-d_m) \times (d-d_1-\cdots-d_m) } \Bigg).
\end{align*}
It concludes a proof.
\end{proof}

\section{The Shapley value and the semivalue}
\label{app:shap_and_semi}
In this section, we elaborate on the Shapley value and the semivalue with a focus on their mathematical properties. To this end, we denote an attribution that is based on a coalition function $v$ by $\phi(\cdot;v)$\footnote{In the manuscript we omitted $v^{\mathrm{(cond)}}$ for notational convenience.}. 

\paragraph{Shapley axioms} We first introduce the four Shapley axioms: Linearity, Null player, Symmetry, and Efficiency.
\begin{itemize}
    \item Linearity: for coalition functions $v_1, v_2 \in \{u \mid u:2^{[d]}\to \mathbb{R}\}$  and $\alpha_1, \alpha_2 \in \mathbb{R}$, $\phi(x_i, \alpha_1 v_1 + \alpha_2 v_2 ) =\alpha_1 \phi(x_i, v_1 )+\alpha_2\phi(x_i, v_2 )$.
    \item Null player: if $v(S\cup \{i\})=v(S)+c$ for every $S\subseteq [d]\backslash\{i\}$ and some $c \in \mathbb{R}$ and $v:2^{[d]}\to \mathbb{R}$, then $\phi(x_i; v)=c$.
    \item Symmetry: for every $v:2^{[d]}\to \mathbb{R}$ and every permutation $\pi$ on $[d]$, $\phi(x_i; \pi^* v) = \pi^* \phi(x_i; v)$ where $\pi^* v$ is defined as $(\pi^* v) (S) := v(\pi(S))$ for every $S \subseteq [d]$.
    \item Efficiency: for every $v:2^{[d]}\to \mathbb{R}$, $\sum_{i\in[d]} \phi(x_i;v) = v([d])$.
\end{itemize}
\citet{shapley1953} showed that the Shapley value is the unique function that satisfies the four axioms, and it has been studied in the literature \citep{owen2014sobol, lundberg2017unified}. 

\paragraph{Semivalue} The relaxation of the Shapley axioms has been a central topic in cooperative game theory \citep{monderer2002variations, weber1988probabilistic}. Our WeightedSHAP is based on the semivalue \citep{dubey1977probabilistic}, which relaxes the Efficiency axiom. We formally define the semivalue in the following definition.
\begin{definition}[semivalue]
We say a function $\phi$ is a semivalue if $\phi$ satisfies Linearity, Null player, and Symmetry axioms.
\end{definition}
In the machine learning literature, \citet{kwon2021beta} used the semivalue to quantify the importance of individual data points. In addition, they showed the theoretical properties of the semivalue in the language of the data valuation problem. In the following, we provide the counterpart theorems on the attribution problem settings. 

\begin{proposition}[Theorem 2 of \citet{kwon2021beta}]
An attribution $\phi$ is a semivalue, if and only if, the exists a weight vector $\mathbf{w}=(w_1, \dots, w_d)^T$ such that $\sum_{j=1} ^d \binom{d-1}{j-1} w_j = d$ and the attribution $\phi$ can expressed as follows.
\begin{align}
    \phi(x_i;v, \mathbf{w}) := \frac{1}{d}\sum_{j=1} ^d \binom{d-1}{j-1} w_j \Delta_j (x_i; v).
    \label{eqn:beta_shapley}
\end{align}
\label{prop:semivalue}
\end{proposition}
By setting $\tilde{w}_j = \binom{d-1}{j-1} w_j/d$, it can be represented as a weighted mean of the marginal contribution in Equation~\ref{eqn:semivalue}, \textit{i.e.}, $\phi_{\mathbf{\tilde{w}}} = \phi(x_i;v, \mathbf{w})$.

\begin{proposition}[Proposition 3 of \citet{kwon2021beta}]
Let $\phi_1$ and $\phi_2$ be two attribution methods such that for any $v$, the sum of feature attributions are same, \textit{i.e.},
\begin{align*}
    \sum_{i=1} ^d \phi_1(x_i; v) =\sum_{i=1} ^d \phi_2(x_i; v). 
\end{align*}
Then, the two attribution methods are identical, \textit{i.e.}, $\phi_1 = \phi_2$.
\end{proposition}
That is, if there are two attributions with the same total sum of attributions across all coalition function $v$, then they are identical. Although the semivalue is not unique, but it is \textit{almost} unique in that it is the only semivalue with a particular sum of attributions. 

We deploy the concept of the semivalue into the attribution problem, which makes the main difference from \citet{kwon2021beta}. In contrast to the data valuation problem, where the marginal contribution based on a small coalition size is preferred because it is more effective to capture the label errors, we observe that it is not necessarily true in the attribution problem. As we discussed in the manuscript, the marginal contribution based on a large coalition size $\Delta_d$ can be preferred due to the signal, but a weighted mean can be preferred to reduce an estimation error.

\section{Implementation details}
\label{app:imple_details}
In this section, we provide implementation details used in Example in Section~\ref{sec:proposed} and numerical experiments in Section~\ref{sec:experiment}. We present implementation algorithms in Section~\ref{app:details_algorithm}, datasets in Section~\ref{app:details_datasets}, and experimental settings in Section~\ref{app:details_experimental}. Our Python-based implementations are available at \url{https://github.com/ykwon0407/WeightedSHAP}.

\subsection{Implementation algorithm}
\label{app:details_algorithm}

Given a finite set $\mathcal{W}$, an easy-to-compute utility function $\mathcal{T}$ and the marginal contribution estimates $\Delta_j(x_i)$, the optimal weight $\mathbf{w}^*$ can be achieved by iteratively evaluating the utility $\mathcal{T}$ for each attribution method $\phi_{\mathbf{w}}$ with $\mathbf{w} \in \mathcal{W}$. This procedure is described in Algorithm~\ref{alg:weightedSHAP}. 

\begin{algorithm}[h]
\caption{Computation of WeightedSHAP}
\begin{algorithmic}
\Require A finite set of weights $\mathcal{W}$. A utility function $\mathcal{T}$. A set of estimates for the marginal contributions $\Delta_j (x_i)$ obtained from Algorithm~\ref{alg:marginal_contributions}.
\Procedure{}{}
\State Initialize a constant $C_{\mathcal{T}}=-\infty$.
\For{$\mathbf{w} \in \mathcal{W}$}
\State Compute $\phi_{\mathbf{w}} (x_i) = \sum_{j=1} ^d \mathbf{w}_j \Delta_j(x_i) $ for all $i \in [d]$.
\State Rank features $(x_{(1)}, \dots, x_{(d)})$ based on the absolute value of their attribution, \textit{i.e.}, $|\phi_{\mathbf{w}} (x_{(1)})| \geq \dots \geq |\phi_{\mathbf{w}} (x_{(d)})|$.
\State Evaluate $\mathcal{T}(\phi_{\mathbf{w}};x, \hat{f})$. 
\If{$C_{\mathcal{T}} \leq \mathcal{T}(\phi_{\mathbf{w}};x, \hat{f})$}
\State $C_{\mathcal{T}} \leftarrow \mathcal{T}(\phi_{\mathbf{w}};x, \hat{f})$ and $\mathbf{w}^* (\mathcal{T}, \mathcal{W}) \leftarrow \mathbf{w}$.
\EndIf
\EndFor
\State $\phi_{\mathrm{WeightedSHAP}}(\mathcal{T}, \mathcal{W}):= \phi_{\mathbf{w}^*(\mathcal{T}, \mathcal{W})}$
\EndProcedure
\end{algorithmic}
\label{alg:weightedSHAP}
\end{algorithm}

\begin{algorithm}[h]
\caption{Estimation of the marginal contributions $\Delta_j(x_i)$}
\begin{algorithmic}
\Require A conditional coalition estimate $\hat{v}^{\mathrm{(cond)}}$. A terminating threshold $\rho$.
\Procedure{}{}
\State Initialize $\hat{\rho}=2\rho$, $B=1$, $\Delta_j ^{(0)} (x_i) = 0$ for all $i,j \in [d]$.
\While{$\hat{\rho} \geq \rho$}
\For{$i \in [d]$}
\State $S \leftarrow \emptyset$.
\State Generate a random order of $[d]\backslash\{i\}$ and denote it by $\eta$.
\For{$j \in [d]$}
\State $\Delta_j ^{(B)} (x_i) \leftarrow \frac{B-1}{B} \Delta_j ^{(B-1)} (x_i) +  \frac{1}{B} ( \hat{v}^{\mathrm{(cond)}}(S\cup\{i\})- \hat{v}^{\mathrm{(cond)}} (S))$.
\State $S \leftarrow S \cup \eta(j)$.
\EndFor
\EndFor
\State Compute the Gelman-Rubin statistics for $\{\Delta_j ^{(b)} (x_i)\}_{b=1} ^{B}$ and take its maximum value among $i,j \in [d]$.
\State $B \leftarrow B + 1$.
\EndWhile
\EndProcedure
\end{algorithmic}
\label{alg:marginal_contributions}
\end{algorithm}

Throughout our experiments, $\mathcal{T}$ is the negative value of the AUP defined in \ref{eqn:prediction_recovery} and $\mathcal{W}$ is a set of 13 different weights, namely, $\mathcal{W} = \{\Delta_1, \Delta_d\} \cup \{\phi_{\mathbf{w}^{\mathrm{Beta}}(\alpha, \beta)} \mid (\alpha, \beta) \in \{ (16,1), (8,1), (4,1), (2,1), (1,1),$ $(1,2), (1,4), (1,8), (1,16), (1,32) \}  \}.$ Here, $\mathbf{w}^{\mathrm{Beta}}(\alpha, \beta) \in [0,1]^{d}$ is defined as
\begin{align*}
    (\mathbf{w}^{\mathrm{Beta}}(\alpha, \beta))_j := \binom{d-1}{j-1} \frac{\mathrm{Beta}(j+\beta-1, d-j+\alpha)}{\mathrm{Beta}(\alpha, \beta)},
\end{align*}
where $\mathrm{Beta}(\alpha, \beta)=\Gamma(\alpha)\Gamma(\beta)/\Gamma(\alpha, \beta)$ is the Beta function and $\Gamma(\cdot)$ is the Gamma function. Inspired by \citet{kwon2021beta}, the key motivation of this function is that (i) this type of a functional form is known to satisfy the condition in Proposition~\ref{prop:semivalue} and (ii) it has a closed-form expression, so it is easy-to-compute. Note that $\mathbf{w}^{\mathrm{Beta}}(1,1)=d^{-1} \mathds{1}_d $ generates the Shapley value, and our algorithm guarantees the superior performance of WeightedSHAP to the Shapley value as $\mathbf{w}^{\mathrm{Beta}}(1,1) \in \mathcal{W}$. Since the Beta weight vector $\mathbf{w}^{\mathrm{Beta}}(\alpha, \beta)$ distributes weights throughout all the marginal contributions $\Delta_1, \dots, \Delta_d$, we can think of it as a regularized version of $\mathds{1}(j=1)$ and $\mathds{1}(j=d)$, which generates $\Delta_1$ and $\Delta_d$, respectively. 

The key input of Algorithm~\ref{alg:weightedSHAP} is the marginal contribution estimates. To obtain this estimate, we first need to estimate a conditional coalition function $v^{\mathrm{(cond)}}$. Following the literature \citep{frye2020shapley, jethani2021fastshap, jethani2021have}, we obtain this by training a surrogate model\footnote{It is known that this surrogate model unbiasedly estimates a conditional expectation of a prediction value given a subset of features under mild conditions \citep{frye2020shapley, covert2020understanding}.} that takes as input a subset of input features and outputs a conditional expectation of a prediction value given the same subset. As for the surrogate model, we use a multilayer perceptron model with the two hidden layers, and each layer has 128 neurons with the ELU nonlinear activation function \citep{clevert2015fast}. For the classification (\textit{resp}. regression) tasks, we consider the Kullback-Libeler divergence (\textit{resp}. MSE loss) as a loss function as suggested in \citep{jethani2021fastshap}. We use the held-out dataset to learn a surrogate model, which corresponds to 10\% of the original dataset. Throughout our experiments, we mostly follow the hyperparameters used in the previous work \citet{jethani2021fastshap}. For instance, we use the same Adam optimizer \citep{kingma2014adam} with the initial learning rate of $10^{-3}$, the epochs of 100, and the mini-batch size of 64. 

With the estimator $\hat{v}^{\mathrm{(cond)}}$, the marginal contributions are estimated as in Algorithm~\ref{alg:marginal_contributions}. Using the fact that the marginal contribution in Equation~\eqref{eqn:marginal_contribution} is defined as a simple average over all the possible subsets $S$ with the same coalition size, we use a sampling-based algorithm \citep{ghorbani2019data, kwon2021beta}. It is well known that a sampling-based algorithm guarantees the convergence to the true marginal contribution value when sampling procedures are repeated. In our experiments, we use a finite number of samplings based on the Gelman-Rubin stopping criteria \citep[Equation (4)]{vats2021revisiting}. We regard the repeated sampling procedures as 10 Markov chains and compute the Gelman-Rubin statistic for the marginal contribution $\Delta_j (x_i)$ for all $i,j \in[d]$. We take their maximum for every iteration and stop the sampling procedure if this maximum is smaller than a prefixed terminating threshold. We use a terminating threshold $\rho=1.005$ which is much smaller than a typical terminating threshold $1.1$ \citep{gelman1995bayesian}.

\subsection{Datasets}
\label{app:details_datasets}
In Section~\ref{sec:experiment} and~\ref{app:additional_results_datasets}, the four real-world datasets are used for the regression tasks. The one synthetic dataset (\texttt{gaussian}) and the seven real-world datasets are used for the classification tasks. All the real-world datasets are downloaded from the UCI repository \citep{Dua2019} or the OpenML platform\footnote{Website: \url{https://www.openml.org/}}. 
As for the synthetic classification dataset (\texttt{gaussian}), we consider the following distribution. For $d=30$ and $\rho=0.25$,
\begin{align*}
    X \sim \mathcal{N}(\mathbf{0}_d, \Sigma_{(\rho,d)}), \quad Y = \mathrm{Bern}(p(X)),
\end{align*}
where $p(X) := \exp(X^T \beta^*)/(1 + \exp(X^T \beta^*))$ for $\beta^*=(1, 0.98, \dots, 0.82, 0, \dots, 0)^T$. The first 10 features are associated with an output and the last 20 features are not directly associated. 

As for the real datasets, from a given raw dataset, we only use the continuous features and exclude an observation if some information is missing. Table~\ref{tab:summary_real_datasets} summarizes the basic information after preprocessing the raw datasets. From this dataset, we randomly take $10,000$ samples and consider it as the entire dataset. We normalize every feature to have zero mean and unit variance. In addition, we generate spurious features that are associated with the original features, but not directly associated with an output by its construction. We repeat the following procedure until the number of features becomes three times the original number of features: given an input matrix $X \in \mathbb{R}^{n \times p}$, we generate a new column $X_{\mathrm{new}} \in \mathbb{R}^{n \times 1}$ by $X_{\mathrm{new}} = \frac{\rho}{1+\rho(p-1)} X\mathds{1}_p + \sqrt{1- \frac{\rho^2 p }{1+\rho(p-1)}} \varepsilon$ where $\varepsilon \sim \mathcal{N}(0,1)$ is randomly drawn from a standard Gaussian distribution and $\rho$ is the average of off-diagonal terms in the correlation matrix of the original input matrix. This procedure is to append more features while keeping the average correlation between two features same. We observe several preprocessing steps we deploy do not affect our main result that WeightedSHAP consistently assigns large values for more influential features. However, they are deployed to increase the stability in training a prediction model and to reflect practical data analysis situations (e.g., features are highly correlated, but not necessarily every feature is associated with an output).

\subsection{Experimental settings}
\label{app:details_experimental}

\paragraph{Experiments in Section~\ref{sec:motivational_examples}} As we have closed-form expressions under the Gaussian assumption by Theorem~\ref{thm:marginal_contribution}, we do not use any estimation algorithm here. As for the Figure~\ref{fig:motivation_gaussian_two_features}, explicit forms for the Shapley value and the optimal order have been used. As for the Figure~\ref{fig:motivation_gaussian_more_than_two_features}, we first generate $10,000$ training data points as follows. For $d=100$ and $\rho \in \{0.2, 0.6\}$, 
\begin{align*}
    X \sim \mathcal{N}(\mathbf{0}_d, \Sigma_{(\rho,d)}), \quad Y = X^T \beta^* + 2\times \varepsilon,
\end{align*}
where $\varepsilon \sim \mathcal{N}(0, 1)$ is a random Gaussian error and $\beta^*=(1, 0.99, \dots, 0.81, 0,\dots, 0)^T$. That is, the first $20$ features are associated with an output, but the last $80$ features are not directly associated with the output. We train a linear model based on the $10,000$ samples and evaluate feature attributions for the held-out 100 samples. All the results in Figure~\ref{fig:motivation_gaussian_more_than_two_features} is based on this held-out test dataset.

\paragraph{Experiment in Section~\ref{sec:proposed}} In contrast to the experiment in Section~\ref{sec:motivational_examples}, we estimate the marginal contributions following Algorithm~\ref{alg:marginal_contributions}. Other experiment settings are same as in Section~\ref{sec:motivational_examples}.

\begin{table}[t]
    \centering
    \caption{A summary of real-world datasets used in experiments.}
    \begin{tabular}{lcccccccccccc}
    \toprule
    Dataset & Sample size  & Input dimension & Source    \\ 
    \midrule
    \textbf{Regression}\\
    \texttt{abalone} & 4177  & 10 & UCI Repository \\
    \texttt{boston} & 506  & 13 & \citet{harrison1978hedonic} \\
    \texttt{airfoil} & 1503  & 5 & UCI Repository \\
    \texttt{whitewine} & 4898  & 11 & UCI Repository \\
    \midrule
    \textbf{Classification}\\
    \texttt{cpu} & 8192 & 22 & \url{https://www.openml.org/d/761} \\
    \texttt{fraud} & 284807 & 31 & \citet{dal2015calibrating} \\
    \texttt{phoneme} & 5404 & 6 & \url{https://www.openml.org/d/1489}  \\
    \texttt{wind} & 6574 & 15 &  \url{https://www.openml.org/d/847}  \\
    \texttt{adult} & 32561 & 11 & UCI Repository \\
    \texttt{2dplanes} & 40768 & 11 & \url{https://www.openml.org/d/727} \\
    \texttt{click} & 1997410 & 12 &  \url{https://www.openml.org/d/1218} \\
    \bottomrule
    \end{tabular}
    \label{tab:summary_real_datasets}
\end{table}

\paragraph{Experiments in Section~\ref{sec:experiment}} We first split the original dataset into four parts: a training dataset, a validation dataset, a dataset to obtain a surrogate model, and a test dataset. The ratio between them is 70\%, 10\%, 10\%, and 10\%, but we take $n_{\mathrm{test}}:= \max(0.1 \times N, 100)$ samples for a test dataset. Here $N$ denotes the number of the original sample size. The training and validation datasets are used for a prediction model $\hat{f}$, and a dataset to obtain a surrogate model is used for $\hat{v}^{\mathrm{(cond)}}$. After that we compute WeightedSHAP for $n_{\mathrm{test}}$ samples. All the results (\textit{e.g.}, prediction recovery error, MSE, and AUC) are based on these $n_{\mathrm{test}}$ test samples.

For a boosting model, we use the \texttt{lightgbm} algorithm \citep{ke2017lightgbm} with the learning rate of $0.005$ and $15$ final leaves and apply the early stopping with the $25$ patience epoch. For the classification (\textit{resp.} regression) tasks, we use the cross entropy (\textit{resp.} MSE) loss function. As for the estimation of WeightedSHAP, see Section~\ref{app:details_algorithm}. In Section~\ref{app:additional_results_linear_model}, we simply use the least squares estimator for a linear prediction model.

\section{Additional experimental results}
\label{app:add_results}
In this section, we provide additional experimental results. Our main findings presented in the manuscript are consistently observed on different models and datasets. We first present estimation error analysis on a different dataset in Section~\ref{app:relative_difference_02}, additional experimental results when a prediction model is linear in Section~\ref{app:additional_results_linear_model}, and additional experimental results on different classification datasets in Section~\ref{app:additional_results_datasets}.

\subsection{Estimation error analysis when $\rho=0.2$}
\label{app:relative_difference_02}

\begin{figure}[h]
    \centering
    \includegraphics[width=0.3\textwidth]{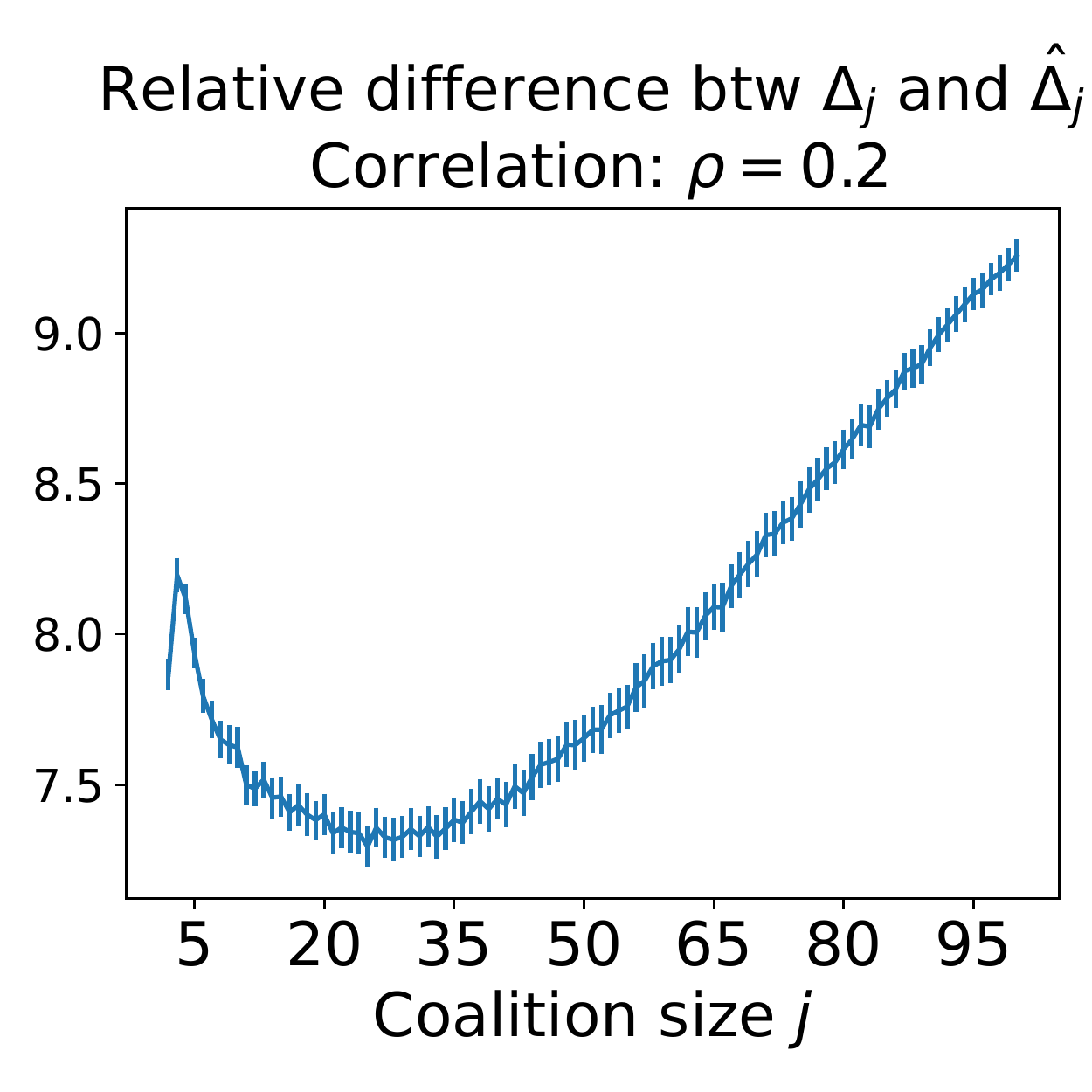}
    \caption{Illustrations of the relative difference between the true marginal contribution $\Delta_j$ and its estimate $\hat{\Delta}_j$ as a function of the coalition size $j \in [100]$. We consider the same setting used in Figure~\ref{fig:motivation_gaussian_more_than_two_features} with $\rho=0.2$. The $\Delta_{100}$ is shown to have the largest relative difference.}
    \label{fig:estimation_error_of_the_marginal_contributions_02}
\end{figure}

We provide an additional estimation error result with a different dataset. In contrast to the Figure~\ref{fig:estimation_error_of_the_marginal_contributions} shown in the manuscript, we generate data points from the Gaussian distribution $\mathcal{N}(\mathbf{0}_{100}, \Sigma_{(0.2, 100)})$, \textit{i.e.}, $\rho=0.2$. 

Figure~\ref{fig:estimation_error_of_the_marginal_contributions_02} shows the relative difference between the true marginal contribution $\Delta_j$ and its estimate $\hat{\Delta}_j$ as a function of the coalition size $j \in \{1, \dots, 100\}$. Similar to Figure~\ref{fig:estimation_error_of_the_marginal_contributions}, the relative difference of $\Delta_{100}$ has the largest estimation error. Given that $\Delta_{100}$ is the most informative marginal contribution in Figure~\ref{fig:motivation_gaussian_more_than_two_features}, it again suggests the use of a weighted mean of the marginal contributions. 

\subsection{Additional experimental results when a prediction model is linear}
\label{app:additional_results_linear_model}
We conduct additional analyses when a prediction model $\hat{f}(x)$ is linear. We use the same experimental setting as in Section~\ref{sec:experiment}, but only a prediction model is changed from a boosting model to a linear model. 

Figures~\ref{fig:prediction_error_regression_linear} and~\ref{fig:inclusion_auc_regression_linear} (\textit{resp.} Figures~\ref{fig:prediction_error_classification_linear} and~\ref{fig:inclusion_auc_classification_linear}) compare WeightedSHAP with the Shapley on the regression tasks (\textit{resp.} the classification tasks). In all experiments, the prediction recovery error of WeightedSHAP is significantly smaller than the MCI and the Shapley value. Moreover, the performance metrics, namely MSE and AUC, of WeightedSHAP are better than or at least comparable to the MCI and the Shapely value. In short, our findings are consistently observed with a linear prediction model, showing the influential features identified by WeightedSHAP are better able to recapitulate the model's predictions compared to the features identified by baseline attribution methods.

\begin{figure}[t]
    \centering
    \subfigure[Illustrations of the prediction recovery error curve on the four regression datasets. The lower, the better.]{
    \includegraphics[width=0.22\textwidth]{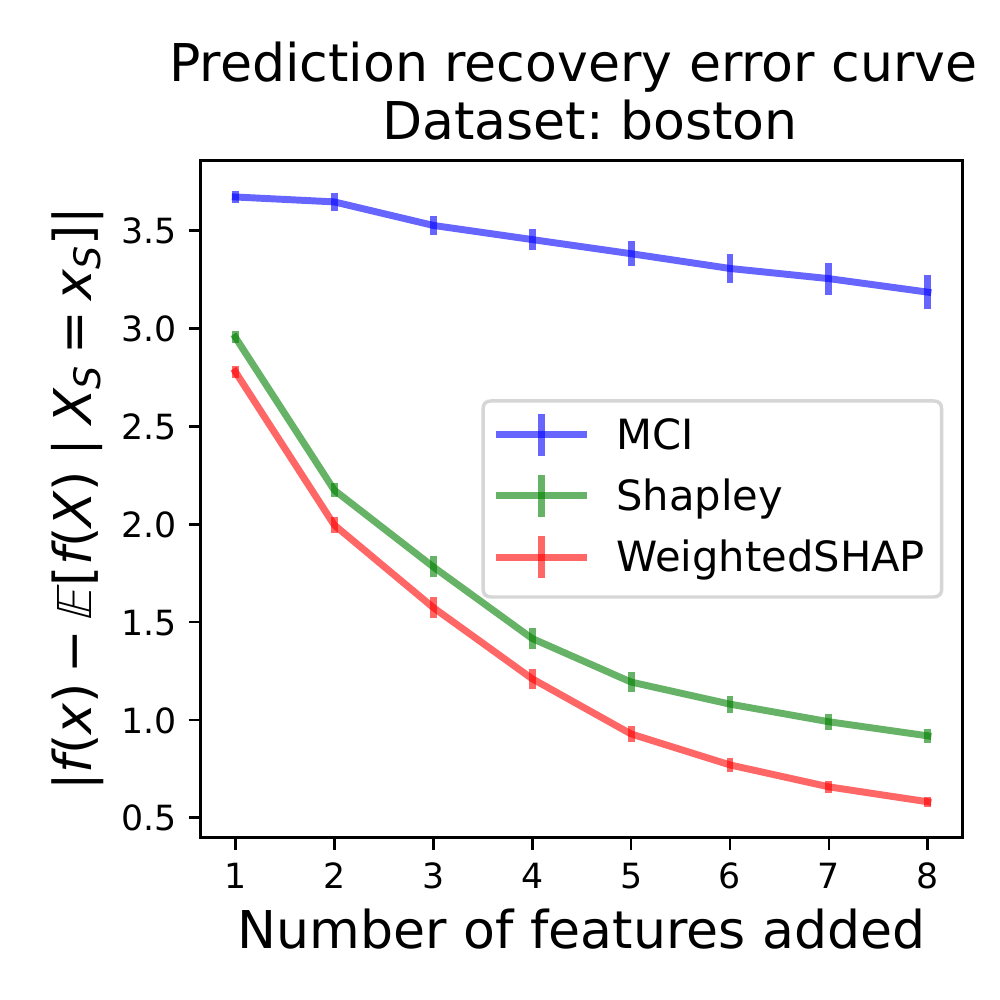}
    \includegraphics[width=0.22\textwidth]{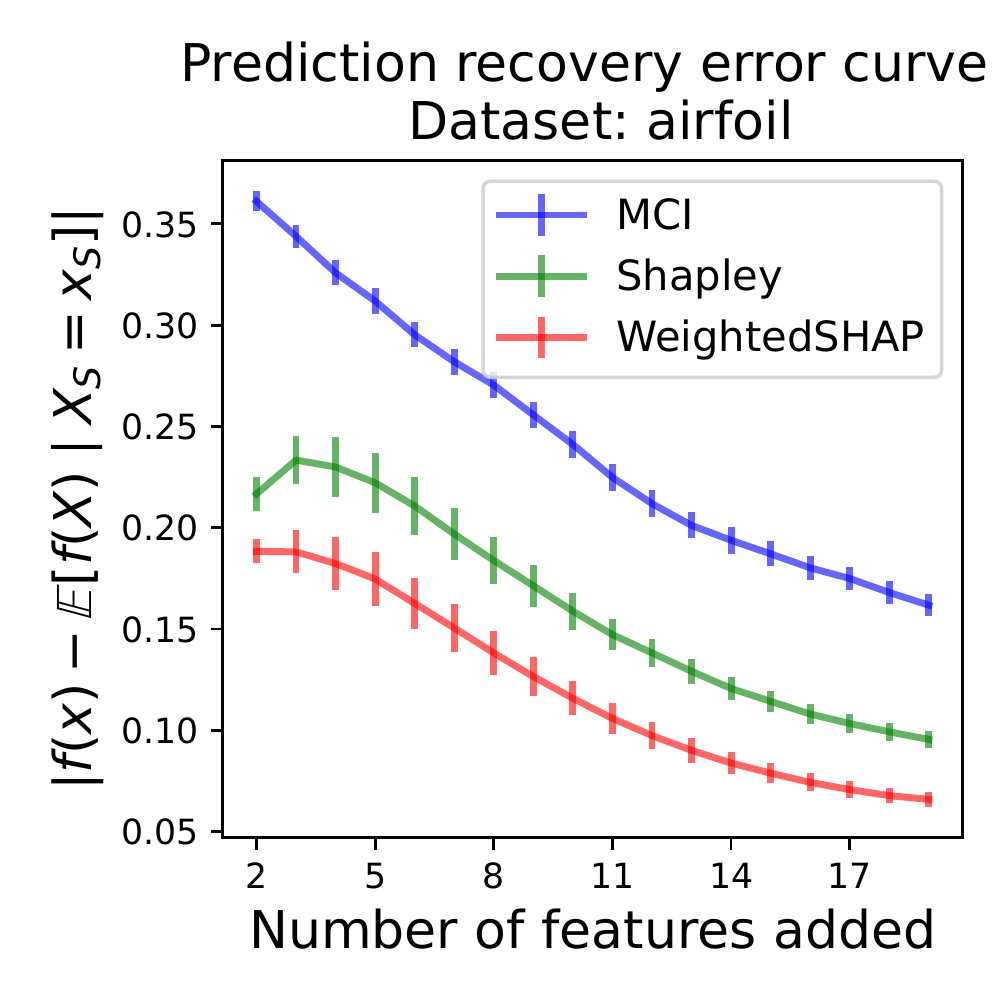}
    \includegraphics[width=0.22\textwidth]{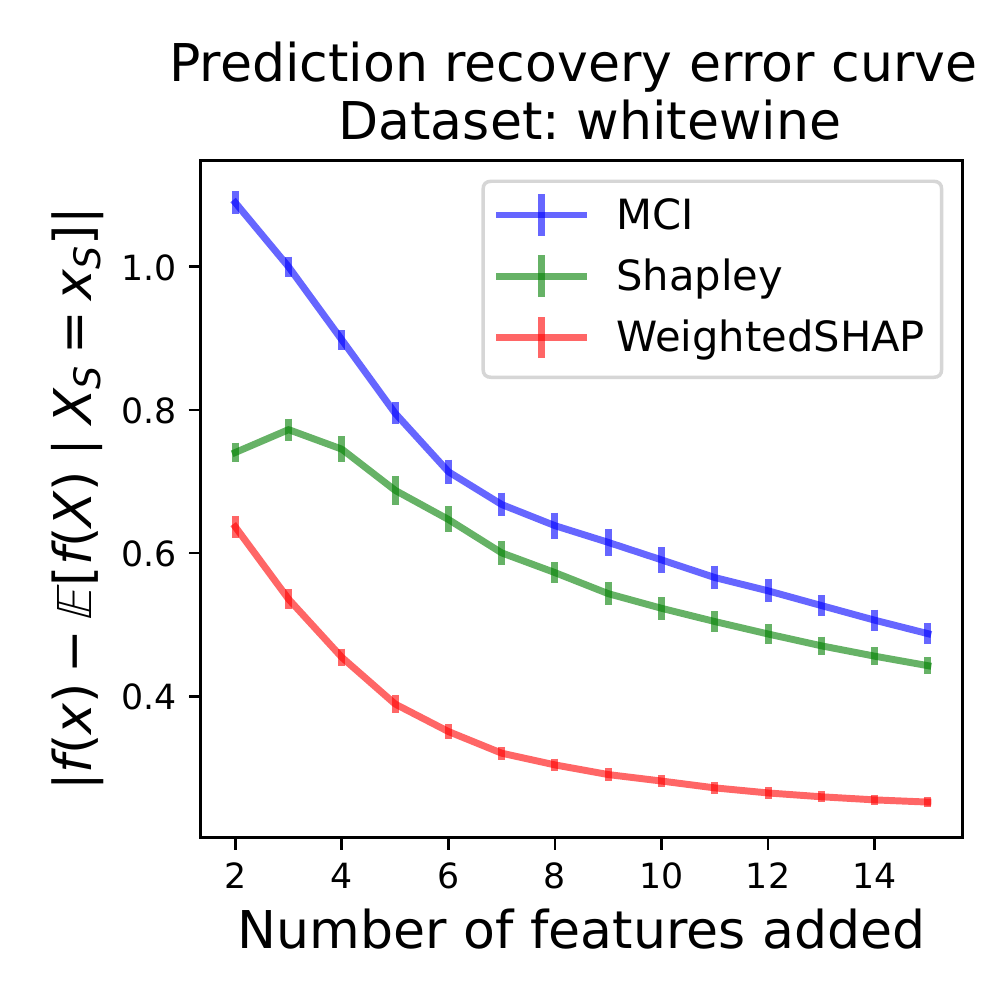}
    \includegraphics[width=0.22\textwidth]{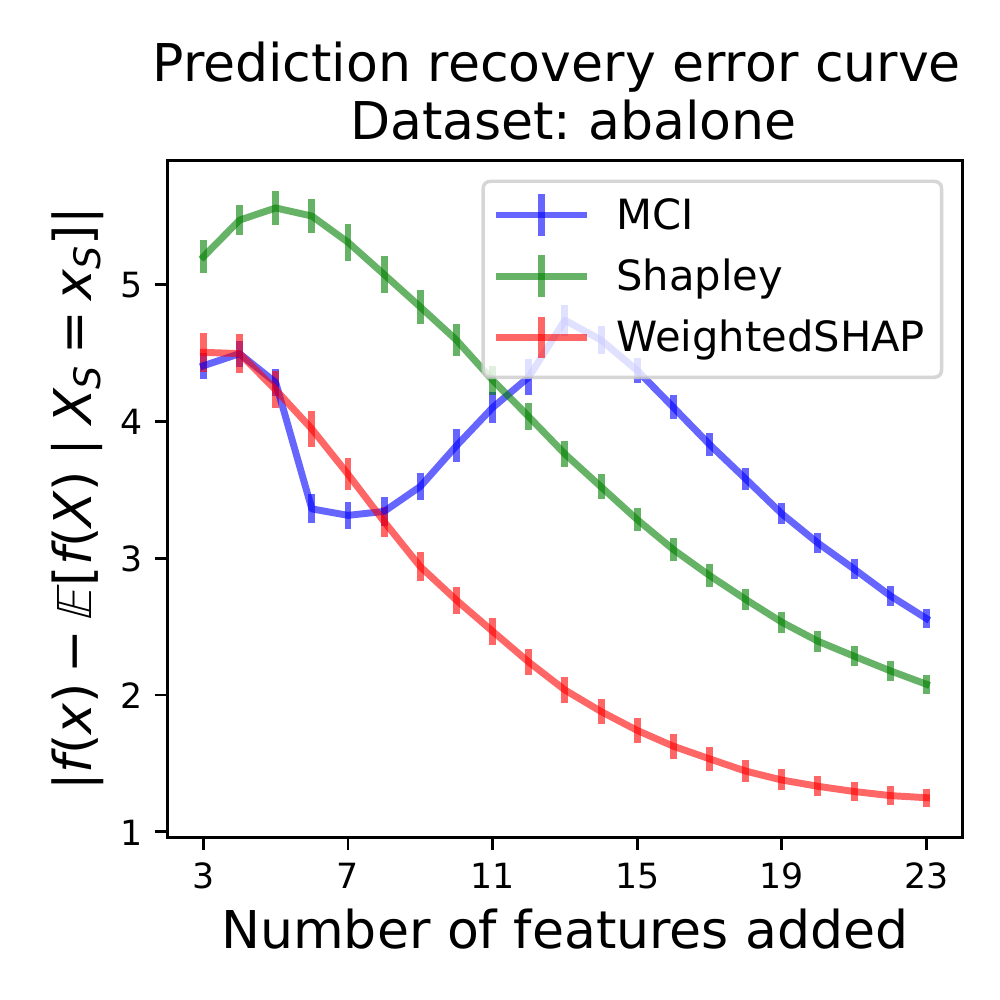}
    \label{fig:prediction_error_regression_linear}
    }
    \subfigure[Illustrations of the Inclusion MSE curve on the four regression datasets. The lower, the better.]{
    \includegraphics[width=0.22\textwidth]{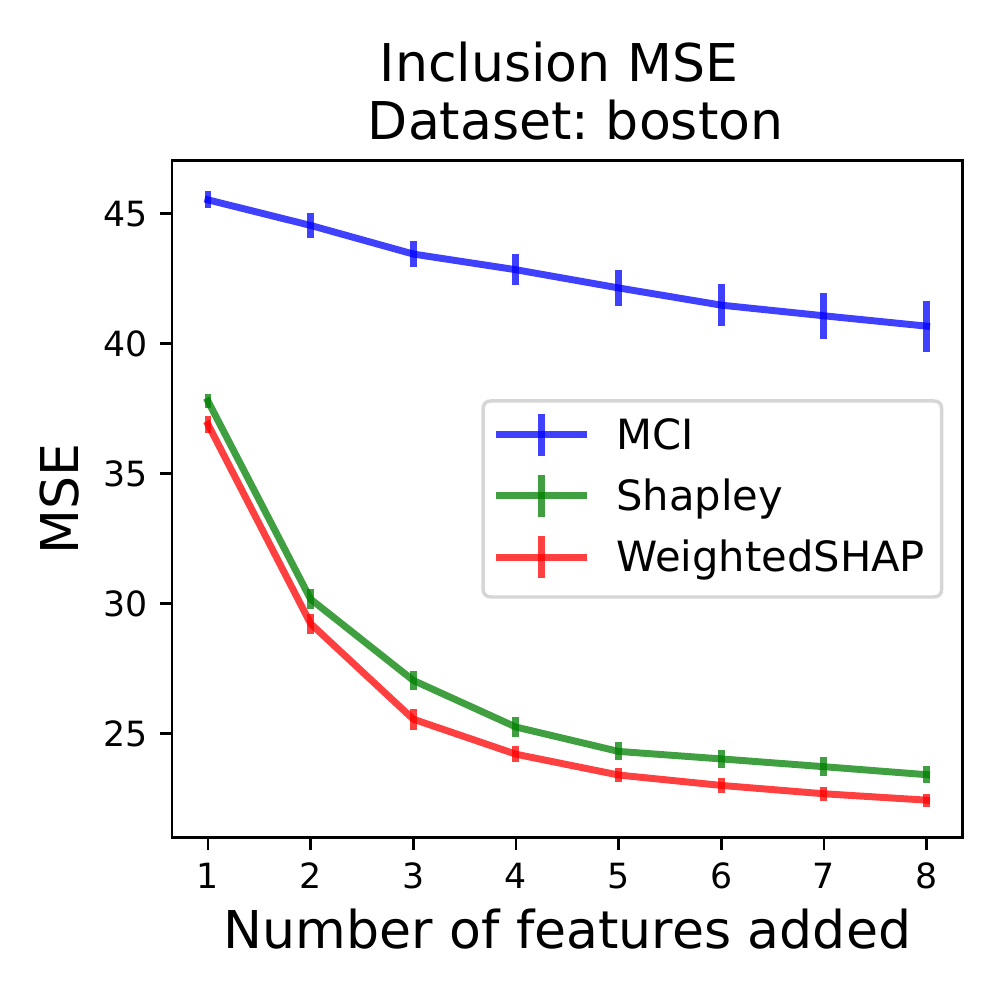}
    \includegraphics[width=0.22\textwidth]{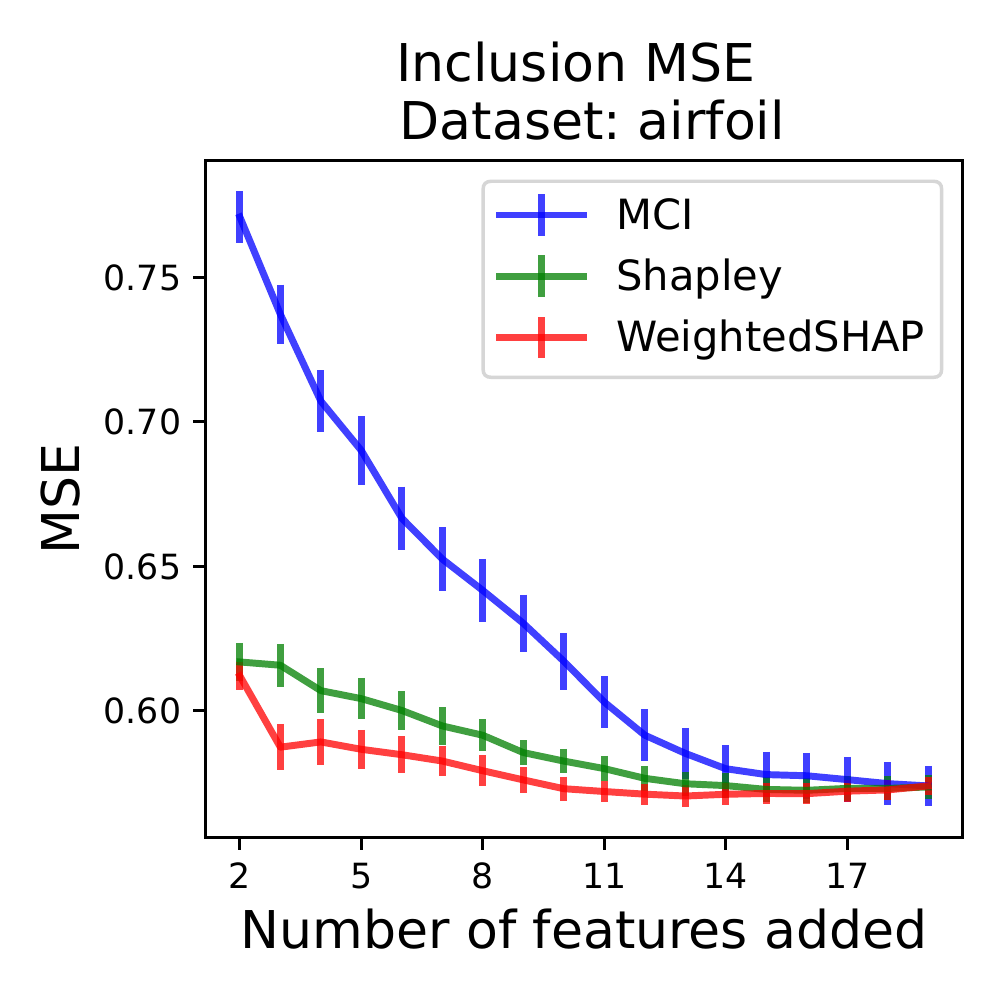}
    \includegraphics[width=0.22\textwidth]{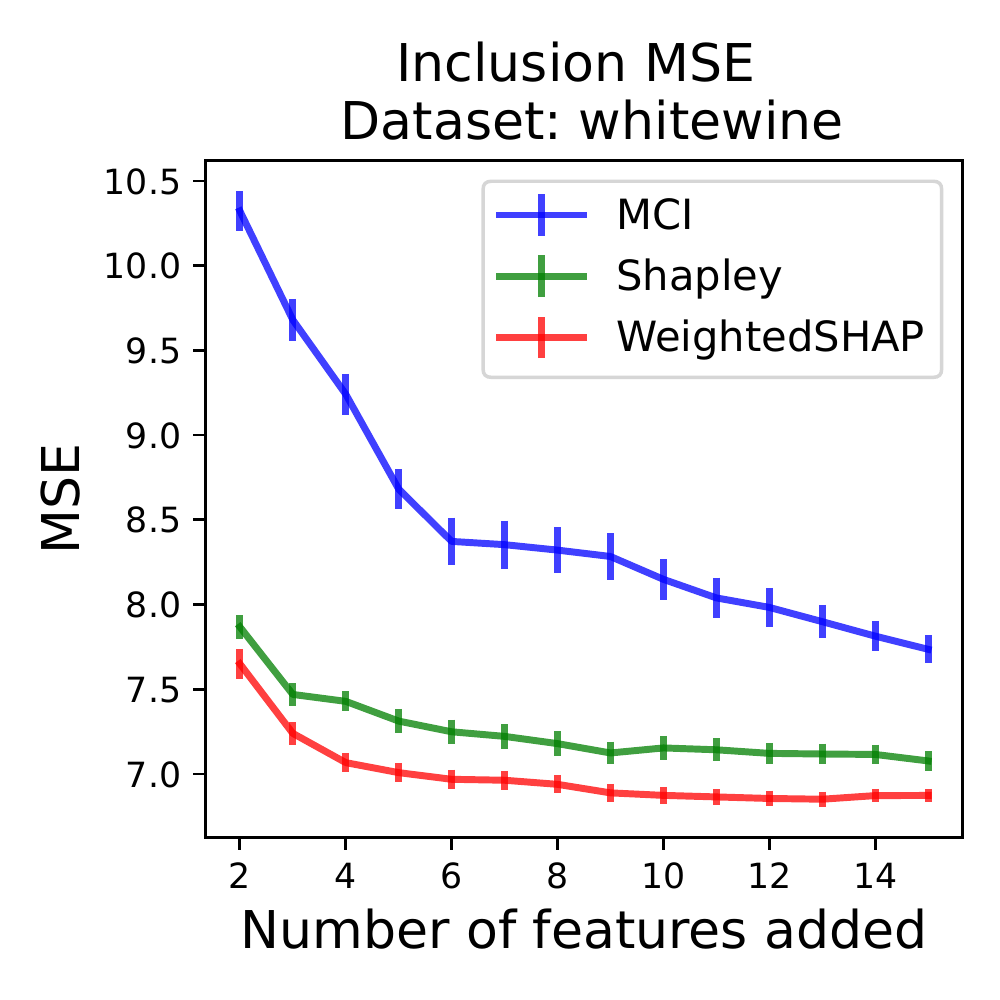}
    \includegraphics[width=0.22\textwidth]{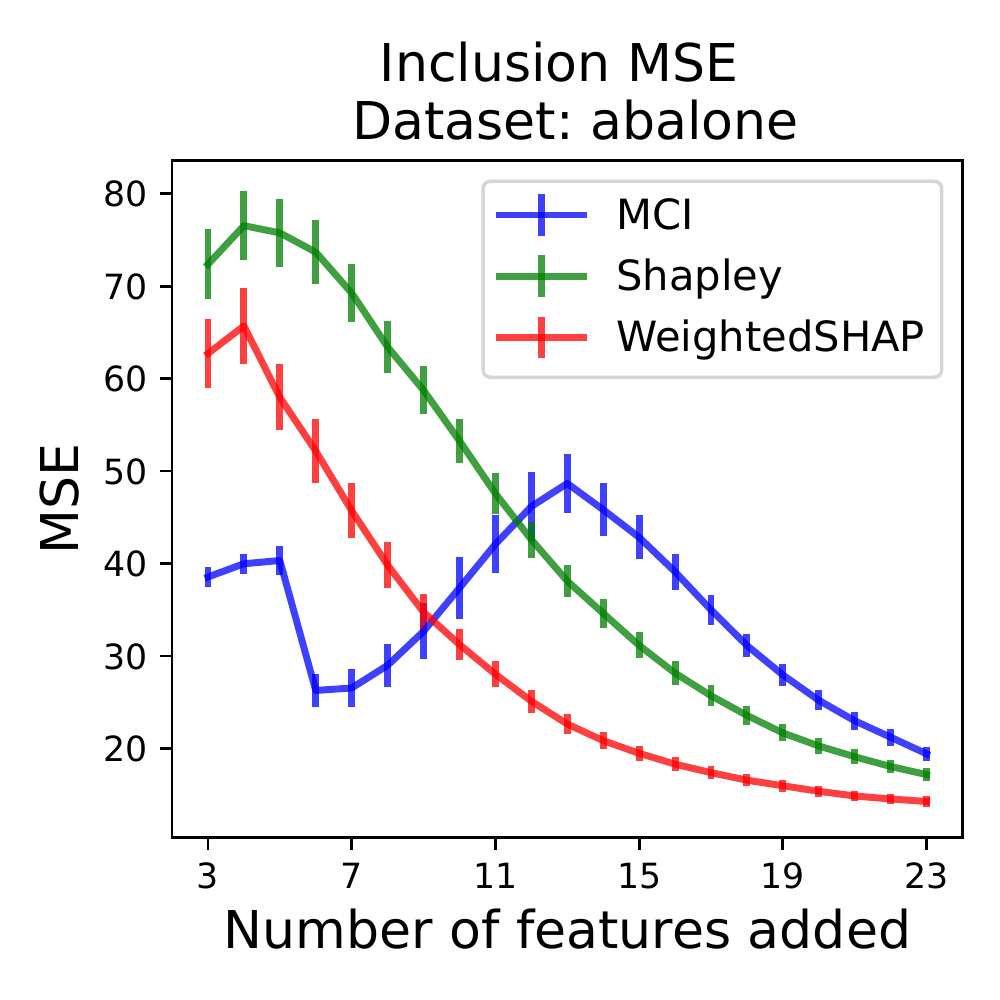}
    \label{fig:inclusion_auc_regression_linear}
    }
    \subfigure[Illustrations of the prediction recovery error curve on the four binary classification datasets. The lower, the better.]{
    \includegraphics[width=0.22\textwidth]{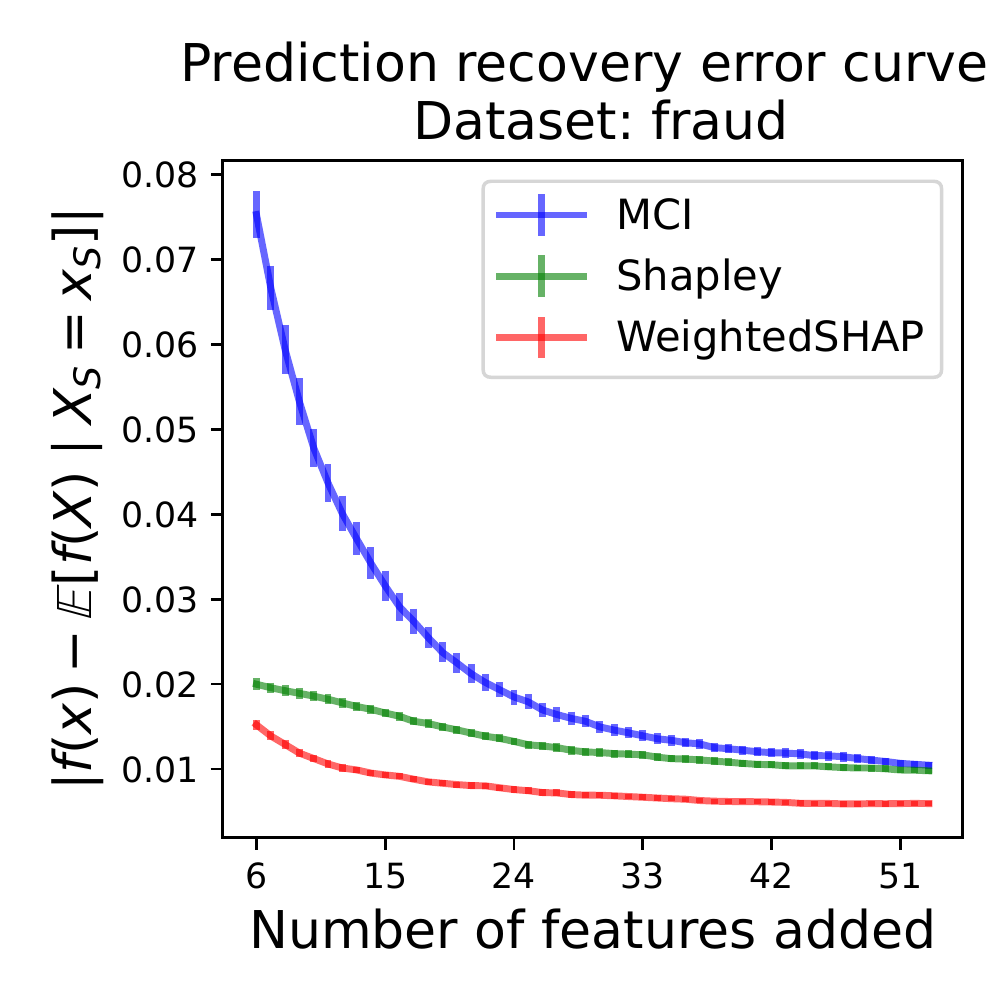}
    \includegraphics[width=0.22\textwidth]{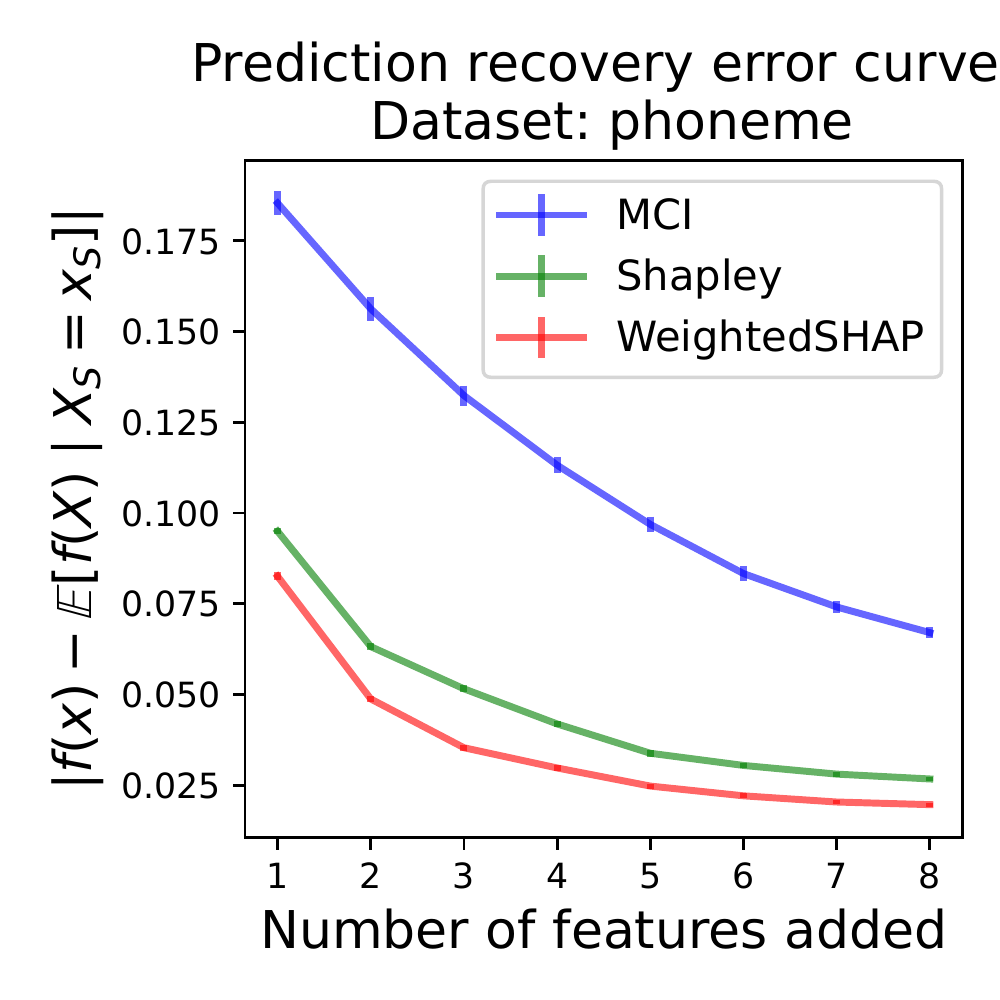}
    \includegraphics[width=0.22\textwidth]{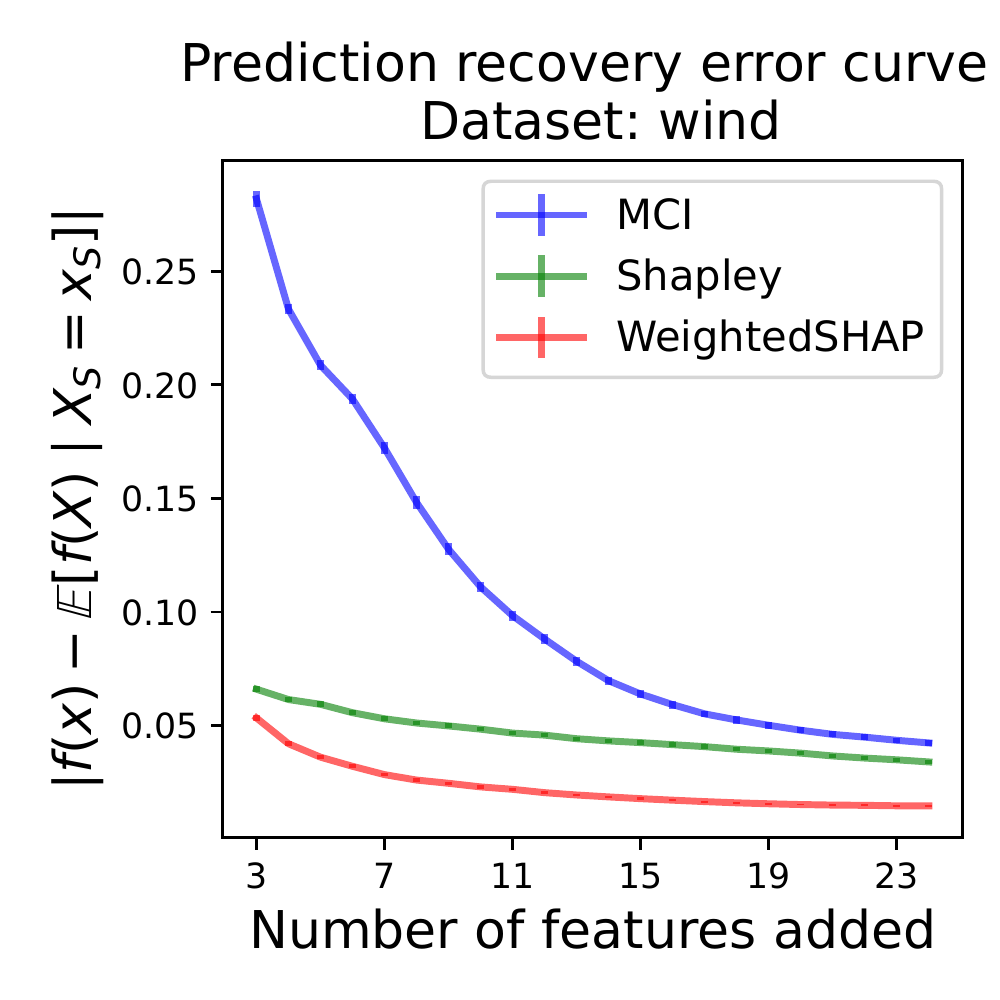}
    \includegraphics[width=0.22\textwidth]{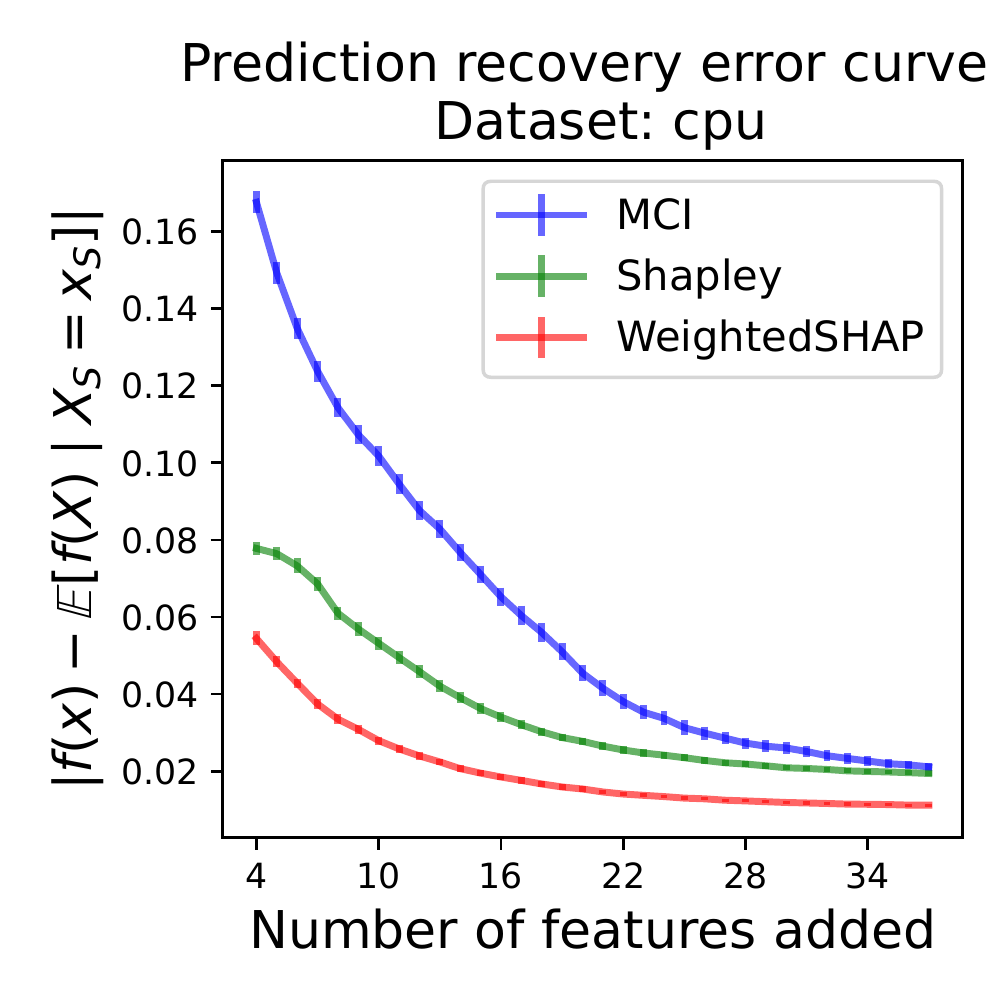}
    \label{fig:prediction_error_classification_linear}
    }
    \subfigure[Illustrations of the Inclusion AUC curve on the four binary classification datasets. The higher, the better.]{
    \includegraphics[width=0.22\textwidth]{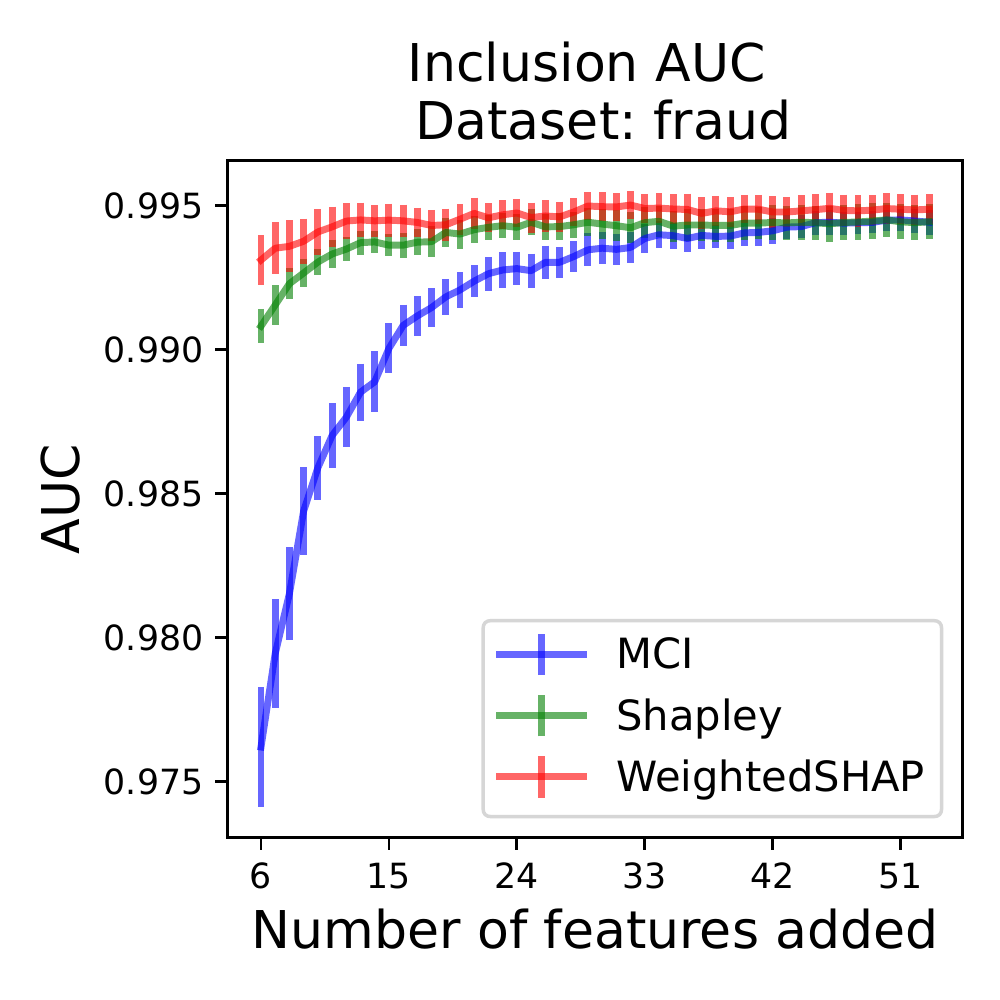}
    \includegraphics[width=0.22\textwidth]{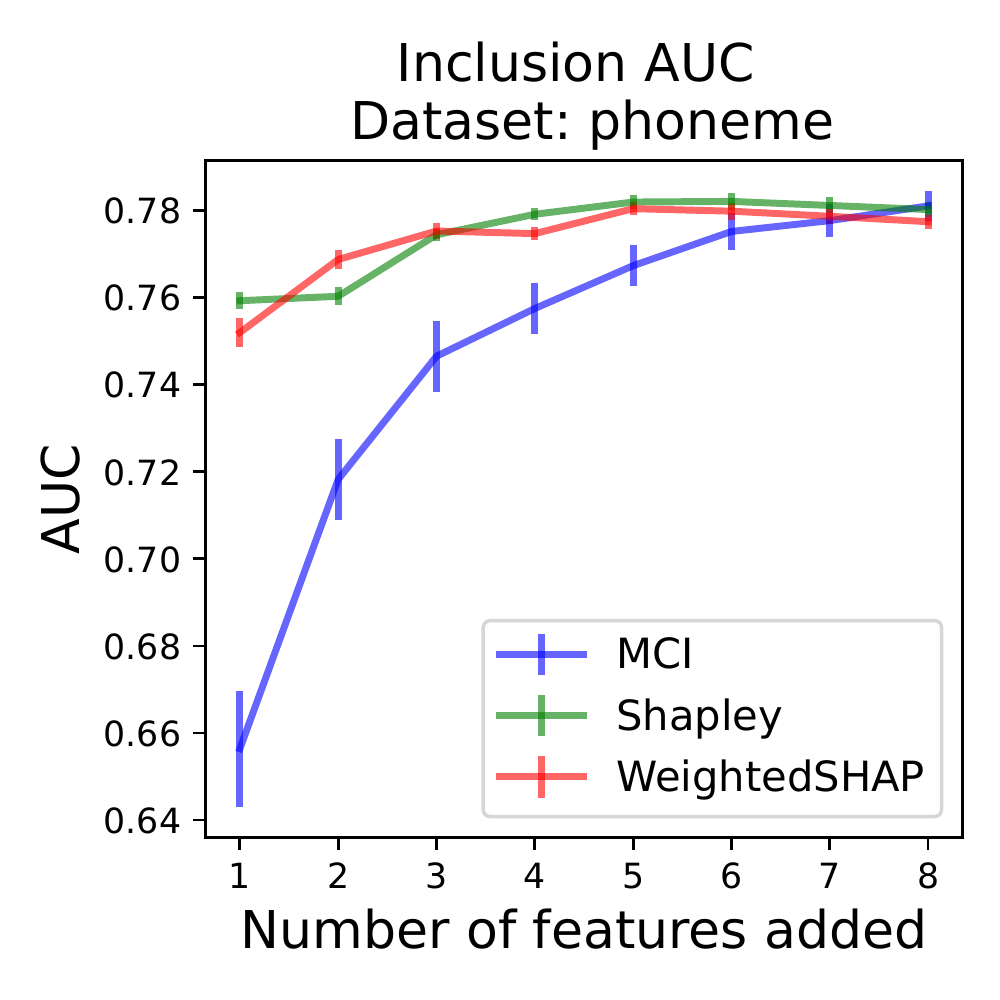}
    \includegraphics[width=0.22\textwidth]{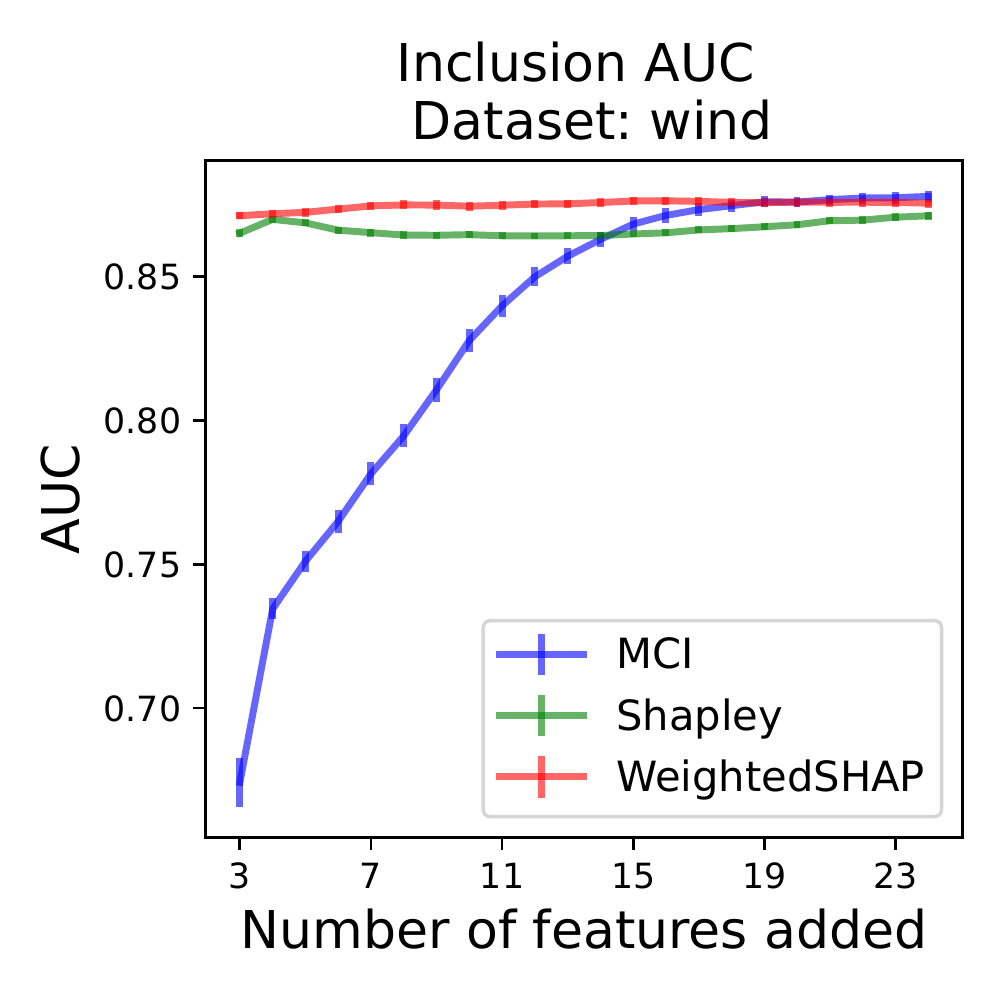}
    \includegraphics[width=0.22\textwidth]{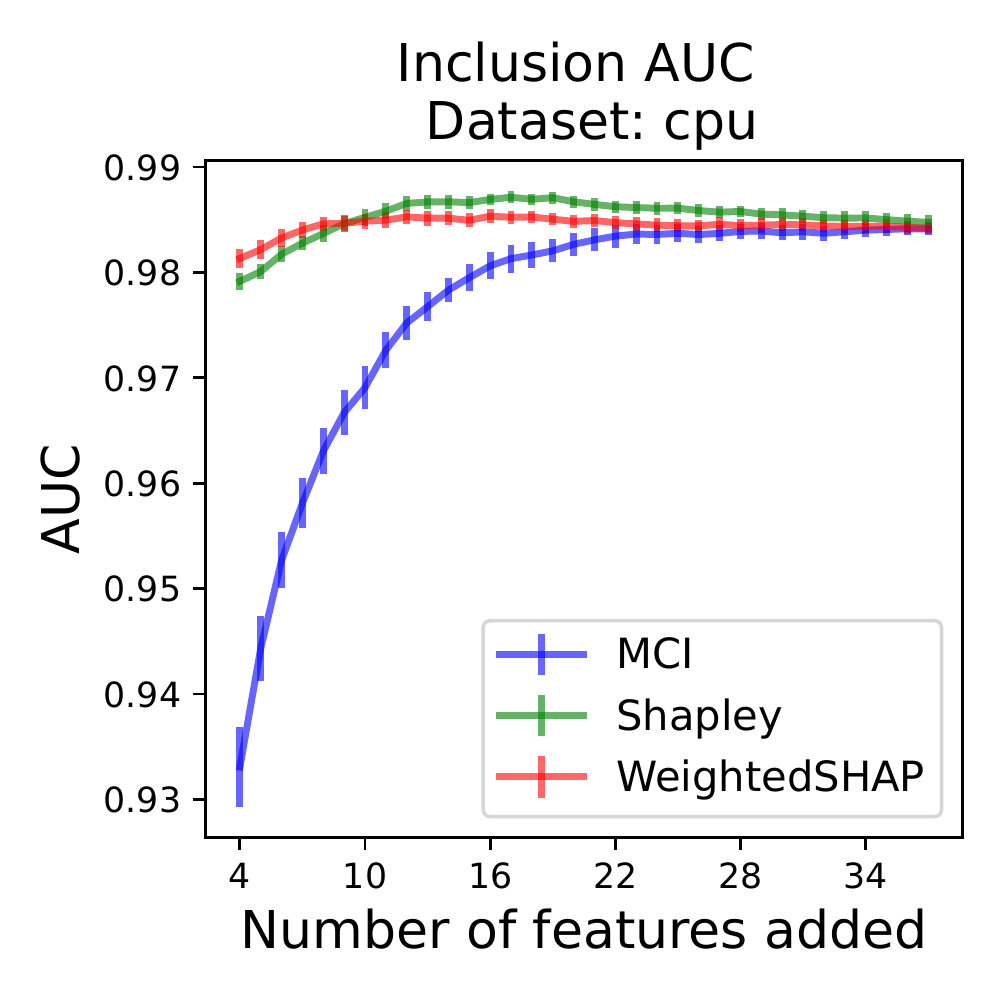}
    \label{fig:inclusion_auc_classification_linear}
    }
    \caption{\textbf{When a linear prediction model is used.} Illustrations of the prediction recovery error curve and the performance curve as a function of the number of features added. We add features from most influential to the least influential. We denote a 95\% confidence interval based on 30 independent runs. Our main findings are consistently observed with a linear prediction model.}
\end{figure}

\subsection{Additional experimental results with different classification datasets}
\label{app:additional_results_datasets}
We conduct additional experiments on different classification datasets. A boosting model is used for a prediction model. In terms of the experimental setting, the only difference from Figures~\ref{fig:prediction_error_classification_boosting} and \ref{fig:inclusion_auc_classification_boosting} is the datasets. Details on the four datasets are provided in Section~\ref{app:details_datasets}. 

Figures~\ref{fig:prediction_error_classification_boosting_additional} and~\ref{fig:inclusion_auc_classification_boosting_additional} show additional experimental results on the four different classification datasets. As in the previous experiments, the prediction recovery error of WeightedSHAP is lower than the Shapley value. As for the inclusion AUC metric, WeightedSHAP is significantly better or at least comparable to the Shapley value. 

\begin{figure}[t]
    \centering
    \subfigure[Illustrations of the prediction recovery error curve on the four binary classification datasets. The lower, the better.]{
    \includegraphics[width=0.22\textwidth]{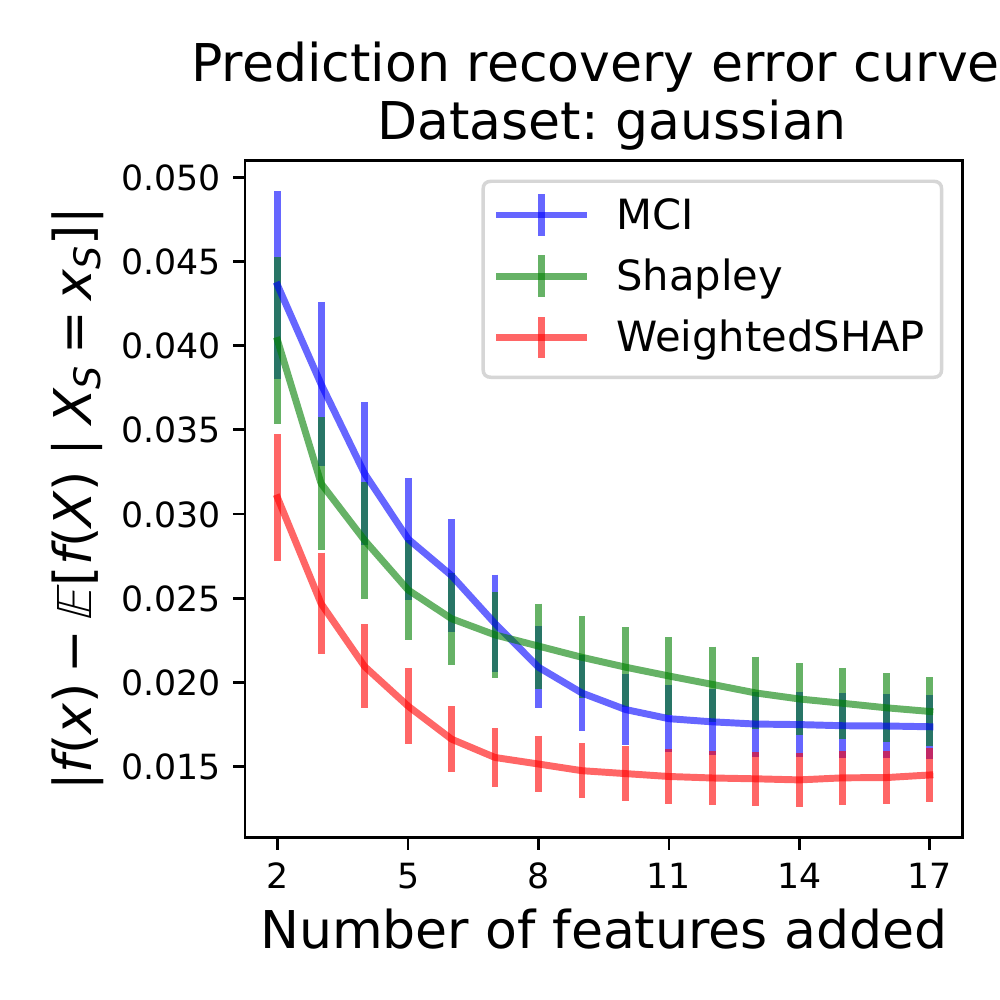}
    \includegraphics[width=0.22\textwidth]{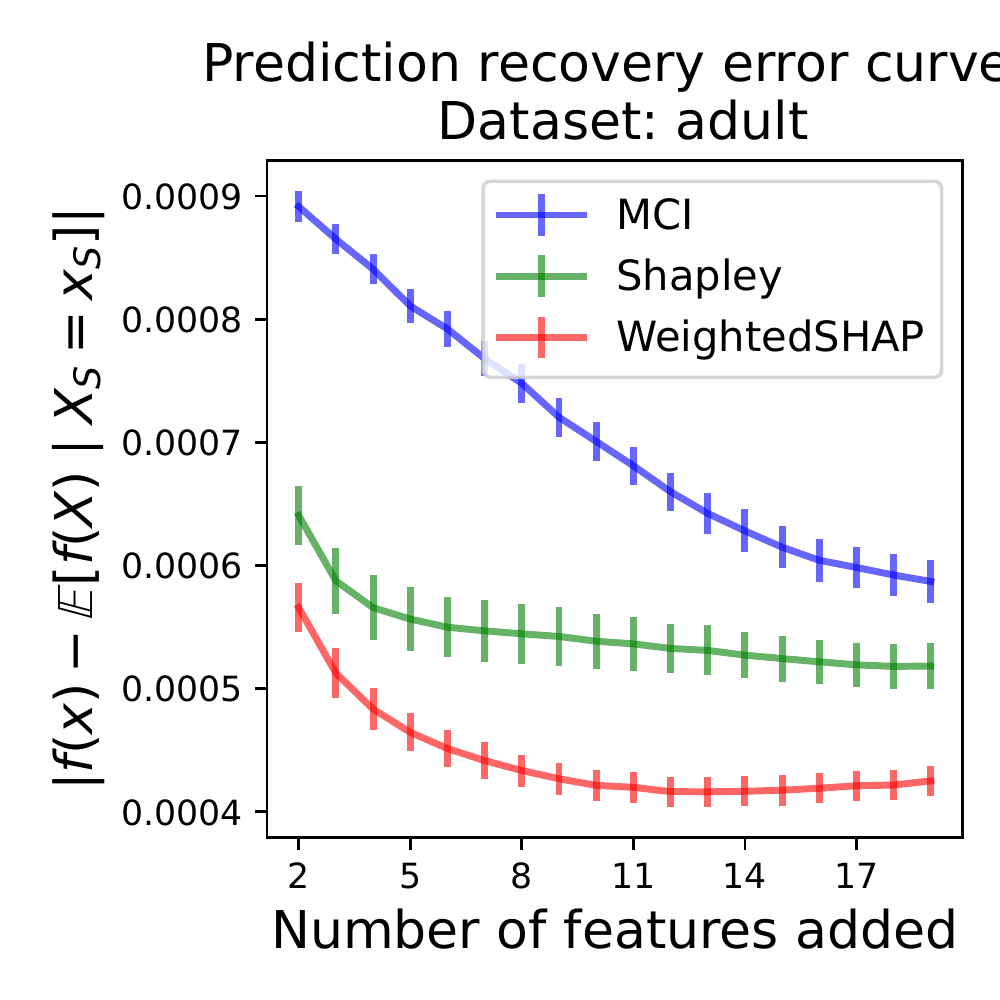}
    \includegraphics[width=0.22\textwidth]{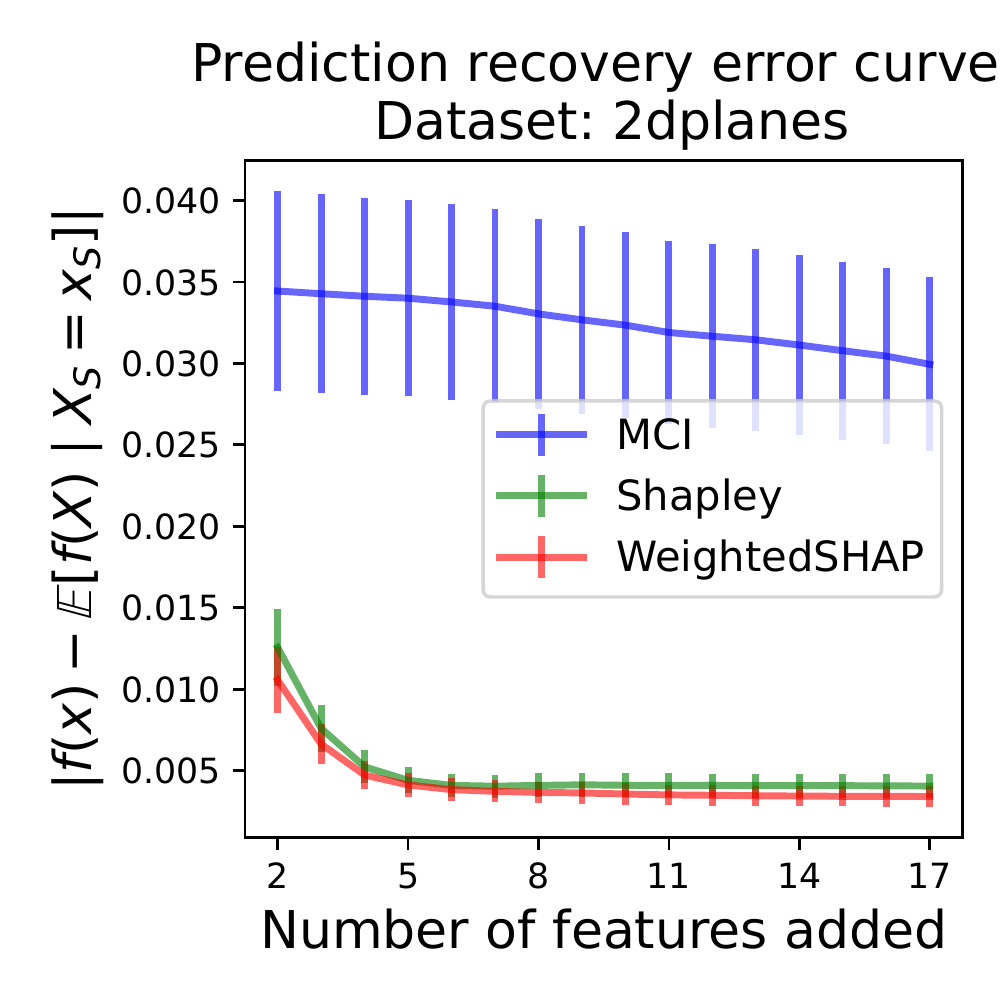}
    \includegraphics[width=0.22\textwidth]{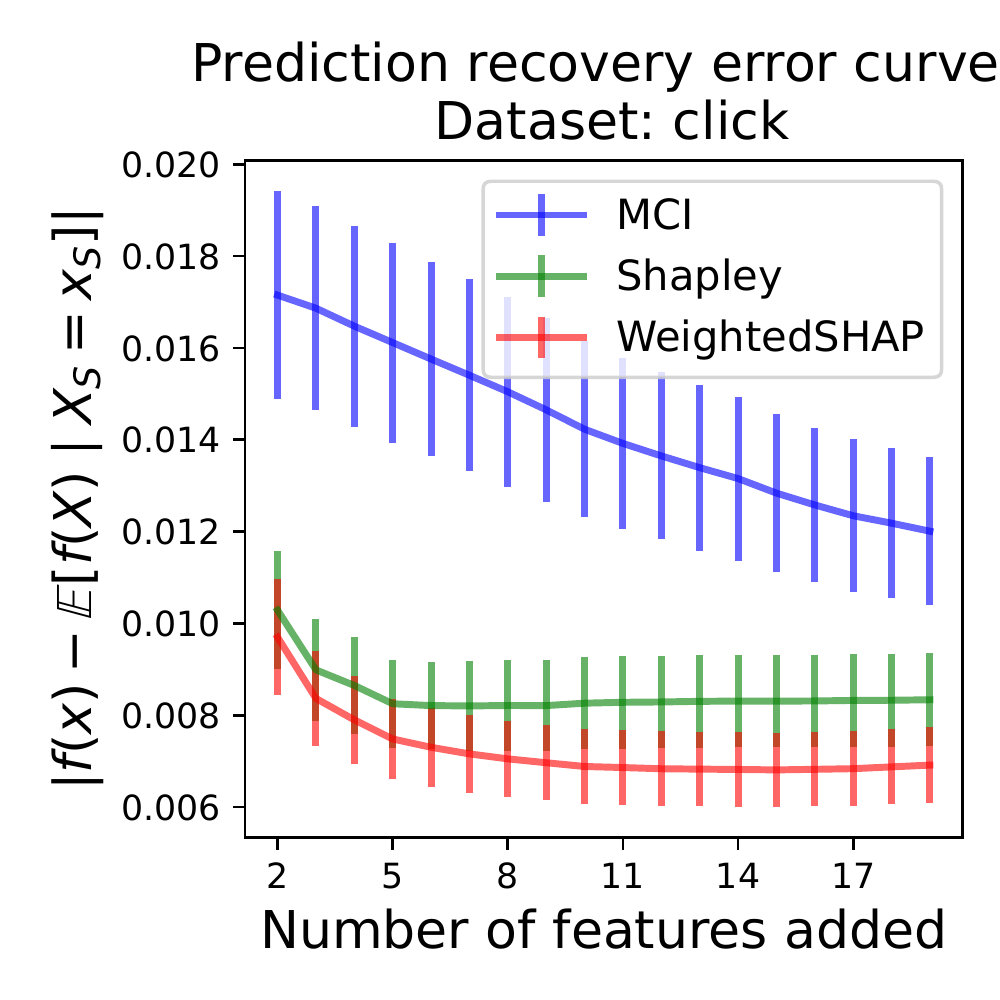}
    \label{fig:prediction_error_classification_boosting_additional}
    }
    \subfigure[Illustrations of the Inclusion AUC curve on the four binary classification datasets. The higher, the better.]{
    \includegraphics[width=0.22\textwidth]{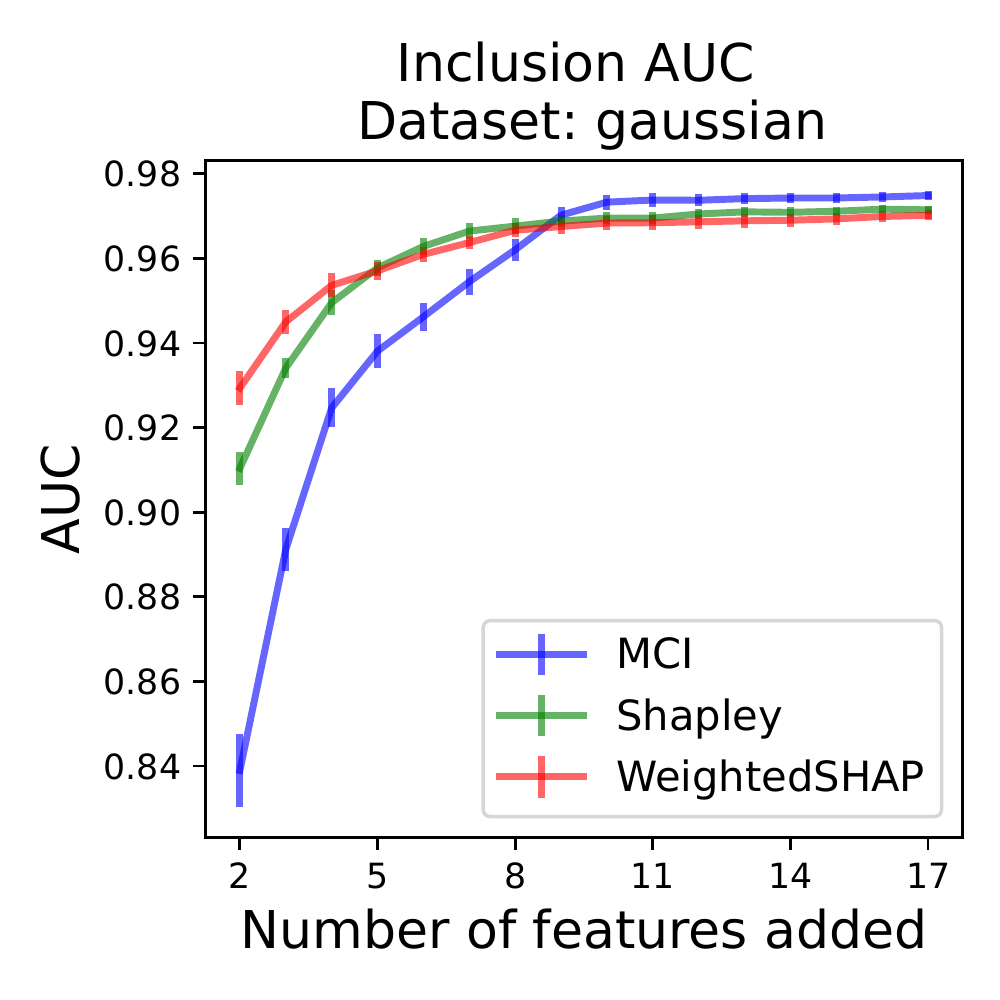}
    \includegraphics[width=0.22\textwidth]{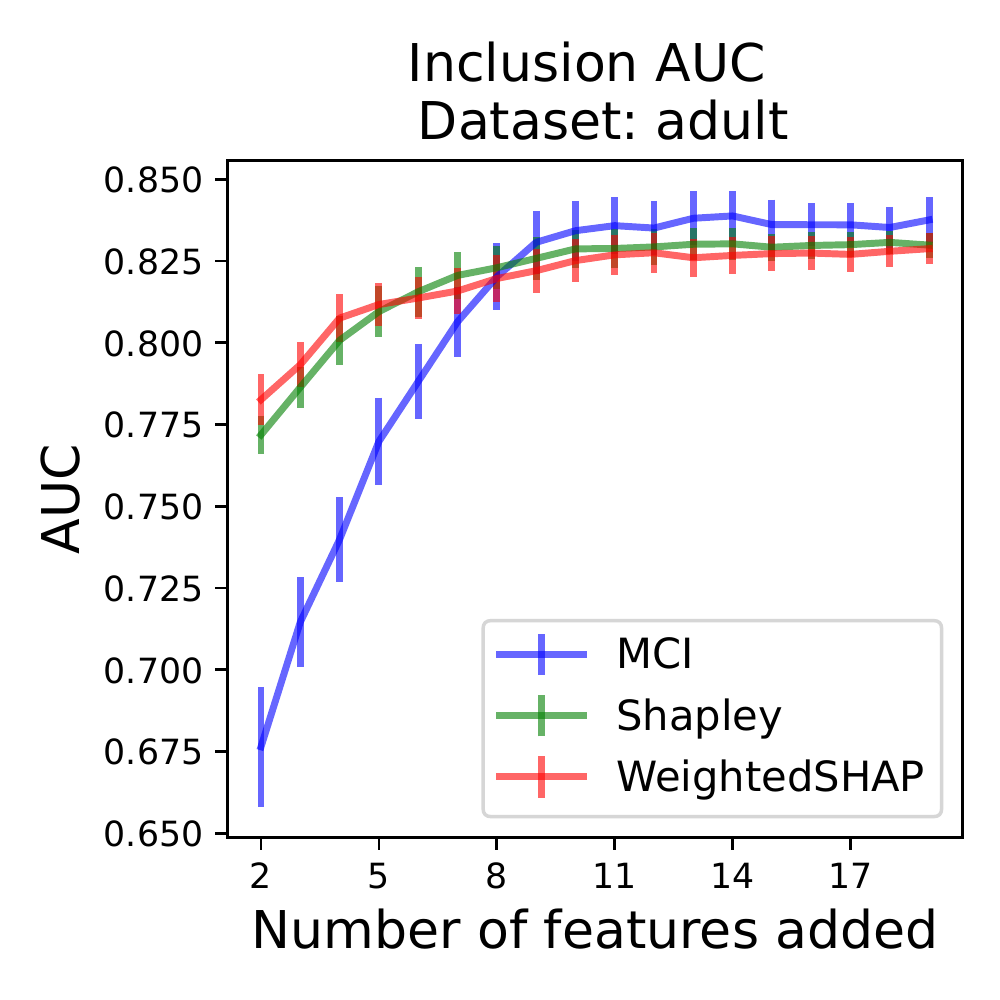}
    \includegraphics[width=0.22\textwidth]{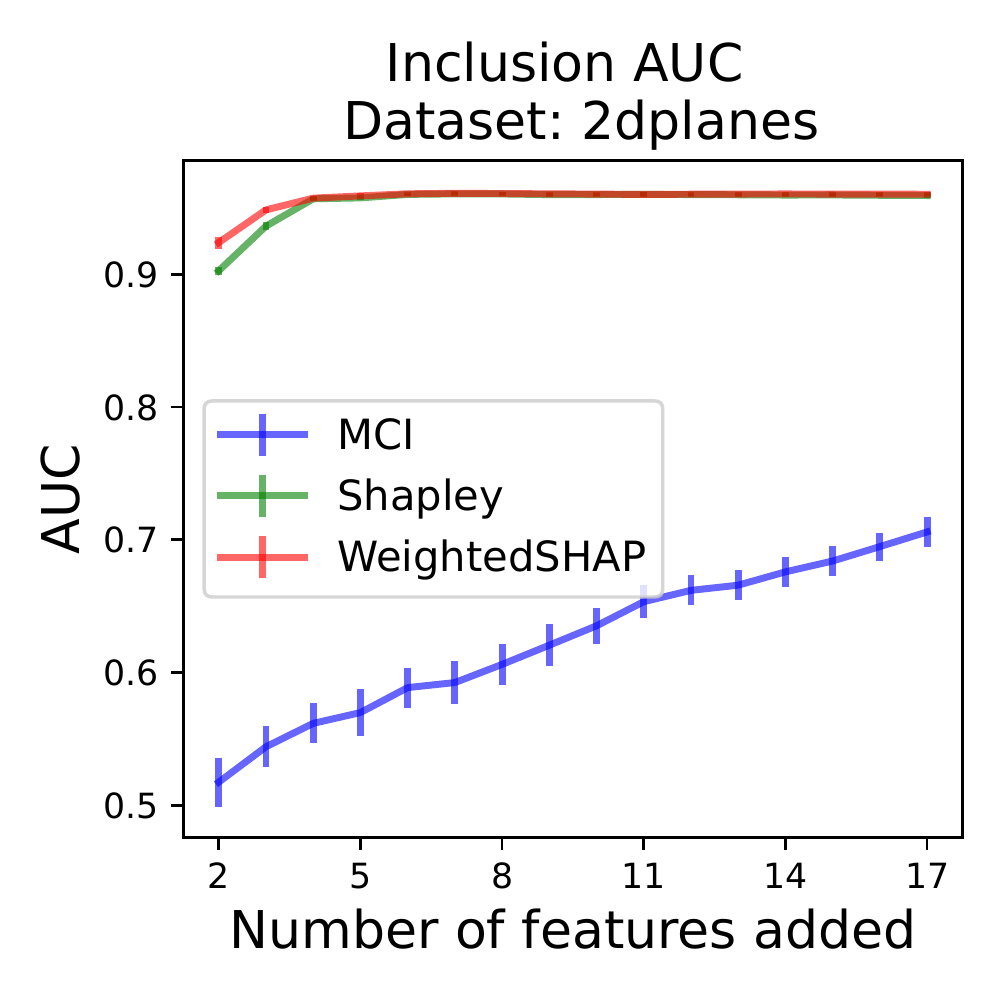}
    \includegraphics[width=0.22\textwidth]{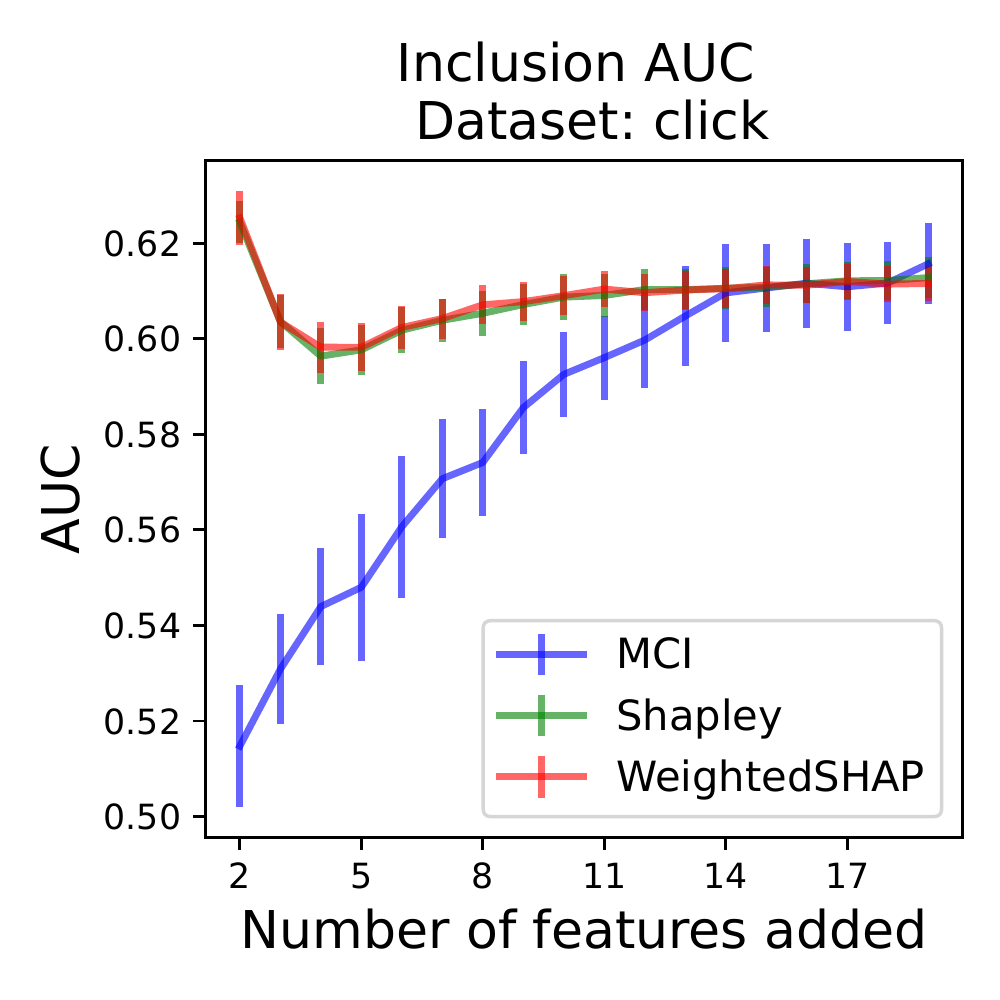}
    \label{fig:inclusion_auc_classification_boosting_additional}
    }
    \caption{\textbf{Additional results on different classification datasets.} Illustrations of the prediction recovery error curve and the performance curve as a function of the number of features added. We add features from most influential to the least influential. We denote a 95\% confidence interval based on 50 independent runs. Our main findings are consistently observed on various classification datasets.}
\end{figure}

\subsection{Additional experimental results with different evaluation metrics}
\label{app:additional_results_evaluations}
We further examine the performance of WeightedSHAP on two different evaluation metrics: the Exclusion performance used in \citet{jethani2021fastshap} and the Inclusion performance using masked features used in \citet{masoomi2021explanations}. As for the Inclusion performance in Figures~\ref{fig:masking_mse_regression_boosting} and ~\ref{fig:masking_auc_classification_boosting}, the main difference from Figures~\ref{fig:inclusion_auc_regression_boosting} and~\ref{fig:inclusion_auc_classification_boosting} is that it considers model predictions based on mean-masked features whenever new features are added, not the conditional expectation. 
For each evaluation metric, we use different utility functions in the weight optimization \eqref{eqn:weightedSHAP}. Specifically, for the Exclusion task, we use \begin{align*}
    \mathcal{T}(\phi) = \sum_{k=1} ^d \left| \hat{f}(x)-\mathbb{E}[\hat{f}(X) \mid X_{\mathcal{J}(k; \phi, x)}=x_{\mathcal{J}(k; \phi, x)} ] \right|,
\end{align*}
where $\mathcal{J}(k; \phi, x) := [d] \backslash \mathcal{I}(k; \phi, x)$ and $\mathcal{I}(k; \phi, x) \subseteq [d]$ be a set of $k$ integers that indicates $k$ most influential features based on their absolute value $|\phi(x_j)|$. That is, it adds the least important feature first. This is equivalent to removing the most important feature from the entire set of features. As for the Inclusion task, we consider the following quantity.
\begin{align*}
    \mathcal{T}(\phi) = \sum_{k=1} ^d \left| \hat{f}(x)- \hat{f}( x_{\mathcal{I}(k; \phi, x)}, \mu_{\mathcal{J}(k; \phi, x)} ) \right|,
\end{align*}
where $\mu$ is a pre-computed mean of the entire features. The main idea is that $\mu$ part is non-informative to a particular prediction, but after replacing the $\mu$ part with influential features, the quantity $\left| \hat{f}(x)- \hat{f}( x_{\mathcal{I}(k; \phi, x)}, \mu_{\mathcal{J}(k; \phi, x)} ) \right|$ is expected to be reduced. 

We note that there is not an agreed-upon objective metric for ML interpretability, and at the same time, practitioners always can choose their downstream objectives depending on their tasks.
Given that the concept of influential features can be dependent on downstream tasks, we believe the attribution should be optimized to downstream objectives. 
For instance, there is no single attribution method that works universally well simultaneously on the Inclusion AUC and Exclusion AUC tasks, and the optimal attribution depends on the evaluation metric. Our WeightedSHAP is anticipated to be flexible in optimizing downstream objectives and finds the most suitable attributions. In contrast, the Shapley value is fixed to every downstream task.

Figures~\ref{fig:exclusion_mse_regression_boosting} and~\ref{fig:masking_mse_regression_boosting} (\textit{resp.} Figures~\ref{fig:exclusion_auc_classification_boosting} and~\ref{fig:masking_auc_classification_boosting}) compare MSE (\textit{resp.} AUC) curves of the three different attribution methods. In most experimental settings, as anticipated, WeightedSHAP substantially outperforms the MCI and the Shapley value. With flexible choices of the utility function, WeightedSHAP shows significantly better performances than standard attribution methods across different downstream objectives.

\begin{figure}[t]
    \centering
    \subfigure[Illustrations of the Exclusion MSE on the four regression datasets. The higher, the better.]{
    \includegraphics[width=0.22\textwidth]{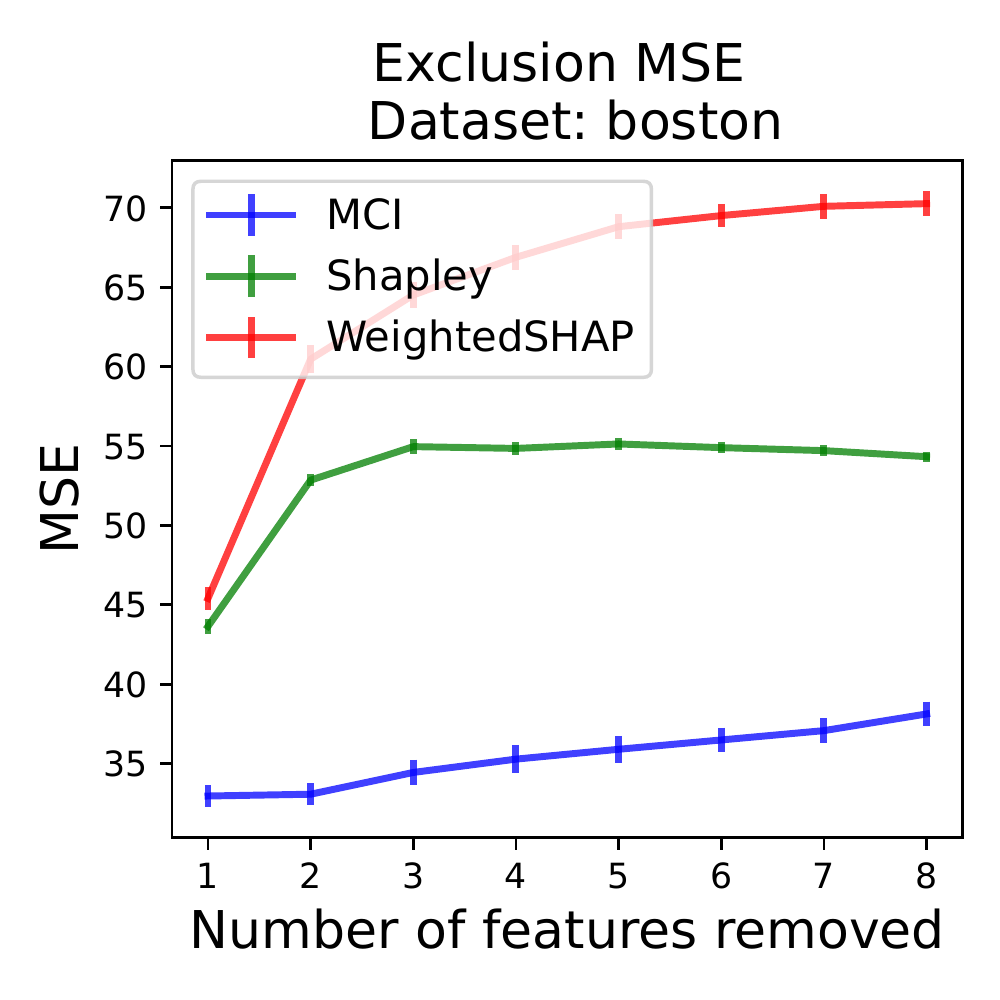}
    \includegraphics[width=0.22\textwidth]{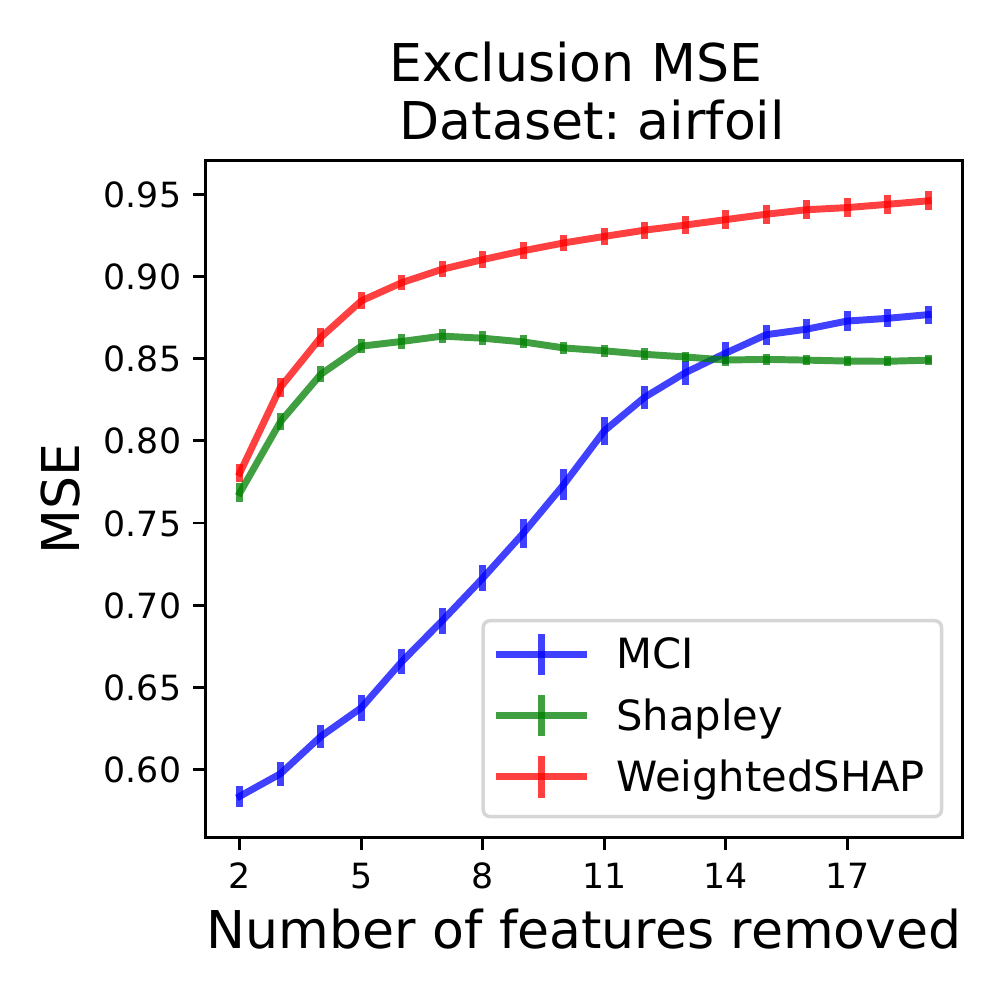}
    \includegraphics[width=0.22\textwidth]{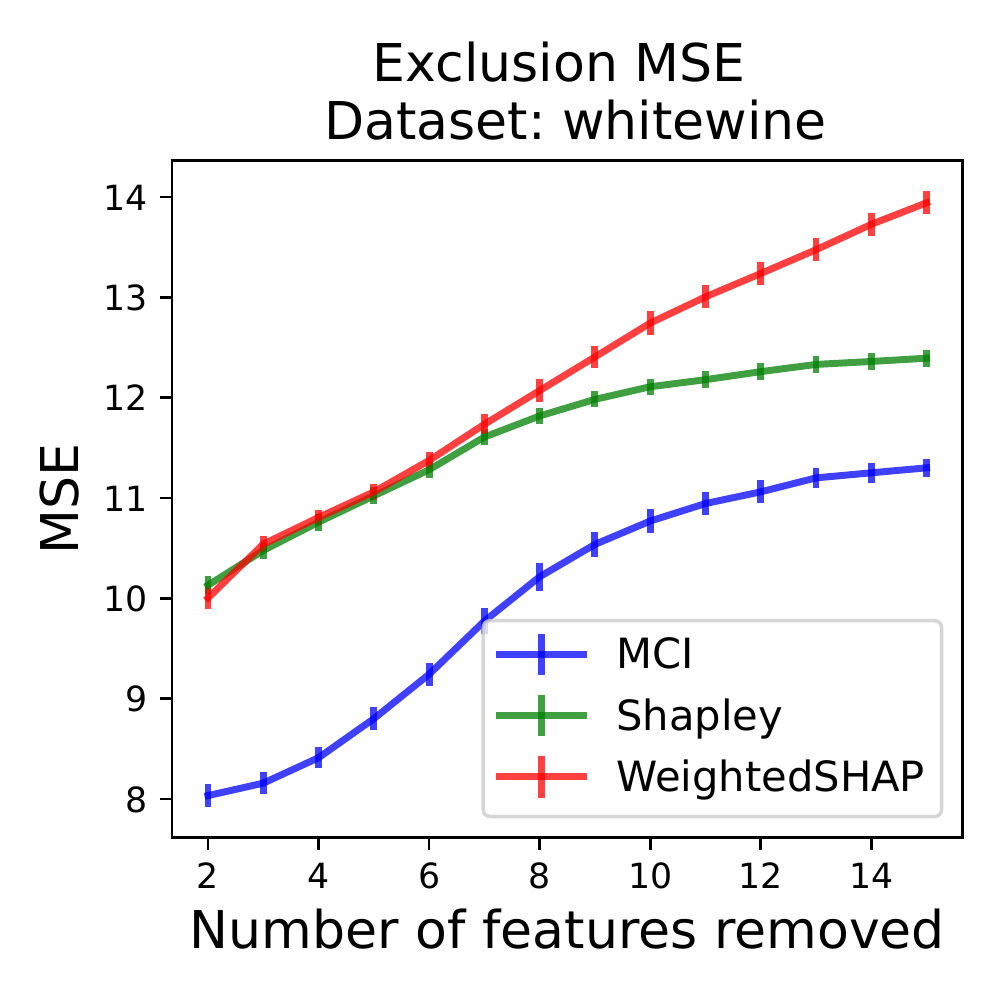}
    \includegraphics[width=0.22\textwidth]{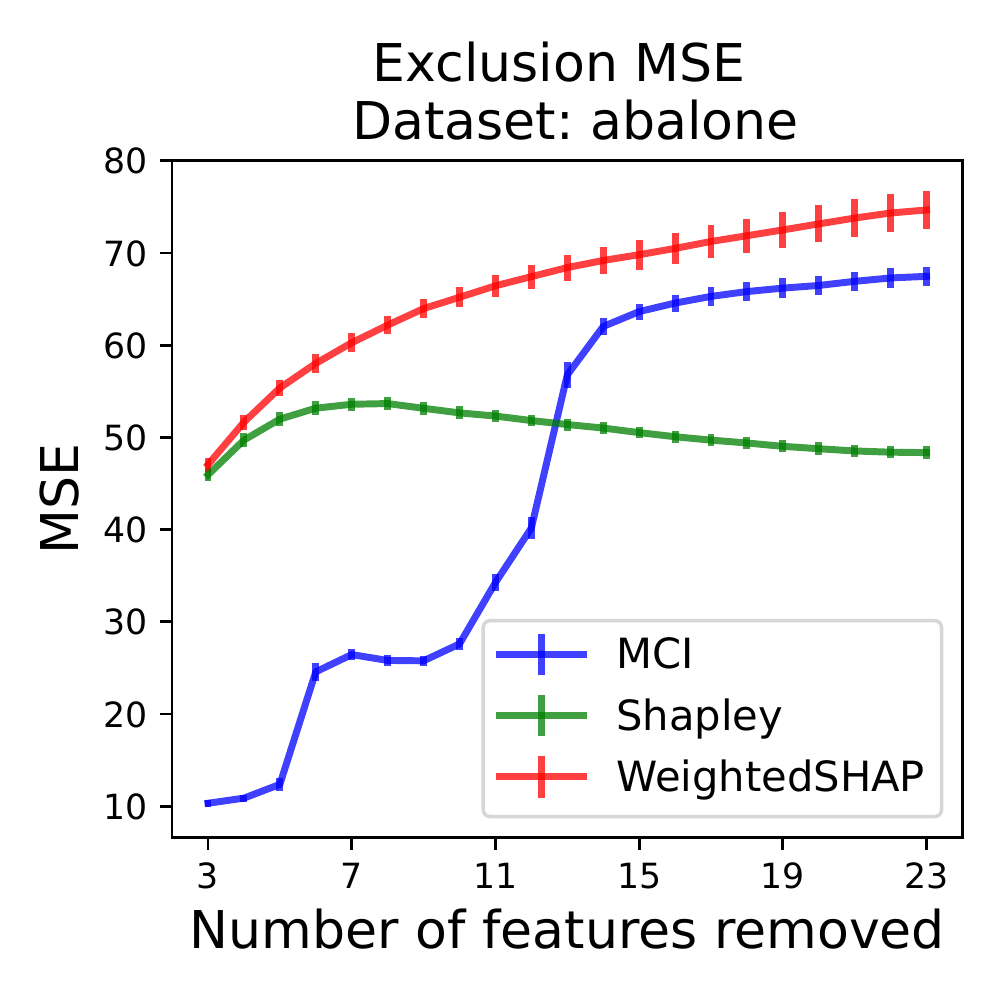}
    \label{fig:exclusion_mse_regression_boosting}
    }
    \subfigure[Illustrations of the Inclusion MSE curve using masked features on the four regression datasets. The lower, the better.]{
    \includegraphics[width=0.22\textwidth]{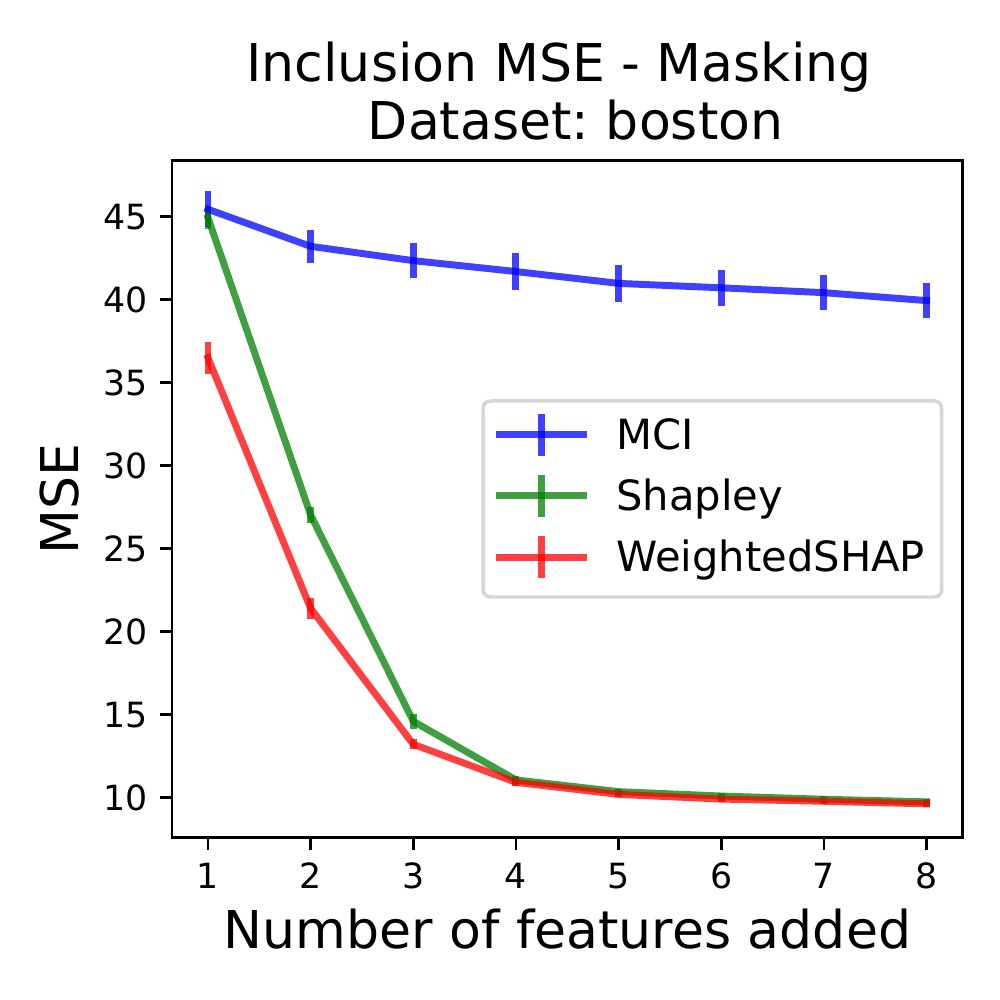}
    \includegraphics[width=0.22\textwidth]{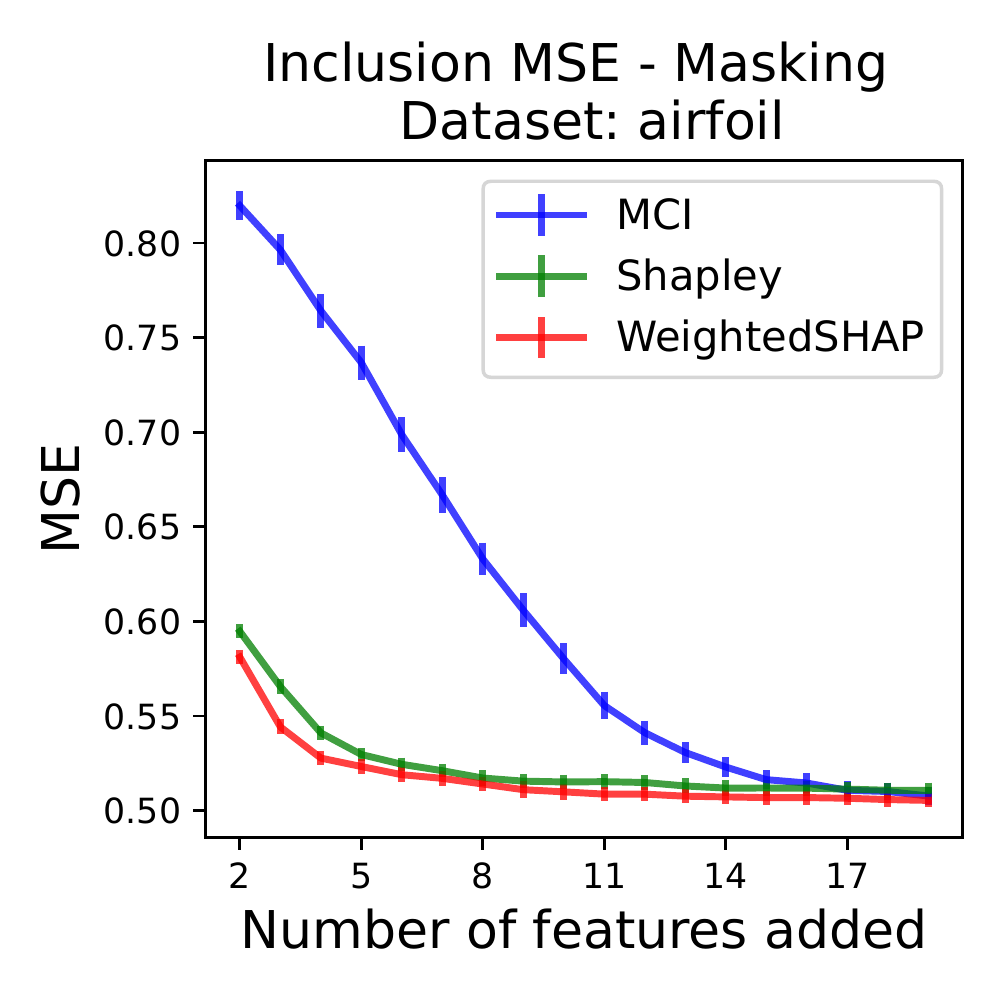}
    \includegraphics[width=0.22\textwidth]{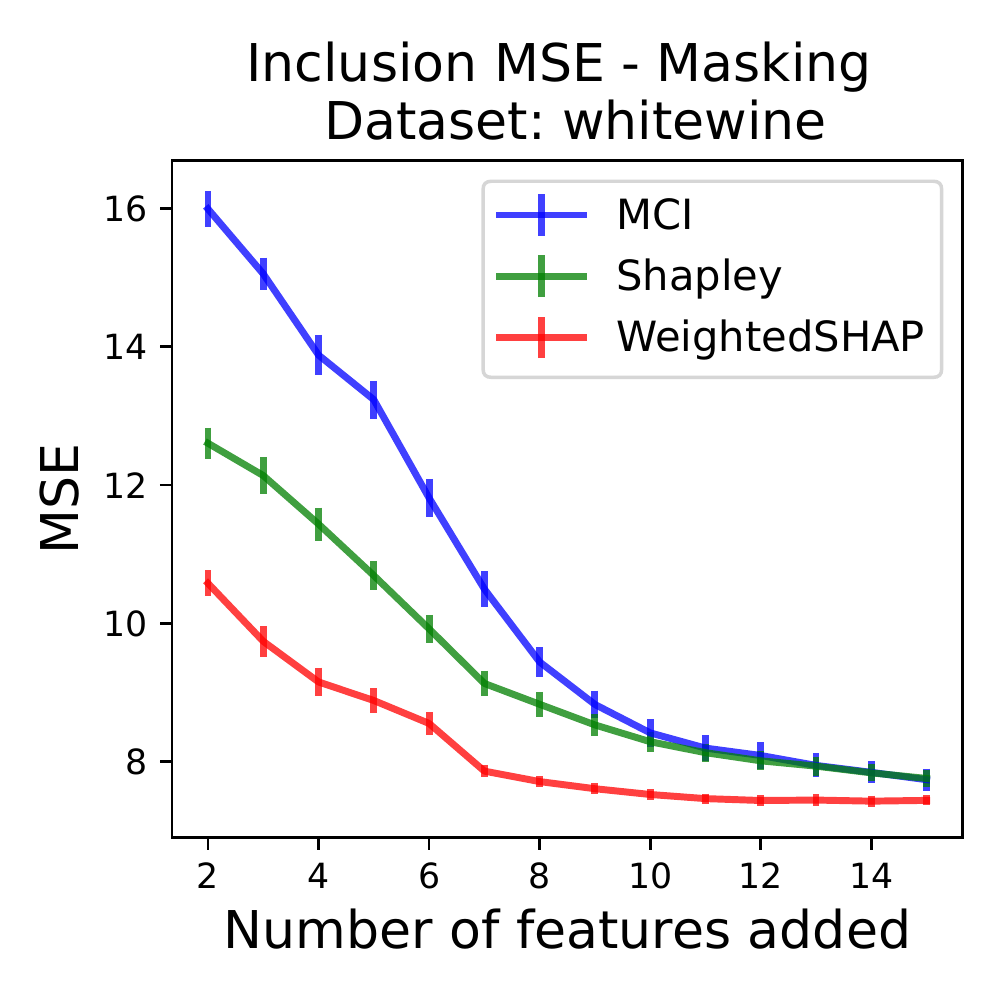}
    \includegraphics[width=0.22\textwidth]{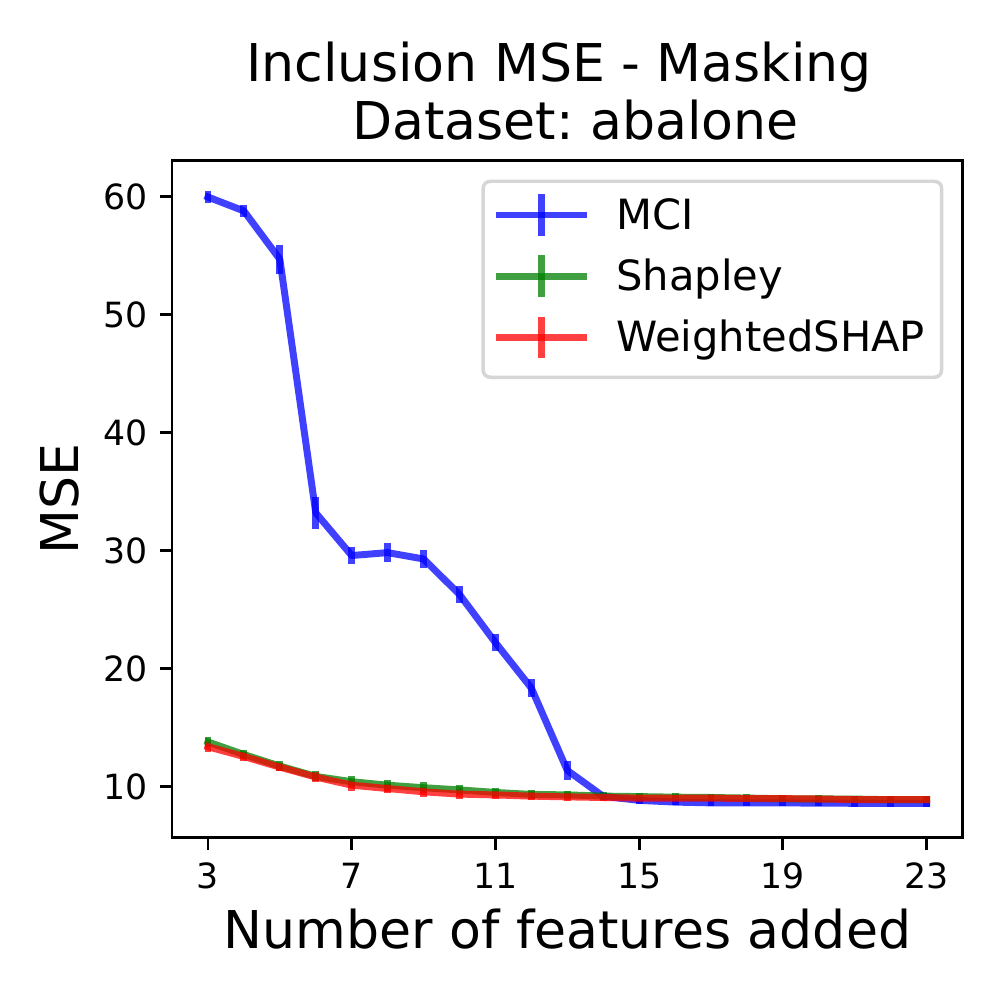}
    \label{fig:masking_mse_regression_boosting}
    }
    \subfigure[Illustrations of the Exclusion AUC curve on the four binary classification datasets. The lower, the better.]{
    \includegraphics[width=0.22\textwidth]{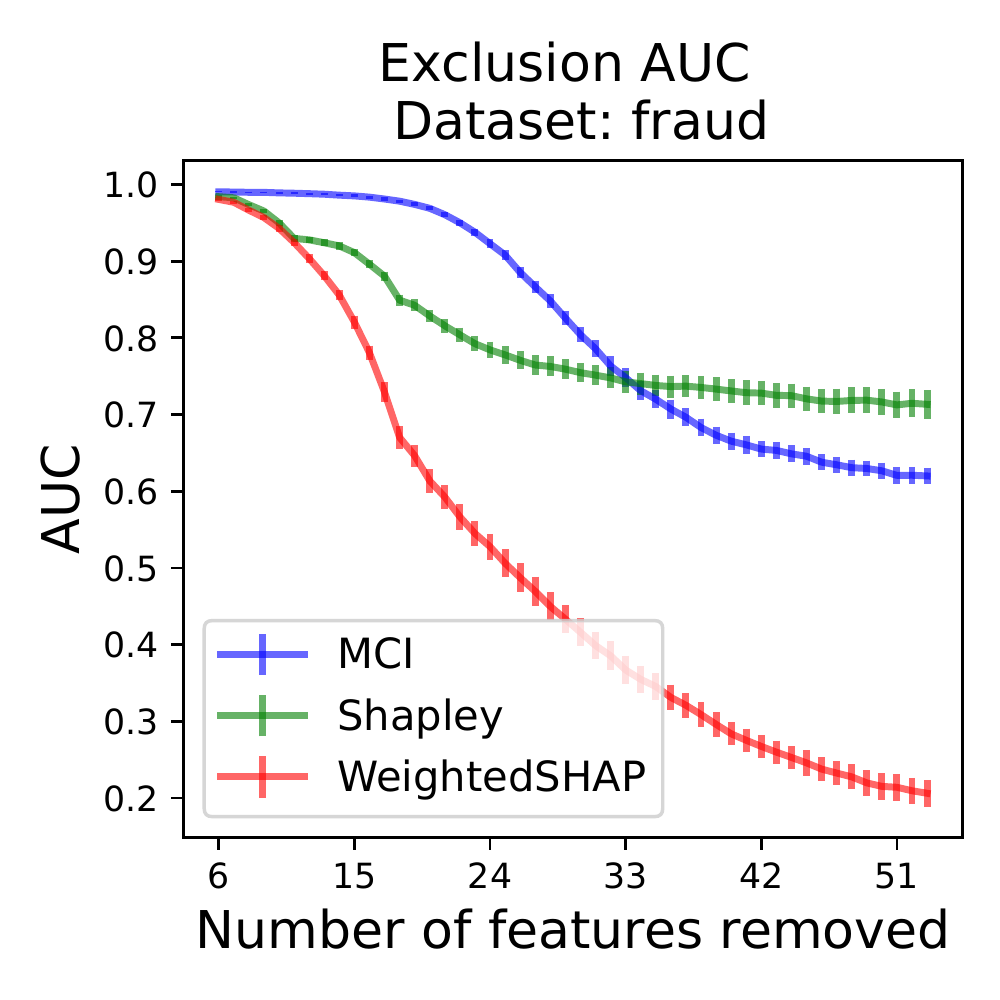}
    \includegraphics[width=0.22\textwidth]{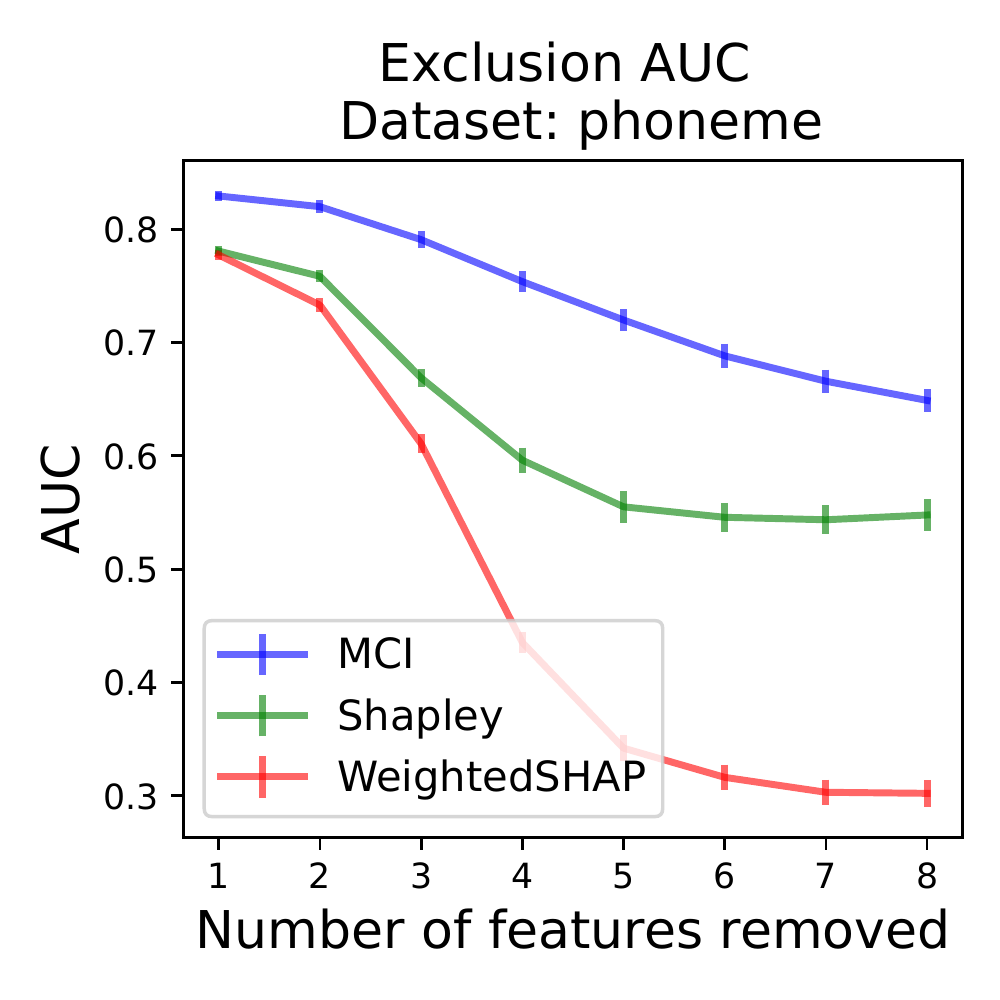}
    \includegraphics[width=0.22\textwidth]{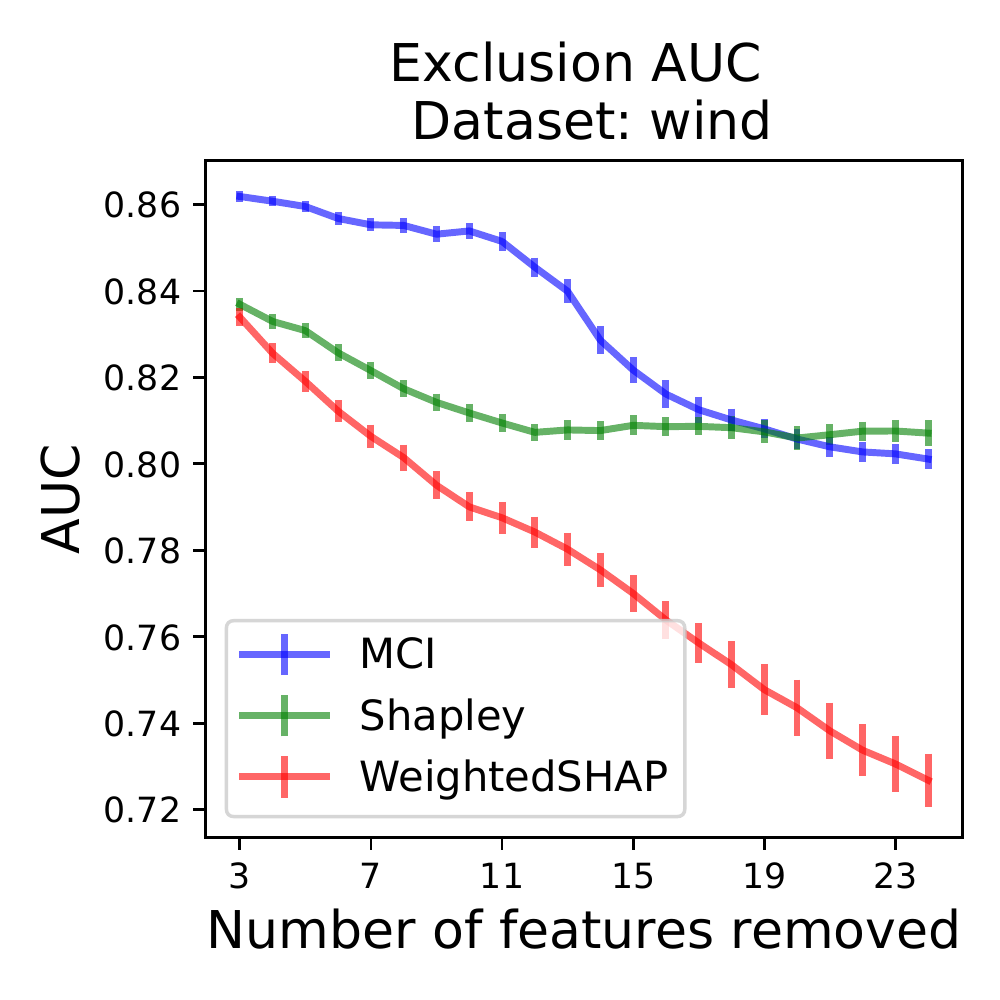}
    \includegraphics[width=0.22\textwidth]{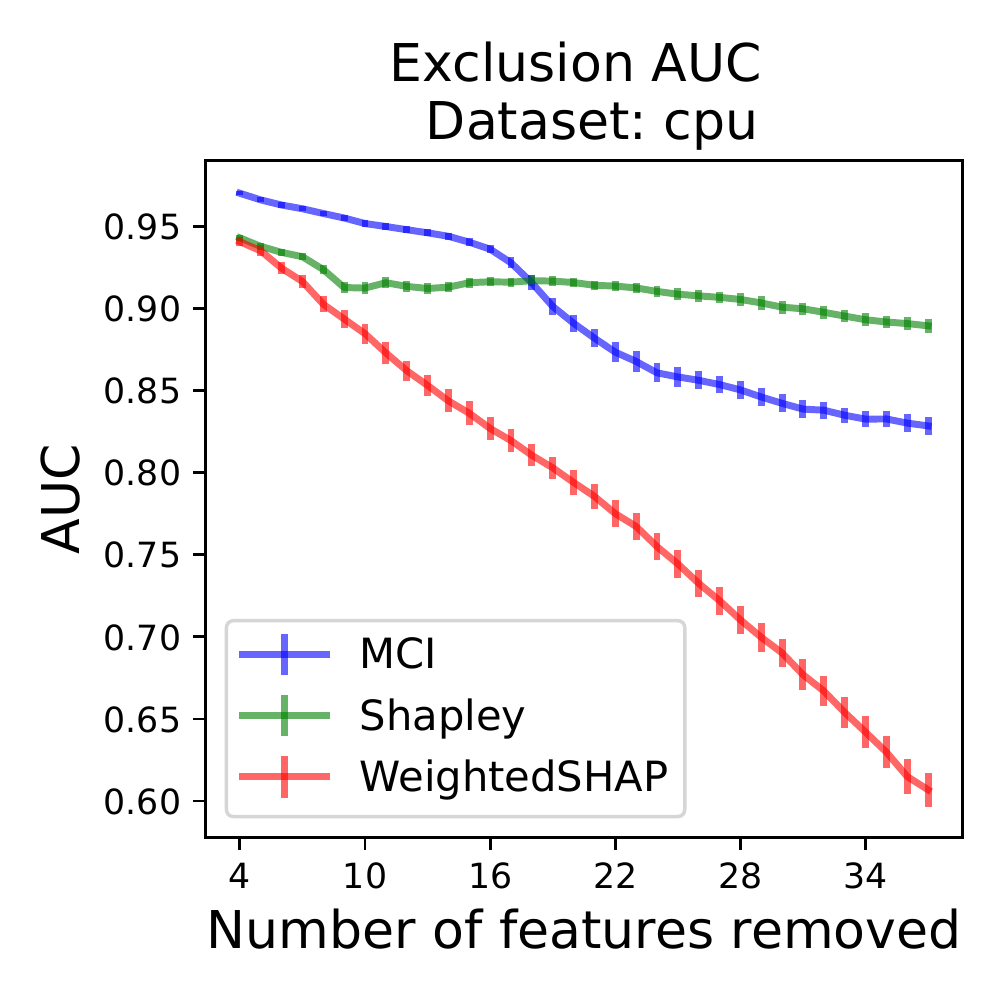}
    \label{fig:exclusion_auc_classification_boosting}
    }
    \subfigure[Illustrations of the Inclusion AUC curve using masked features on the four binary classification datasets. The higher, the better.]{
    \includegraphics[width=0.22\textwidth]{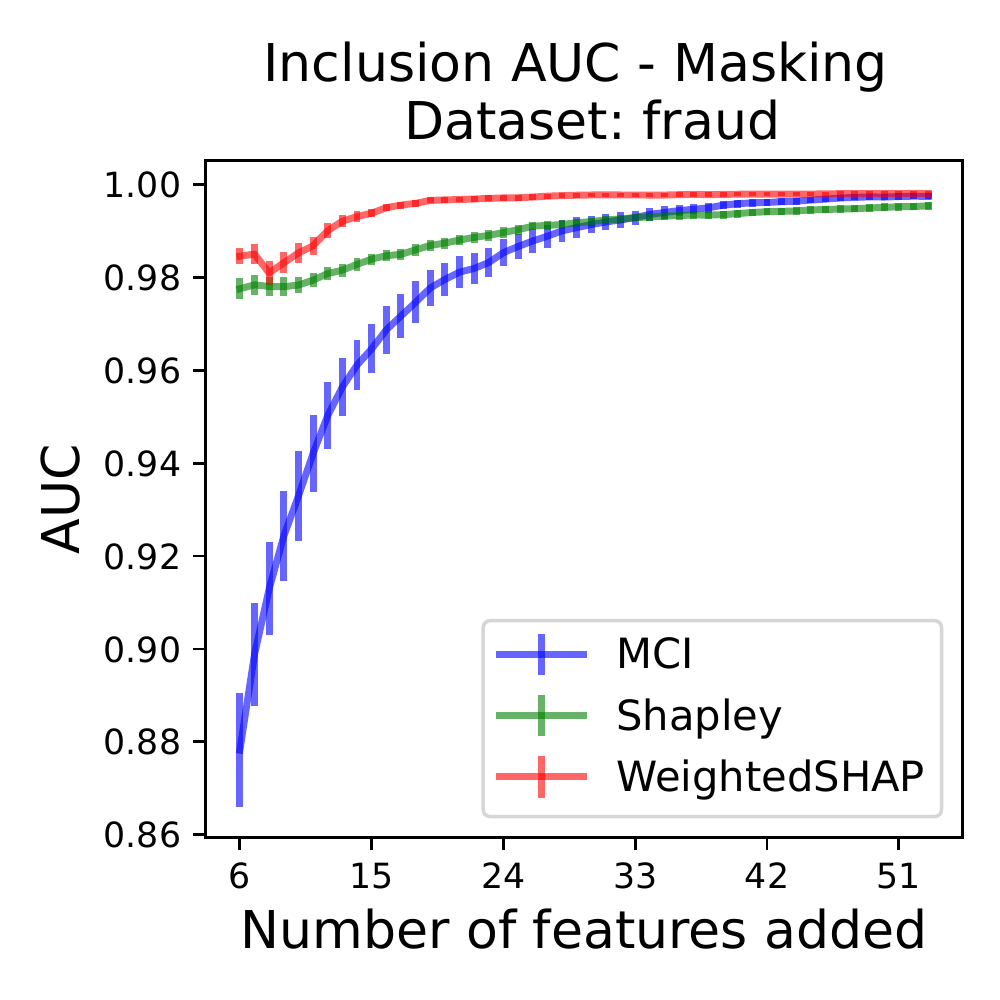}
    \includegraphics[width=0.22\textwidth]{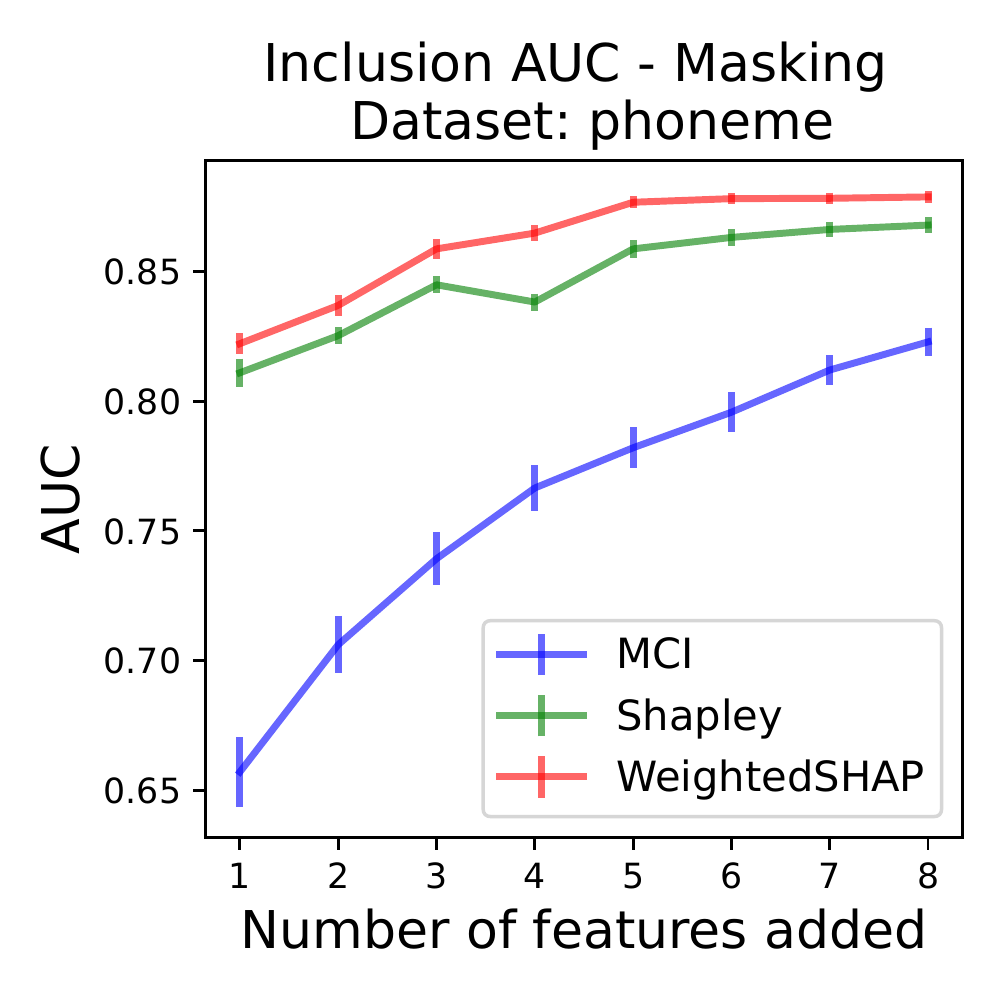}
    \includegraphics[width=0.22\textwidth]{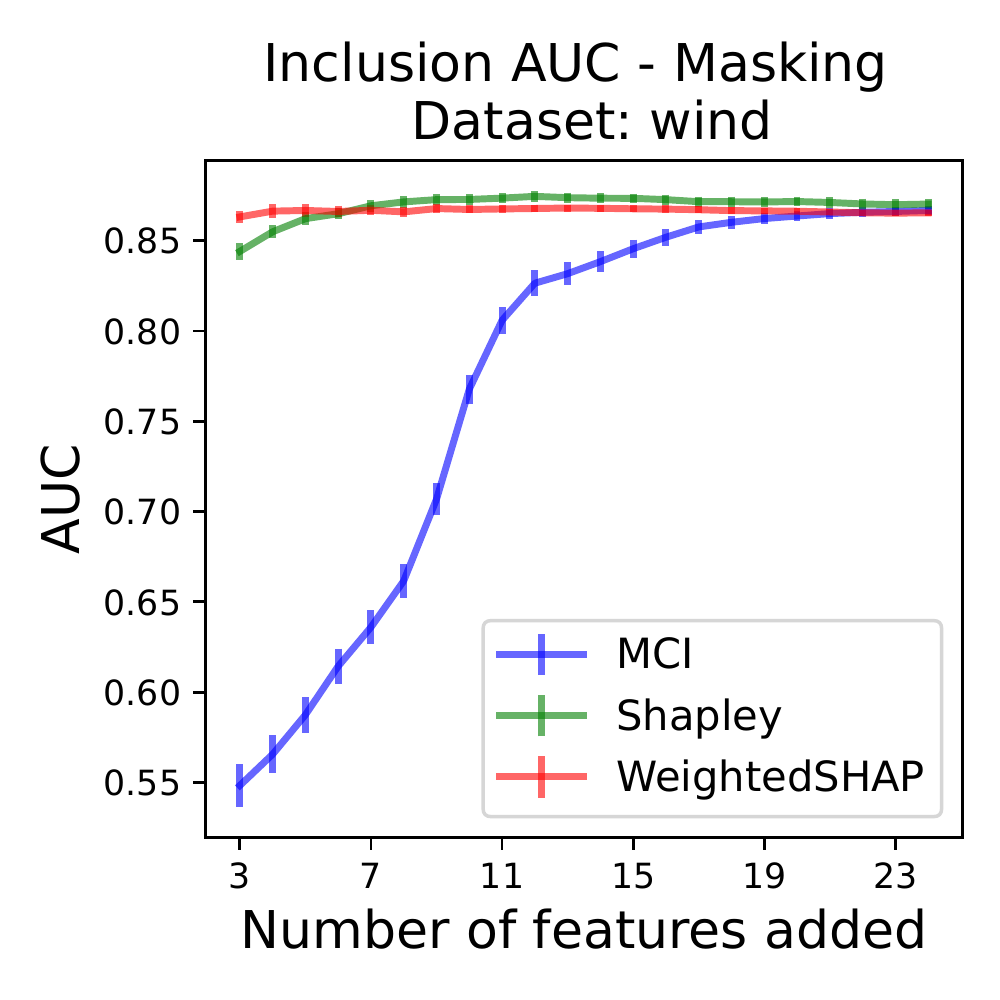}
    \includegraphics[width=0.22\textwidth]{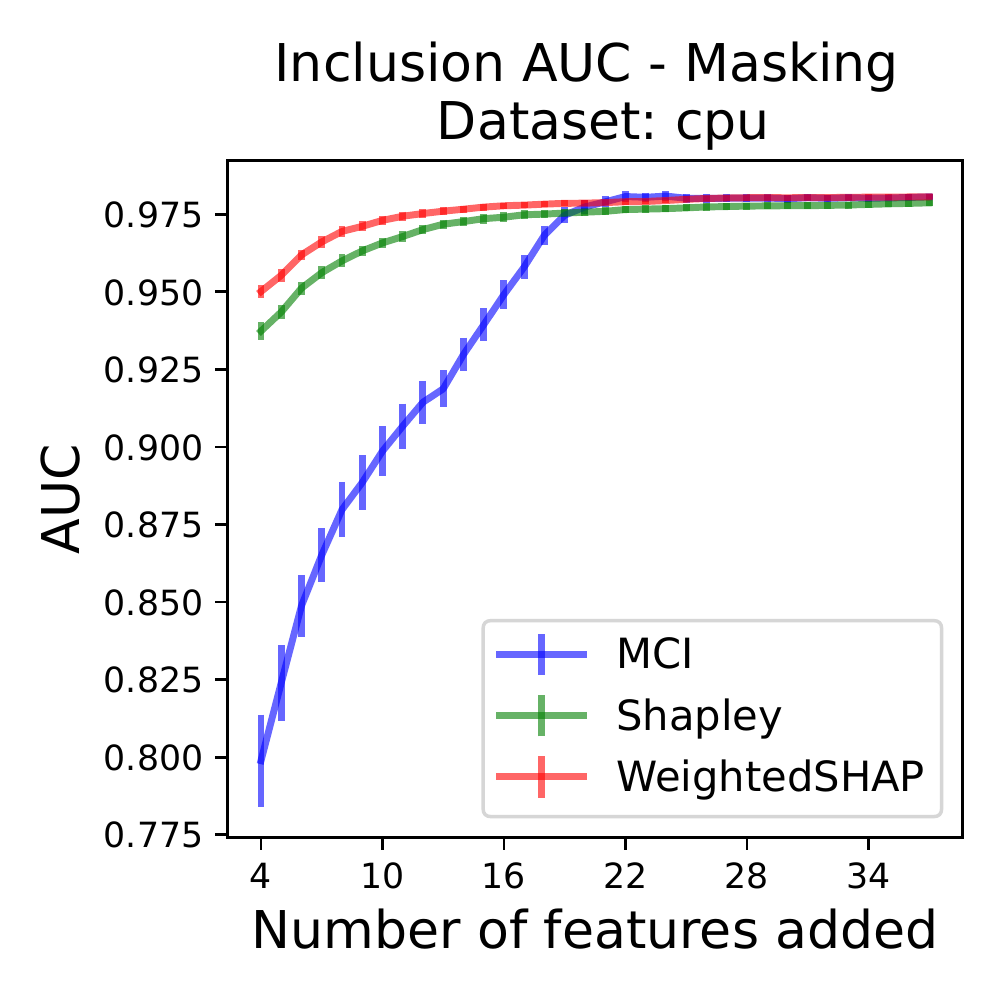}
    \label{fig:masking_auc_classification_boosting}
    }
    \caption{\textbf{WeightedSHAP on different evaluation metrics.} Illustrations of the Exclusion performance curve (\textit{resp.} Inclusion performance curve using masked features) as a function of the number of removed (\textit{resp.} added). We remove (\textit{resp.} add) features from most influential to the least influential. We denote a 95\% confidence interval based on 50 independent runs. WeightedSHAP significantly outperforms the MCI and the Shapley value on different evaluation metrics.}
    \label{fig:different_evaluation_metrics}
\end{figure}

\bibliographystyle{unsrtnat}
\bibliography{ref}

\end{document}